\documentclass[english]{article}

\makeatletter

\usepackage[margin =1in]{geometry}

\usepackage{amsmath,amsthm,amssymb,amsfonts,mathtools}
\usepackage{bbm} %
\usepackage{blkarray} %
\usepackage{stackengine} %
\usepackage{enumitem}
\usepackage[square,sort,comma,numbers]{natbib}
\usepackage{dsfont} %
\usepackage[normalem]{ulem} %
\usepackage{appendix}

\numberwithin{equation}{section}

\DeclarePairedDelimiter{\dotp}{\langle}{\rangle}
\DeclarePairedDelimiter{\smin}{s_{\min}(}{)}

\makeatletter
\newcommand{\longdash}[1][2em]{%
  \makebox[#1]{$\m@th\smash-\mkern-7mu\cleaders\hbox{$\mkern-2mu\smash-\mkern-2mu$}\hfill\mkern-7mu\smash-$}}
\makeatother
\newcommand{\omitskip}{\kern-\arraycolsep}

\newcommand{\C}{{\mathbb C}}
\newcommand{\E}{{\mathbb E}}

\newcommand{\R}{{\mathbb R}}

\newcommand{\real}{\mathbb{R}}%
\newcommand{\N}{{\mathbb N}}

\newcommand{\cB}{\mathcal{B}}

\newcommand{\cP}{\mathcal{P}}

\newcommand{\cS}{\mathcal{S}}

\newcommand{\cX}{\mathcal{X}}

\newcommand{\dd}{\textup{d}}

\newcommand{\Padj}{P^\ast}

\newcommand{\onevec}{{1}}%
\newcommand{\indic}{\operatorname{\mathbbm{1}}}%

\newcommand{\sampleOne}{{[1]}}
\newcommand{\sampleTwo}{{[2]}}

\newcommand{\Amax}{A_{\max}}
\newcommand{\bmax}{b_{\max}}
\newcommand{\gammin}{\gamma_{\min}}
\newcommand{\gammax}{\gamma_{\max}}
\newcommand{\law}{\mathcal{L}}
\newcommand{\bigO}{\mathcal{O}}
\newcommand{\indep}{\perp \!\!\! \perp}

\newcommand{\var}{\operatorname{Var}}

\newcommand{\absgap}{\operatorname{\gamma^\star}}
\newcommand{\spec}{\operatorname{Spec}}

\newcommand{\TV}{\operatorname{TV}}
\newcommand{\tr}{\operatorname{Tr}}
\newcommand{\ddup}{\mathrm{d}}
\newcommand{\ltwopi}{L^2(\pi)}

\newcommand{\obs}{g}

\usepackage[unicode=true,pdfusetitle, bookmarks=true,bookmarksnumbered=true,bookmarksopen=false, breaklinks=true,pdfborder={0 0 0},pdfborderstyle={},backref=page,colorlinks=true,citecolor=blue]{hyperref}
\usepackage{comment} %
\usepackage{cancel}

\usepackage{graphicx}
\usepackage{wrapfig}
\usepackage{tikz} %
\usepackage{float}
\usepackage{subfigure}
\usepackage{caption}

\newtheorem{thm}{Theorem}[section]
\newtheorem{definition}[thm]{Definition}
\newtheorem{proposition}[thm]{Proposition}
\newtheorem{lem}[thm]{Lemma}
\newtheorem{cor}[thm]{Corollary}

\newtheorem{assumption}{Assumption}

\newtheorem*{question*}{Question}
\newtheorem{claim}{Claim}
\newtheorem*{claim*}{Claim}
\newtheorem{remark}{Remark}

\global\long\def\BarA{\Bar A}%
\global\long\def\mS{\mathcal{S}}%

\global\long\def\dA{D}%
\global\long\def\dAbar{\Bar D}%
\global\long\def\Upsilon{\upsilon}%
\global\long\def\S{\mathbb{S}}%
\global\long\def\indic{\operatorname{\mathbbm{1}}}%

\makeatother

\begin{document}

\title{Bias and Extrapolation in Markovian Linear Stochastic Approximation
with Constant Stepsizes}

\author{Dongyan (Lucy) Huo,\texorpdfstring{$^\mathsection$}{} Yudong Chen,\texorpdfstring{$^\dagger$}{} Qiaomin Xie,\texorpdfstring{$^\ddagger$}{}%
    \texorpdfstring{\footnote{Emails: \texttt{dh622@cornell.edu}, \texttt{yudong.chen@wisc.edu}, \texttt{qiaomin.xie@wisc.edu}}\\~\\
	\normalsize $^\mathsection$School of Operations Research and Information Engineering, Cornell University\\
	\normalsize $^\dagger$Department of Computer Sciences, University of Wisconsin-Madison\\
	\normalsize $^\ddagger$Department of Industrial and Systems Engineering, University of Wisconsin-Madison}{}
}

\date{}

\maketitle

\begin{abstract}
We consider Linear Stochastic Approximation (LSA) with constant stepsize
and Markovian data. Viewing the joint process of the data and LSA
iterate as a time-homogeneous Markov chain, we prove its convergence
to a unique limiting and stationary distribution in Wasserstein distance
and establish non-asymptotic, geometric convergence rates. Furthermore,
we show that the bias vector of this limit admits an infinite series
expansion with respect to the stepsize. Consequently, the bias is
proportional to the stepsize up to higher order terms. This result
stands in contrast with LSA under i.i.d.\ data, for which the bias
vanishes. In the reversible chain setting, we provide a general characterization
of the relationship between the bias and the mixing time of the Markovian
data, establishing that they are roughly proportional to each other.

While Polyak-Ruppert tail-averaging reduces the variance of LSA iterates,
it does not affect the bias. The above characterization allows us
to show that the bias can be reduced using Richardson-Romberg extrapolation
with $m\ge2$ stepsizes, which eliminates the $m-1$ leading terms
in the bias expansion. This extrapolation scheme leads to an exponentially
smaller bias and an improved mean squared error, both theoretically
and empirically. Our results immediately apply to the Temporal Difference
learning algorithm with linear function approximation, and stochastic
gradient descent applied to quadratic functions.
\end{abstract}

\section{Introduction}

\label{sec:intro}

We consider the following Linear Stochastic Approximation (LSA) iteration driven by Markovian data: 
\begin{align*}
\theta_{k+1}=\theta_{k}+\alpha\big(A(x_{k})\theta_{k}+b(x_{k})\big),\quad k=0,1,2,\ldots,
\end{align*}
where $(x_{k})_{k\ge0}$ is a Markov chain representing the underlying
data stream, $A$ and $b$ are deterministic functions, and $\alpha>0$
is a constant stepsize. LSA is an iterative data-driven procedure
for approximating the solution $\theta^{*}$ to the linear fixed point
equation $\BarA\theta^{*}+\bar{b}=0$, where $\BarA:=\sum_{i}\pi_{i}A(i)$,
$\bar{b}:=\sum_{i}\pi_{i}b(i)$, and $\pi$ is the unique stationary
distribution of the chain $(x_{k})_{k\ge0}.$

Stochastic Approximation (SA), which uses stochastic updates to solve fixed-point equations, is a fundamental algorithmic paradigm in many
areas including stochastic control and filtering \citep{kushner2003-yin-sa-book,borkar08-SA-book},
approximate dynamic programming and reinforcement learning (RL) \citep{Bertsekas19-RL-book,Sutton18-RL-book}.
For example, the celebrated Temporal Difference (TD) learning algorithm~\citep{Sutton1988-td}
in RL, potentially equipped with linear function approximation, is
a special case of LSA and an important algorithm primitive in RL.
Variants of the TD algorithm such as TD($\lambda$) and Gradient TD
can also be cast as LSA~\citep{Lakshminarayanan18-LSA-Constant-iid}.
Another important special case of SA is stochastic gradient descent
(SGD). When applied to a quadratic objective function, SGD can be
viewed as an LSA procedure.

Classical work on SA focuses on settings with diminishing stepsizes,
which allow for asymptotic convergence of $\theta_{k}$ to $\theta^{*}$
\citep{Robbins51-Monro-SA,Blum54-SA,borkar2000-ode-sa}. More recently,
SA with constant stepsizes has attracted attention due to its simplicity,
fast convergence, and good empirical performance. A growing line of
work has been devoted to this setting and established non-asymptotic
results valid for finite values of $k$ \citep{Lakshminarayanan18-LSA-Constant-iid,srikant-ying19-finite-LSA,chen21-finite-td,Bhandari21-linear-td}.
These results provide \emph{upper bounds} on the mean-squared error
(MSE) $\E\|\theta_{k}-\theta^{*}\|^{2}$, and such bounds typically
consist of the sum of two terms: a finite-time term\footnote{This term is sometimes called ``(finite-time) bias'' in prior work.
It should not be confused with the \emph{asymptotic} bias discussed
below.} that decays with $k$, and a steady-state MSE bound that is independent
of~$k$. 

In this work, we study LSA with constant stepsizes through the lens of
Markov chain theory. This perspective allows us to delineate the convergence
and distributional properties of LSA viewed as a stochastic process.
Consequently, we provide a more precise characterization of the MSE
in terms of the decomposition 
\[
\E\|\theta_{k}-\theta^{*}\|^{2}\asymp\underbrace{\|\E\theta_{k}-\E\theta_{\infty}^{(\alpha)}\|^{2}}_{\text{optimization error}}+\underbrace{\var(\theta_{k})}_{\substack{\text{variance}}
}+\underbrace{\|\E\theta_{\infty}^{(\alpha)}-\theta^{*}\|^{2}}_{\text{asymptotic bias}^{2}},
\]
where the random variable $\theta_{\infty}^{(\alpha)}$ denotes the limit (as $k\to\infty$) of the LSA
iterate $\theta_{k}$ with stepsize $\alpha$. Our main results, summarized
below, characterize the behavior of the above three terms.

\paragraph*{Convergence and optimization error.}

With a constant stepsize $\alpha$, the process $(x_{k},\theta_{k})_{k\ge0}$
is a time-homogeneous Markov chain. Under appropriate conditions,
we show that the sequence of $(x_{k},\theta_{k})$ converges to a
unique limiting random variable $(x_{\infty},\theta_{\infty}^{(\alpha)})$
in distribution and in $W_{2}$, the Wasserstein distance of order
2. Moreover, the distribution of $(x_{\infty},\theta_{\infty}^{(\alpha)})$
is the unique stationary distribution of the chain $(x_{k},\theta_{k})_{k\ge0}$.
We provide non-asymptotic bounds on the distributional distance between
$\theta_{k}$ and $\theta_{\infty}^{(\alpha)}$ in $W_{2}$, which
in turn upper bounds the optimization error $\|\E\theta_{k}-\E\theta_{\infty}^{(\alpha)}\|$.
Both bounds decay exponentially in $k$ thanks to the use of a constant
stepsize. We emphasize that the existence of the limit $\theta_{\infty}^{(\alpha)}$
and the convergence rate cannot be deduced from existing upper bounds
on the MSE $\E\|\theta_{k}-\theta^{*}\|^{2}$, which does not vanish
as $k\to\infty.$

\paragraph*{Variance and asymptotic bias.}

The variance $\var(\theta_{k})$ is of order $\bigO(1)$ as $k$ grows.
The variance can be made vanishing by averaging the LSA iterates.
For example, the Polyak-Ruppert tail-averaged iterate $\bar{\theta}_{k}:=\frac{1}{k/2}\sum_{t=k/2}^{k-1}\theta_{t}$
has variance of order $\bigO(1/k)$. Consequently, for large $k$,
the MSE of $\bar{\theta}_{k}$ is dominated by the asymptotic bias,
i.e., $\E\|\bar{\theta}_{k}-\theta_{*}\|^{2}\approx\|\E\bar{\theta}_{\infty}-\theta^{*}\|^{2}=\|\E\theta_{\infty}^{(\alpha)}-\theta^{*}\|^{2}$.
Our second main result establishes that the asymptotic bias is proportional
to the stepsize $\alpha$ (up to a second order term): 
\begin{align}
\E\theta_{\infty}^{(\alpha)}-\theta^{*}=\alpha B^{(1)}+\bigO(\alpha^{2}),\label{eq:bias_intro}
\end{align}
where $B^{(1)}$ is a vector independent of $\alpha$ and admits an
explicit expression in terms of $A,b$ and the transition kernel $P$
of the underlying Markov chain $(x_{k})_{k\ge0}$. Crucially, the
first order term in equation~\eqref{eq:bias_intro} is an \emph{equality}
rather than an upper bound. This equality implies that the asymptotic
bias is not affected by averaging the LSA iterates.

\paragraph*{Bias expansion and extrapolation.}

While the asymptotic bias persists under Polyak-Ruppert averaging,
the equality~\eqref{eq:bias_intro} implies that bias can be reduced
using a simple technique called Richardson-Romberg (RR) extrapolation:
run LSA with two stepsizes $\alpha$ and $2\alpha$, compute the respective
averaged iterates $\bar{\theta}_{k}^{(\alpha)}$ and $\bar{\theta}_{k}^{(2\alpha)}$,
and output their linear combination $\widetilde{\theta}_{k}^{(\alpha)}:=2\bar{\theta}_{k}^{(\alpha)}-\bar{\theta}_{k}^{(2\alpha)}$.
Doing so cancels out the leading term in the bias characterization~\eqref{eq:bias_intro},
resulting in an order-wise smaller bias $\E\widetilde{\theta}_{\infty}^{(\alpha)}-\theta^{*}=\bigO(\alpha^{2}).$

In fact, the bias characterization~\eqref{eq:bias_intro} can be
generalized to higher orders. We establish that the bias admits the
following \emph{infinite series expansion}: 
\begin{align}
\E\theta_{\infty}^{(\alpha)}-\theta^{*}=\alpha B^{(1)}+\alpha^{2}B^{(2)}+\alpha^{3}B^{(3)}+\cdots,\label{eq:bias_inf_intro}
\end{align}
where the vectors $\{B^{(i)}\}$ are independent of $\alpha$. Consequently,
RR extrapolation can be executed with $m\ge2$ stepsizes to eliminate
the $m-1$ leading terms in equation~\eqref{eq:bias_inf_intro},
reducing the bias to the order $\bigO(\alpha^{m}).$\\

Put together, the above results show that the combination of constant stepsize, averaging, and extrapolation allows one to approach the
\emph{best of three worlds}: (a) using a constant stepsize leads to
fast, exponential-in-$k$ convergence of the optimization error, (b)
tail-averaging eliminates the variance at an (optimal) $1/k$ rate,
and (c) RR extrapolation order-wise reduces the asymptotic bias. We
highlight that the $m$ iterate sequences used in RR extrapolation
can be computed in parallel, using the same data stream $(x_{k})_{k\ge0}$.
Compared with standard LSA, the above combined procedure is data efficient
(in terms of the sample complexity $k$ for achieving a given MSE),
does not require sophisticated tuning of the stepsize, and incurs
a minimal increase in the computational cost.

The results above should be contrasted with the setting of LSA with
\emph{i.i.d.\ data}, where the $x_{k}$'s are sampled independently
from some distribution $\pi$. In this setting, it has been shown
(sometimes implicitly) in existing work that the asymptotic bias is
zero \citep{Lakshminarayanan18-LSA-Constant-iid,Mou20-LSA-iid}. Similar
results are known for SGD applied to a quadratic function with gradients
contaminated by independent noise \citep{bach2013,Dieuleveut20-bach-SGD}.
It is perhaps surprising that using Markovian data leads to a non-zero
bias, even when the LSA iteration is linear in $\theta_{k}$. To provide
intuition, we plot the dependency graphs for LSA with i.i.d.\ data
and Markovian data in Figure~\ref{fig:dag-lsa}. It can be seen that
in the Markovian setting, the correlation between $x_{k}$'s leads
to additional correlation among $\theta_{k}$'s; in particular, the
iterate sequence $(\theta_{k})_{k\ge0}$ is no longer a Markov chain
by itself. As such, $\theta_{k+1}$ has an implicit, \emph{nonlinear}
dependence on $\theta_{k}$ through $(x_{k-1},x_{k})$. This non-linearity
is the source of the asymptotic bias.

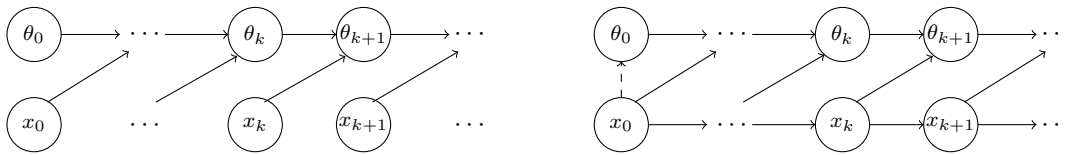
\begin{figure}[H]
\centering \subfigure{ \begin{tikzpicture}[scale=0.6]
    \draw (-2.5, 0) circle (.6) node{\footnotesize $x_{0}$};
    \draw (-2.5, 2) circle (.6) node{\footnotesize $\theta_0$};
    \draw[->] (-1.9, 2) -- (-0.6,2);
    \draw[->](-2.2, 0.5) -- (-0.4,1.62);

    \draw (0, 0) node{$\cdots$};
    \draw[->](0.2,0.5) -- (2.0,1.52);

    \draw (2.4, 0) circle (.6) node{\footnotesize $x_{k}$};
    \draw[->](2.6,0.5) -- (4.4,1.52);
    
    \draw (4.8, 0) circle (.6) node{\footnotesize $x_{k+1}$};
    \draw[->](5,0.5) -- (6.9,1.62);
    
    \draw (7.2, 0) node{$\cdots$};
    
    \draw (0, 2) node{$\cdots$};
    \draw[->] (0.4,2) -- (1.8,2);
    
    \draw (2.4, 2) circle (.6) node{\footnotesize $\theta_k$};
    \draw[->] (3,2) -- (4.2,2);
    
    \draw (4.8, 2) circle (.6) node{\footnotesize $\theta_{k+1}$};
    \draw[->] (5.4,2) -- (6.7,2);
    
    \draw (7.2, 2) node{$\cdots$};
    \end{tikzpicture} } \hspace{0.8cm} \subfigure{ \begin{tikzpicture}[scale=0.6]
    \draw (-2.5, 0) circle (.6) node{\footnotesize $x_{0}$};
    \draw (-2.5, 2) circle (.6) node{\footnotesize $\theta_0$};
    \draw[->] (-1.9, 0) -- (-0.6,0);
    \draw[->] (-1.9, 2) -- (-0.6,2);
    \draw[dashed,->](-2.5,0.6) -- (-2.5,1.4);
    \draw[->](-2.2, 0.5) -- (-0.4,1.62);

    \draw (0, 0) node{$\cdots$};
    \draw[->] (0.4,0) -- (1.8,0);
    \draw[->](0.2,0.5) -- (2.0,1.55);

    \draw (2.4, 0) circle (.6) node{\footnotesize $x_{k}$};
    \draw[->] (3.0,0) -- (4.2,0);
    \draw[->](2.7,0.5) -- (4.4,1.52);
    
    \draw (4.8, 0) circle (.6) node{\footnotesize $x_{k+1}$};
    \draw[->] (5.4,0) -- (6.7,0);
    \draw[->](5.1,0.5) -- (6.9,1.62);
    
    \draw (7.2, 0) node{$\cdots$};
    
    \draw (0, 2) node{$\cdots$};
    \draw[->] (0.4,2) -- (1.8,2);
    
    \draw (2.4, 2) circle (.6) node{\footnotesize $\theta_k$};
    \draw[->] (3.0,2) -- (4.2,2);
    
    \draw (4.8, 2) circle (.6) node{\footnotesize $\theta_{k+1}$};
    \draw[->] (5.4,2) -- (6.7,2);
    
    \draw (7.2, 2) node{$\cdots$};
    \end{tikzpicture} } \caption{Dependency graphs of LSA. \textit{Left}: i.i.d.\ data. \textit{Right}:
Markovian data. }
\label{fig:dag-lsa} 
\end{figure}

\paragraph*{Bias and mixing time.}

We quantify the observations above by relating the magnitude of the
asymptotic bias to the mixing time of the underlying Markov chain
$(x_{k})_{k\ge0}$ and the absolute spectral gap of the transition
kernel $P$, denoted by $\absgap(P)$. 
We show that the leading coefficient $B^{(1)}$ in the expansion~\eqref{eq:bias_inf_intro}
has a norm upper bounded by $\bigO\big(\frac{1-\absgap(P)}{\absgap(P)}\big)$.
It is well known that the mixing time of $(x_{k})_{k\ge0}$ can be
tightly upper and lower bounded by functions of $\absgap(P)$ \citep{Paulin2015,Levin17-mixing_book}.
Consequently, the faster the underlying chain $(x_{k})_{k\ge0}$ mixes,
the smaller the asymptotic bias is. As a special case, LSA with i.i.d.\ data
has zero mixing time and $1-\absgap(P)=0$, hence zero bias.\\

Our results hold for LSA driven by a Markov chain $(x_{k})_{k\ge0}$
on a general (possibly continuous) state space. These results immediately
apply to Markovian settings of the TD algorithm with linear function
approximation and SGD for quadratic functions. Furthermore, we provide
numerical results for LSA, TD and SGD, which corroborate our theory
and demonstrate the benefit of using constant stepsizes, tail averaging,
and RR extrapolation.

\paragraph*{Paper Organization:}

In Section~\ref{sec:related}, we review existing results related
to our work. After setting up the problem and assumptions in Section~\ref{sec:setup},
we present our main results in Section~\ref{sec:main}. In Section~\ref{sec:experiments},
we provide numerical results for LSA, TD and SGD. We outline the proofs of the main results in Section~\ref{sec:proof_sketch}. The paper
is concluded in Section~\ref{sec:conclusion} with a discussion of
future directions. The proofs of our theoretical results are provided
in the Appendix.

\section{Related Work}

\label{sec:related}

In this section, we review existing results that are most related
to our work. The literature on SA and SGD is vast. Here we mainly
discuss prior works in the non-asymptotic and constant stepsize regime,
with a focus on the Markovian noise setting.

\subsection{Classical Results on SA and SGD}

The study of stochastic approximation and stochastic gradient descent
dates back to the seminal work of Robbins and Monro \citep{Robbins51-Monro-SA}.
Convergence results in classical works typically assume that the stepsize
sequence $(\alpha_{k})_{k\geq1}$ satisfies $\sum_{k=1}^{\infty}\alpha_{k}=\infty\text{ and }\sum_{k=1}^{\infty}\alpha_{k}^{2}<\infty.$
This assumption implies that the stepsize sequence is diminishing.
Under suitable conditions, Robbins and Monro prove that SA and SGD
algorithms asymptotically converge in the $L^{2}$ sense \citep{Robbins51-Monro-SA},
and Blum shows that the convergence holds almost surely \citep{Blum54-SA}.
Subsequent works \citep{Ruppert88-Avg,polyak90_average} propose the
technique now known as the Polyak-Ruppert (PR) averaging. A Central Limit Theorem (CLT) for the asymptotic normality of the averaged iterates
is established in \citep{Polyak92-Avg}. Borkar and Meyn \citep{borkar2000-ode-sa}
introduce the Ordinary Differential Equation (ODE) technique for analyzing
SA algorithms. Utilizing the ODE technique, the recent work \citep{Meyn21_ode}
establishes a functional CLT for SA driven by Markovian noise.

The asymptotic theory of SA and SGD is well-developed and covered
in several excellent textbooks~\citep{kushner2003-yin-sa-book,borkar08-SA-book,Benveniste12-sa-book,WrightRecht2022_OptBook}.
Our work, in comparison, focuses on the setting of constant stepsizes
and provides non-asymptotic bounds.

\subsection{SA and SGD with Constant Stepsizes}

A growing body of recent works considers the constant stepsize setting
of SA and the closely related SGD algorithm. A majority of works in this line assume i.i.d.\ noise, with a number of finite-time results.
The work in \citep{Lakshminarayanan18-LSA-Constant-iid} analyzes LSA and establishes finite-time upper and lower bounds on the MSE.
The work \citep{Mou20-LSA-iid} provides refined results with the
optimal dependence on problem-specific constants, as well as a CLT
for the averaged iterates with an exact characterization of the asymptotic
covariance matrix. Using new results on random matrix products, the
work \citep{durmus2021-LSA} establishes tight concentration bounds
of LSA, which are extended to LSA with iterate averaging in \citep{durmus22-LSA}. 

Closely related to our work is \citep{Dieuleveut20-bach-SGD}, which
studies constant stepsize SGD for strongly convex and smooth functions.
By connecting SGD to time-homogeneous Markov chains, they establish
that the iterates converge to a unique stationary distribution. This
result is generalized to non-convex and non-smooth functions with
quadratic growth in the work \citep{Yu21-stan-SGD}, which further
establishes asymptotic normality of the averaged SGD iterates. Subsequent
work \citep{chen21-siva_asymptotic} studies the limit of the stationary
distribution as stepsize goes to zero. Note that these aforementioned
results are established under the i.i.d.\ noise setting.

More recent works study constant-stepsize SA and SGD under Markovian
noise. The work \citep{srikant-ying19-finite-LSA} provides finite-time
bounds on the MSE of LSA. The work \citep{Mou21-optimal-linearSA}
considers LSA with averaging and establishes instance-dependent MSE
upper bounds with tight dimension dependence. The papers \citep{srikant-ying19-finite-LSA,durmus22-LSA}
establish bounds on higher moments of LSA iterates. Going beyond linear
SA, the work \citep{chen20-contract-SA} considers general SA with
contractive mapping and provides finite-time convergence results.
The work \citep{guy2020} studies SGD for linear regression problems
with Markovian data and constant stepsizes. Most of these results
focus on the upper bounds of the MSE and do not decouple the effect of
the asymptotic bias. 

A portion of our results are similar in spirit to \citep[Proposition 2]{Dieuleveut20-bach-SGD}
and \citep[Theorem 3]{durmus2021-LSA}, in that we both study LSA
and SGD with constant stepsizes in the lens of Markov chain analysis.
A crucial difference is that we consider the Markovian data setting
whereas they consider i.i.d.\ data. Arising naturally in stochastic
control and RL problems, the Markovian setting leads to non-zero asymptotic
bias and new analytical challenges, which are not present in the i.i.d.\ setting.
In particular, our analysis involves more delicate coupling arguments
and builds on the Lyapunov function techniques from \citep{srikant-ying19-finite-LSA}.
Along the way, we obtain a refinement of the MSE bounds from the work
\citep{srikant-ying19-finite-LSA}. We discuss these analytical challenges
and improvements in greater detail after stating our theorems; see
Sections~\ref{sec:main} and~\ref{sec:proof_sketch}.

\subsection{Applications in Reinforcement Learning and TD Learning}

Many iterative algorithms in RL solve for the fixed point of Bellman
equations and can be viewed as special cases of SA~\citep{Sutton18-RL-book,Bertsekas19-RL-book}.
The TD algorithms \citep{Sutton1988-td} with linear function approximation,
including TD(0) and more generally TD($\lambda$), are LSA procedures.
Our results can be specialized to TD learning and hence are related
to existing works in this line.

Classical results on TD Learning, similarly to those on SA, focus
on asymptotic convergence under diminishing stepsizes \citep{Sutton1988-td,Dayan92-tdlambda,Dayan1994-tdlambda,Tsitsiklis97-td_paper}.
More recent works provide finite-time results. The work \citep{Dalal18-lineartd0}
is among the first to provide MSE and concentration bounds for linear
TD learning in its original form, and their analysis assumes diminishing
stepsize and i.i.d.\ noise. The work \citep{Bhandari21-linear-td}
presents finite-time analysis of TD(0) under both i.i.d.\ and Markovian
noise, with both diminishing and constant stepsizes. Their results
require a projection step in TD(0) to ensure boundedness. The
Lyapunov analysis in \citep{srikant-ying19-finite-LSA} on LSA, when
specialized to TD(0), removes this projection step and proves upper
bounds on the MSE. The recent work in \citep{chen21-offpolicy,chen21-finite-td}
uses Lyapunov theory to study the tabular TD and obtains finite sample
convergence guarantees. The paper \citep{Khamaru21-td-instance} provides
sharp, instance-dependent $\ell_{\infty}$ error bounds for the tabular
TD algorithm with i.i.d.\ data.

We mention in passing that Q-learning~\citep{Watkins92-QLearning},
another standard algorithm in RL, can be viewed as a nonlinear SA
procedure with contractive mappings. Q-learning has been studied in
both classical and recent work, e.g., \citep{Tsitsiklis1994-QLearn,Szepesvari97-QLearn-Rates,EvenDar04-QLearn-rates}
and \citep{chen21-finite-td,Chandak22-QLearn}. Generalizing our results
to nonlinear SA and Q-learning is an interesting future direction.

\section{Set-up and Assumptions}
\label{sec:setup}

We formally set up the problem and introduce the assumptions and notations used
in the sequel.

\subsection{Problem Set-up}
\label{sec:problem_setup} 

Let $(x_{k})_{k\geq0}$ be a Markov chain on a general state space
$\cX$. Consider the following linear stochastic approximation iteration:
\begin{equation}
\theta_{k+1}^{(\alpha)}=\theta_{k}^{(\alpha)}+\alpha\left(A(x_{k})\theta_{k}^{(\alpha)}+b(x_{k})\right),\quad k=0,1,\ldots,\label{eq:update-rule}
\end{equation}
where $A:\cX\to\R^{d\times d}$ and $b:\cX\to\R^{d}$ are deterministic
functions, and $\alpha>0$ is a constant stepsize. In what follows,
we omit the superscript in $\theta_{k}^{(\alpha)}$ when the dependence
on $\alpha$ is clear from the context. The initial distribution of
$\theta_{0}$ may depend on $x_{0}$ but is independent of $(x_{k})_{k\ge1}$
given $x_{0}$.  

Let $\pi$ be the stationary distribution of the Markov chain $(x_{k})$
and define the shorthands 
\begin{equation}
\BarA:=\E_{\pi}[A(x)]\in\real^{d\times d}\quad\text{and}\quad\bar{b}:=\E_{\pi}[b(x)]\in\real^{d},
\label{eq:bar-limit}
\end{equation}
where $\E_{\pi}[\cdot]$ denotes the expectation with respect to $x\sim\pi$.
The iterative procedure~\eqref{eq:update-rule} computes an approximation
of the target vector $\theta^{*}$, defined as the solution to the
steady-state equation 
\begin{equation}
\BarA\theta+\bar{b}=0.
\label{eq:steady-state-equation}
\end{equation}
Our general goal is to characterize the relationship between the iterate
$\theta_{k}$ and the target solution $\theta^{*}$.

The stochastic process of the LSA iterates, $(\theta_{k})_{k\ge0}$,
is not a Markov chain itself: given $\theta_{k}$, the random variables
$\theta_{k+1}$ and $\theta_{k-1}$ are correlated through the underlying
Markov process $(x_{0},x_{1},\ldots,x_{k})$. However, direct calculation
verifies that the joint process $(x_{k},\theta_{k})_{k\ge0}$ is a
Markov chain on the state space $\cX\times\R^{d}$. This chain is
time-homogeneous as the stepsize $\alpha$ is independent of $k$. 

\begin{remark} 
\label{rmk:sgd-remark} 
    The LSA iteration~\eqref{eq:update-rule}
    covers as a special case the SGD algorithm applied to minimizing a
    quadratic function $f(\theta)=-\frac{1}{2}\theta^{\top}\BarA\theta-\bar{b}\theta=\E_{\pi}\big[-\frac{1}{2}\theta^{\top}A(x)\theta-b(x)\theta\big],$
    where $-\BarA$ is the symmetric expected Hessian matrix. Note that
    LSA is more general than SGD for quadratic minimization, as $A(x)$
    need not be symmetric, in which case $\theta\mapsto-(\BarA\theta+\Bar b)$
    is not a gradient field.  
\end{remark}

\begin{remark}
\label{rmk:sparsity}
    One primary advantage of LSA (and SGD) is the low computational cost of the iteration \eqref{eq:update-rule}, particularly in forming the matrix-vector product $A(x_k)\theta_k$. For example, in TD(0) the matrix $A(x_k)$ is rank-one (and in addition 2-sparse in the tabular case); see Section~\ref{sec:td-section}. In comparison, the expected matrix $\BarA$, as well as its running empirical estimate $\hat{A}_k:=\frac{1}{k}\sum_{t=0}^{k-1}A(x_t)$, are typically dense and full-rank. 
\end{remark}

\subsubsection{General State Space Markov Chains}

\label{sec:hilbert-notation} 

As the underlying Markov chain $(x_{k})_{k\geq0}$ is on a general
state space $\cX$, we review some relevant concepts and notations.
We assume throughout the paper that the $\cX$ is Borel, i.e., the
$\sigma$-algebra $\cB(\cX)$ is Borel. Let $P$ denote the transition
kernel of the chain. A distribution $\pi$ is the stationary/invariant
distribution if $\int_{\cX}\pi(\dd x)P(x,B)=\pi(B),\forall B\in\cB(\cX)$.
For a function $f:\cX\to\R^{d}$, we write $\pi(f)=\int_{\cX}\pi(\dd x)f(x)$.
Define the $\pi$-weighted inner product $\dotp{f,g}_{\ltwopi}=\int_{\cX}\pi(\dd x)f^{\top}(x)g(x)$
and the induced norm $\|f\|_{\ltwopi}=(\dotp{f,f}_{\ltwopi})^{1/2}$.
Let $\ltwopi=\{f:\|f\|_{\ltwopi}<\infty\}$ denote the corresponding Hilbert space of $\real^{d}$-valued, square-integrable, and measurable
functions on $\cX$.\footnote{\label{fn:quotient_space}As customary, two functions $f$ and $g$
are identified as the same element in $\ltwopi$ if they are equal
$\pi$-almost everywhere, i.e, $\|f-g\|_{\ltwopi}=0$. } For an operator $T:\ltwopi\to\ltwopi$, its operator norm is defined
as $\|T\|_{\ltwopi}=\sup_{\|f\|_{\ltwopi}=1}\|Tf\|_{\ltwopi}$. The
transition kernel $P$ is a bounded linear operator on $\ltwopi$
with norm $\|P\|_{\ltwopi}=1$. Also, define the kernel/operator $\Pi=1\otimes\pi$
by $\Pi(x,\cdot)=\pi$; equivalently $(\Pi f)(x)=\pi f,\forall x\in\cX$.

When the state space $\cX$ is Borel, there exists a kernel $P^{\ast}$
as a \emph{regular} conditional probability that satisfies $\int_{A}\pi(\dd x)P(x,B)=\int_{B}\pi(\dd y)P^{\ast}(y,A),$
$\forall A,B\in\cB(\cX)$, and $P^{\ast}$ defines the probability
law for the time-reversed chain of $(x_{k})_{k\ge0}$ \citep[Chapter 21.4, Theorem 19]{fristedt1997}.
Moreover, $P^{\ast}$ is the adjoint operator to $P$ in $\ltwopi$,
i.e., $\dotp{f,Pg}_{\ltwopi}=\dotp{P^{\ast}f,g}_{\ltwopi}$. The Markov
chain $(x_{k})_{k\ge0}$ is called \emph{reversible} with respect
to $\pi$ if $P$ is self-adjoint, i.e., $\Padj=P$.

Define the spectrum of $P$ as 
\[
\spec(P)=\Big\{\lambda\in\C\backslash0:(\lambda I-P)^{-1}\text{ does not exist as a bounded linear operator on }L^{2}(\pi)\Big\}.
\]
The set $\spec(P)$ contains the eigenvalues of $P$. The absolute
spectral gap of $P$ is defined as 
\begin{equation}
\absgap(P)=\begin{cases}
1-\sup\{|\lambda|:\lambda\in\spec(P),\lambda\neq1\} & \text{if eigenvalue \ensuremath{1} has multiplicity \ensuremath{1}},\\
0 & \text{otherwise}.
\end{cases}\label{eq:absgap-def}
\end{equation}
When $P$ has a unique invariant distribution $\pi$, the eigenvalue
$1$ has multiplicity $1$ \citep[Proposition 22.1.2]{Douc2018} and
hence $\absgap(P)>0$. When $P$ is reversible with respect to $\pi$,
$\spec(P)$ lies on the real line and we have the expression $\absgap(P)=1-\|P-\Pi\|_{\ltwopi}.$

We remark that all definitions above coincide with the familiar ones
when the state space $\cX$ is finite. For example, we have $\absgap(P)=1-|\lambda_{2}(P)|$,
where $|\lambda_{2}(P)|$ is the second largest eigenvalue modulus
(SLEM) of the transition probability matrix $P$. 

\subsection{Assumptions}

\label{sec:assumption-section}

We now state the assumptions needed for our main theorems.

\begin{assumption} \label{assumption:uniform-ergodic} $(x_{k})_{k\geq0}$
is a uniformly ergodic Markov chain on a Borel state space $(\cX,\cB(\cX))$
with transition kernel $P$ and a unique stationary distribution $\pi$.
The initial state $x_{0}$ is drawn from~$\pi$. \end{assumption}

Recall that a Markov chain is called uniformly ergodic if $\sup_{x\in\cX}\|P^{n}(x,\cdot)-\pi\|_{\TV}\to0$
as $n\to\infty$, where $\|\cdot\|_{\TV}$ is the total variation
norm. A uniformly ergodic chain satisfies the following seemingly
stronger condition \citep[Theorem 16.0.2]{Meyn12_book}: there exist
constants $r\in[0,1)$ and $R>0$ such that 
\begin{equation}
\sup_{x\in\cX}\big\| P^{k}(x,\cdot)-\pi\big\|_{\TV}\leq Rr^{k};\label{eq:geom-mix-rate}
\end{equation}
that is, the chain converges to $\pi$ from any initial $x_{0}$ at
a uniform geometric rate. All irreducible, aperiodic, and finite state
space Markov chains are uniformly ergodic. Uniform ergodicity also
allows for the chain to have transient states in addition to a single
recurrent class. The uniform ergodicity assumption is used in the
prior work~\citep{Bhandari21-linear-td,dong22-sgd,durmus22-LSA}.
It is possible to further relax this assumption (e.g., as in~\citep{srikant-ying19-finite-LSA,Mou21-optimal-linearSA});
we do not pursue this direction in this paper.

The additional stationarity assumption $x_{0}\sim\pi$, which has
been used in a number of previous papers~\citep{Bhandari21-linear-td,Mou21-optimal-linearSA},
is imposed merely to simplify several mathematical expressions. This assumption is not essential; it can be removed by applying our analysis
to the joint Markov chain $(x_{k},\theta_{k})_{k\geq0}$ after the
marginal $(x_{k})_{k\geq0}$ has approximately mixed, which happens
quickly thanks to the geometric mixing property~\eqref{eq:geom-mix-rate}.

Our following two assumptions are similar to those in the work~\citep{srikant-ying19-finite-LSA,durmus22-LSA}.
Let $\|\cdot\|$ denote the Euclidean norm for vectors and the spectral
norm (i.e., the largest singular value) for matrices.

\begin{assumption} \label{assumption:bounded} It holds that 
\[
\Amax:=\sup_{x\in\cX}\|A(x)\|\leq1\quad\text{and}\quad\bmax:=\sup_{x\in\cX}\|b(x)\|<\infty.
\]
\end{assumption}

Assumption~\ref{assumption:bounded} implies the bounds $\|\BarA\|\leq\Amax\leq1$
and $\|\Bar b\|\leq\bmax$. The constant $1$ in the assumption is chosen for convenience and can be relaxed to other positive constants
by rescaling the LSA update (\ref{eq:update-rule}).

Playing an important role in our analysis is the mixing time of the
Markov chain $(x_{k})_{k\ge0}$ with respect to the functions $A(\cdot)$
and $b(\cdot)$, defined as follows. 

\begin{definition} For $\epsilon\in(0,1)$, the $\epsilon$-mixing
time of $(x_{k})_{k\geq0}$ with respect to $(A,b)$ is defined to
be the smallest number $\tau_{\epsilon}\geq1$ satisfying 
\begin{align}
\big\|\E[A(x_{k})\mid x_{0}=x]-\BarA\big\| & \leq\epsilon\cdot\Amax,\quad\forall x\in\cX,\,\forall k\geq\tau_{\epsilon},\label{eq:a-mix-time}\\
\big\|\E[b(x_{k})\mid x_{0}=x]-\bar{b}\big\| & \leq\epsilon\cdot\bmax,\quad\forall x\in\cX,\,\forall k\geq\tau_{\epsilon}.\label{eq:b-mix-time}
\end{align}
\end{definition}

Under Assumptions~\ref{assumption:uniform-ergodic} and~\ref{assumption:bounded},
the $\epsilon$-mixing time satisfies $\tau_{\epsilon}\leq K\log\frac{1}{\epsilon}$
for all $\epsilon\in(0,1)$, where the number $K\geq1$ is independent
of $\epsilon$. This fact can be seen in the inequality 
\[
\big\|\E[A(x_{k})\mid x_{0}=x]-\BarA\big\|\leq\Amax\left(2\sup_{x\in\cX}\|P^{k}(x,\cdot)-\pi\|_{\TV}\right)\leq2\Amax Rr^{k},
\]
where the last step follows from equation~\eqref{eq:geom-mix-rate};
a similar argument applies to $b(x_{k})$.

In the sequel, unless specified otherwise, we always choose $\epsilon=\alpha$
and write $\tau\equiv\tau_{\alpha}$.

\begin{assumption} \label{assumption:hurwitz} The matrix $\BarA$
is Hurwitz, i.e., all eigenvalues have strictly negative real parts.
\end{assumption}

Assumption~\ref{assumption:hurwitz} is standard in the study of
the stability of dynamical systems. Under this assumption, it is well
known that there exists a symmetric positive definite matrix $\Gamma$
satisfying the Lyapunov equation $\BarA^{\top}\Gamma+\Gamma\BarA=-I,$
where $I$ is the $d$-by-$d$ identity matrix. Denote by $\gammin$
and $\gammax$ the minimum and maximum eigenvalues of the matrix $\Gamma$
respectively. We have $\gammin>0$ and 
\begin{equation}
\gammin\|v\|^{2}\leq v^{\top}\Gamma v\leq\gammax\|v\|^{2},\quad\forall v\in\R^{d}.\label{eq:gamma-property}
\end{equation}
 Moreover, the matrix $\BarA$ is invertible and with smallest singular
value $\smin{\BarA}>0$, and the target solution $\theta^{*}$ to
steady-state equation (\ref{eq:steady-state-equation}) is unique.

\subsection{Notations}

\label{sec:notations}

In general, we adopt the notational convention that upper-case letters
(e.g., $A$) denote matrices and lowercase letters (e.g., $b$) denote
vectors or scalars. The lowercase letter $c$ and its derivatives
$c',c_{0}$, etc.\ denote universal numerical constants, whose values
may change from line to line. Recall that $\|\cdot\|$ denotes the Euclidean norm for vectors and the spectral norm for matrices, and
$\left\Vert \cdot\right\Vert _{\ltwopi}$ denotes the norm on $\ltwopi$
and the induced operator norm.

We generally use $B\equiv B(A,b,P)$ and its derivatives to denote
quantities (vectors or matrices) that depend only on $A,b$, and $P$,
but independent of the stepsize $\alpha$ and the iteration index
$k$. As we are primarily interested in how various quantities scale
with $\alpha$ and $k$, we make use of the following big-O notation:
for a quantity $h$ that may depend on all problem parameters, we write $h=\bigO(f(\alpha,k))$
if it holds that $\|h\|\leq B(A,b,P)\cdot f(\alpha,k)$ for some $B(A,b,P)$
independent of $\alpha$ and $k$, where $f$ is a function of $\alpha$
and $k$. For example, $h=\bigO(\alpha/k)$ means $\|h\|\le B(A,b,P)\cdot\alpha/k$.
In addition, when presenting results on the relationship to mixing
time, we use $C\equiv C(A,b,\pi)$ to denote a quantity that may depend
on $\pi$ but not other properties of $P$ (e.g., its mixing time
and spectral gap). 

We use $\law(z)$ to denote the law of a random variable $z$, and
$\var(z)$ its covariance matrix. We write $z_{1}\indep z_{2}\mid z_{3}$
if the random variables $z_{1}$ and $z_{2}$ are conditionally independent
given $z_{3}$. Let $\mathcal{P}_{2}(\R^{d})$ be the space of square-integrable
distributions on $\R^{d}$. Let $\cP_{2}(\cX\times\R^{d})$ be the
set of distributions $\Bar{\nu}$ on $\cX\times\R^{d}$ with the property
that the marginal of $\Bar{\nu}$ on $\R^{d}$ is square-integrable. 

\section{Main Results}

\label{sec:main}

In this section, we present our main results. In Section~\ref{sec:converge-to-limit},
we study the convergence of the LSA iterates $(x_{k},\theta_{k})_{k\ge0}$
to a unique limiting distribution. In Section~\ref{sec:bias-expansion},
we characterize the above limit and its relationship with the stepsize
and mixing time. We explore the implications of these results for
PR tail averaging and RR extrapolation in Section~\ref{sec:implications}.
We apply our results to TD(0) Learning in Section~\ref{sec:td-section}
and to SGD in Section~\ref{sec:sgd-implication}.

\subsection{Convergence to Limit Distribution}

\label{sec:converge-to-limit}

Our convergence results are based on the Wasserstein distance \citep{Villani08-ot_book}.
The Wasserstein distance of order 2 between two probability measures
$\mu$ and $\nu$ in $\mathcal{P}_{2}(\R^{d})$ is defined as
\begin{align*}
W_{2}(\mu,\nu) & =\inf_{\xi\in\Pi(\mu,\nu)}\bigg(\int_{\R^{d}}\|u-v\|^{2}\,\ddup\xi(u,v)\bigg)^{1/2}\\
 & =\inf\Big\{\left(\E[\|\theta-\theta'\|^{2}]\right)^{1/2}:\;\law(\theta)=\mu,\;\law(\theta')=\nu\Big\},
\end{align*}
where $\Pi(\mu,\nu)$ denotes the set of all couplings between $\mu$
and $\nu$, i.e., the collection of joint distributions in $\mathcal{P}_{2}(\R^{d}\times\R^{d})$
with marginal distributions $\mu$ and $\nu$. To study the joint
process $(x_{k},\theta_{k})_{k\ge0}$, we extend the above distance
to the space $\cP_{2}(\cX\times\R^{d})$. Under the uniform ergodicity
Assumption~\ref{assumption:uniform-ergodic}, it is natural to metricize
the space $\cX$ with the discrete metric $d_{0}$, where $d_{0}(x,x'):=\indic\{x\neq x'\}$.
We then define the following metric $\Bar d$ on the product space
$\cX\times\R^{d}$: 
\[
\Bar d\big((x,\theta),(x',\theta')\big):=\sqrt{d_{0}(x,x')+\|\theta-\theta'\|^{2}}.
\]
For a pair of distributions $\Bar{\mu}$ and $\Bar{\nu}$ in $\cP_{2}(\cX\times\R^{d}),$
we consider the Wasserstein-2 distance w.r.t.\ $\Bar d$: 
\begin{equation}
\begin{aligned}\Bar W_{2}(\Bar{\mu},\Bar{\nu}) & =\inf\bigg\{\Big(\E\big[d_{0}(x,x')+\|\theta-\theta'\|^{2}\big]\Big)^{1/2}:\;\law\big((x,\theta)\big)=\Bar{\mu},\;\law\big((x',\theta')\big)=\Bar{\nu}\bigg\}.\end{aligned}
\label{eq:w2-definition-extended}
\end{equation}
It follows immediately from definition that $W_{2}\big(\law(\theta),\law(\theta')\big)\le\Bar W_{2}\big(\law(x,\theta),\law(x',\theta')\big).$
Note that convergence in $W_{2}$ or $\Bar W_{2}$ implies the usual
convergence in distribution plus the convergence of the first two
moments \citep[Definition 6.8 and Theorem 6.9]{Villani08-ot_book}.

Our first theorem establishes the convergence of the Markov chain
$(x_{k},\theta_{k})_{k\geq0}$ in $\Bar W_{2}$ to a unique stationary
distribution and characterizes the geometric convergence rate.

\begin{thm} \label{thm:thm-converge} Suppose that Assumptions~\ref{assumption:uniform-ergodic},~\ref{assumption:bounded}
and~\ref{assumption:hurwitz} hold, and the stepsize $\alpha$ satisfies
\begin{equation}
\alpha\tau_{\alpha}<\frac{0.05}{95\gammax}.\label{eq:alpha-constraint}
\end{equation}

\begin{enumerate}
\item Under all initial distributions of $\theta_{0}$, the sequence of
random variables $(x_{k},\theta_{k})_{k\geq0}$ converges in $\Bar{W_{2}}$
to a unique limit $(x_{\infty},\theta_{\infty})\sim\Bar{\mu}$. Moreover,
it holds that 
\begin{equation}
\tr(\var(\theta_{\infty}))\le\alpha\tau_{\alpha}\kappa,\qquad\text{where }\kappa:=720\cdot\frac{\gammax^{2}}{\gammin}\cdot s_{\min}^{-2}(\BarA)\bmax^{2}.\label{eq:kappa-def}
\end{equation}
\item $\Bar{\mu}$ is the unique stationary distribution of the Markov chain
$(x_{k},\theta_{k})_{k\geq0}$. 
\item Let $\mu:=\law(\theta_{\infty})$ be the second marginal of $\Bar{\mu}$.
For all $k\ge\tau_{\alpha}$, it holds that 
\begin{equation}
W_{2}^{2}\Big(\law(\theta_{k}),\mu\Big)\le\bar{W}_{2}^{2}(\law(x_{k},\theta_{k}),\Bar{\mu})\leq20\,\frac{\gammax}{\gammin}\Big(\E[\|\theta_{0}-\E[\theta_{\infty}]\|^{2}]+\tr(\var(\theta_{\infty}))\Big)\cdot\left(1-\frac{0.9\alpha}{\gammax}\right)^{k}.\label{eq:w2-thetak-to-mu}
\end{equation}
\end{enumerate}
\end{thm}

We outline the proof of Theorem~\ref{thm:thm-converge} in Section~\ref{sec:thm-converge-proof-sketch},
deferring the complete proof to Appendix~\ref{sec:proof_thm_conv_limit_dist}.
This convergence result is valid under the stepsize condition~\eqref{eq:alpha-constraint},
stated as an upper bound on the product $\alpha\tau_{\alpha}$. Since
$\tau_{\alpha}\leq K\log\frac{1}{\alpha}$ for some constant $K\geq1$
independent of $\alpha$ (see Section~\ref{sec:assumption-section}),
the condition~\eqref{eq:alpha-constraint} is satisfied for sufficiently
small $\alpha$. Note that the limiting distribution $\Bar{\mu}$
is in general not a product distribution of its marginals $\pi$ and
$\mu$.

We remark on the techniques for proving Theorem~\ref{thm:thm-converge}.
To establish the convergence of a Markov chain to a unique stationary
distribution, a standard approach is to show that the chain is positive
recurrent by verifying irreducibility and Lyapunov drift conditions.
This approach has been developed for Markov chains on general state
spaces~\citep{Meyn12_book} and is adopted in the prior work \citep{Yu21-stan-SGD,Meyn21_ode,meyn22-meanshift}.
However, it is unclear how to implement this approach for the general
LSA iteration~\eqref{eq:update-rule}. For example, suppose that
the stepsize $\alpha$ and the functions $A$ and $b$ take on rational
values. If the initial $\theta_{0}$ is rational, then $\theta_{k}$
only takes rational values for all $k\ge0$. If $\theta_{0}$ is irrational,
then $\theta_{k}$ remains irrational. As such, it seems challenging
to certify $\psi$-irreducibility and recurrence for the chain $(x_{k},\theta_{k})_{k\ge0}$.
Instead, we prove weak convergence through the convergence in the
Wasserstein distance, which can be bounded via coupling arguments.
The Wasserstein distance has also been used in the prior work \citep{Dieuleveut20-bach-SGD,guy2020,durmus2021-LSA}
to study SGD and LSA in the i.i.d.\ data setting. Their analysis
relies heavily upon the i.i.d.\ assumption and the one-step contraction property
$W_{2}^{2}(\law(\theta_{k+1}),\mu)<W_{2}^{2}(\law(\theta_{k}),\mu)$.
Establishing this property in our Markovian setting is difficult if
not impossible---we elaborate in Section~\ref{sec:thm-converge-proof-sketch}.
Our proof makes use of a different and more delicate coupling argument.\\

As a corollary of Theorem~\ref{thm:thm-converge}, we obtain geometric
convergence for the first two moments of $\theta_{k}$. 

\begin{cor} \label{cor:non-asymptotic-bounds} Under the setting
of Theorem~\ref{thm:thm-converge}, for all $k\geq\tau_{\alpha}$
we have 
\begin{equation}
\label{eq:first-moment-geometric}
\left\Vert \E[\theta_{k}]-\E[\theta_{\infty}]\right\Vert   \leq C\cdot\left(1-\frac{0.9\alpha}{\gammax}\right)^{k/2}\quad\text{and}\quad
\left\Vert \E\left[\theta_{k}\theta_{k}^{\top}\right]-\E\left[\theta_{\infty}\theta_{\infty}^{\top}\right]\right\Vert   \leq C'\cdot\left(1-\frac{0.9\alpha}{\gammax}\right)^{k/2},
\end{equation}
for some $C\equiv C(A,b,\pi)$ and $C'\equiv C'(A,b,\pi)$ that are
independent of $\alpha$ and $k$. \end{cor} The proof of Corollary~\ref{cor:non-asymptotic-bounds}
is given in Appendix~\ref{sec:non-asymptotic-bounds-proof}.

\subsection{Expansion and Characterization of Bias}

\label{sec:bias-expansion}

Having show that $\theta_{k}^{(\alpha)}$ converges in distribution
to a limit $\theta_{\infty}^{(\alpha)}$, our next theorem characterizes
the asymptotic bias $\E[\theta_{\infty}^{(\alpha)}]-\theta^{*}$.
Some additional notations are needed to write down an explicit expression
for the bias. Recall that $\ltwopi$ is the set of $\real^{d}$-valued
square-integrable functions on $\cX$, $\Padj$ is the adjoint of
$P$ as an operator on $\ltwopi$, and $\Pi=1\otimes\pi$. Let $\dA:\ltwopi\to\ltwopi$
denote the operator given by $(\dA f)(x)=A(x)f(x)$ for each $x\in\cX,f\in\ltwopi$,
with its normalized version $\dAbar$ is given by $(\dAbar f)(x)=\BarA^{-1}A(x)f(x)$.
We define the operator $\Xi:\ltwopi\to\ltwopi$, the function $\Upsilon:\cX\to\real^{d}$
and the vectors $B^{(i)}\in\real^{d},i=1,2,\ldots$ as 
\begin{subequations}
\label{eq:bias-coef-def}
\begin{align}
\Xi & :=(I-\Padj+\Pi)^{-1}(\Padj-\Pi)\dA(I-\Pi\dAbar),\label{eq:xi-def}\\
\Upsilon & :=(I-\Padj+\Pi)^{-1}(\Padj-\Pi)(A\theta^{\ast}+b),\label{eq:upsilon-def}\\
B^{(i)} & :=-\pi\dAbar\Xi^{i-1}\Upsilon.\label{eq:Bi-def}
\end{align}
\end{subequations}
Note that $\Xi$, $\Upsilon$ and $B^{(i)}$ are independent of $\alpha$.

\begin{thm} \label{thm:bias-characterization} 
Suppose that Assumptions~\ref{assumption:uniform-ergodic},~\ref{assumption:bounded}
and~\ref{assumption:hurwitz} hold, and $\alpha$ satisfies equation~\eqref{eq:alpha-constraint}.
Then $\Xi$, $\Upsilon$ and $B^{(i)},i=1,2,\ldots$ are bounded operators
on $\ltwopi$. Moreover:
\begin{enumerate}
\item For each $m=1,2,\ldots$, we have the expansion 
\begin{align}
\E[\theta_{\infty}^{(\alpha)}] & -\theta^{\ast}=\sum_{i=1}^{m}\alpha^{i}B^{(i)}+\bigO(\alpha^{m+1}).\label{eq:bias-char}
\end{align}
\item Suppose in addition that $\alpha<1/\left\Vert \Xi\right\Vert _{\ltwopi}$.\footnote{See equation (\ref{eq:alpha_cond_inf_series}) for an explicit sufficient
condition for $\alpha<1/\left\Vert \Xi\right\Vert _{\ltwopi}$. } We have the infinite series expansion 
\begin{equation}
\E[\theta_{\infty}^{(\alpha)}]-\theta^{\ast}=\sum_{i=1}^{\infty}\alpha^{i}B^{(i)}=-\alpha\pi\dAbar\big(I-\alpha\Xi\big)^{-1}\Upsilon.\label{eq:bias-expansion}
\end{equation}
\end{enumerate}
\end{thm}

Theorem~\ref{thm:bias-characterization} is akin to a Taylor series
expansion of $\E[\theta_{\infty}^{(\alpha)}]$ with respect to the
stepsize $\alpha$. The existence of such an expansion is non-trivial:
$\theta_{\infty}^{(\alpha)}$ is undefined at $\alpha=0$, and it
is not clear a priori whether $\E[\theta_{\infty}^{(\alpha)}]$ is
a differentiable and analytic function of $\alpha$. We emphasize
that equations~\eqref{eq:bias-char} and~\eqref{eq:bias-expansion}
are equalities, hence the bias is non-zero whenever $B^{(i)}\neq0$
for some $i\ge1$. In particular, averaging the LSA iterates $\theta_{k}$
does not affect this bias and only reduces the variance.

The proof of Theorem~\ref{thm:bias-characterization}, outlined in
Section~\ref{sec:thm-bias-proof-sketch} and completed in Appendix~\ref{sec:bias-proof},
is based on the following idea. As discussed in Section~\ref{sec:intro},
the asymptotic bias arises due to the implicit nonlinear dependence
between $\theta_{k+1}$ and $\theta_{k}$ as both of them depend on
the state $x_{k}$ of the underlying Markov chain. If $\theta_{k}$
were independent of $x_{k}$, the bias would be zero. This observation
suggests that the bias is determined by the strength of dependence
between $\theta_{k}$ and $x_{k}$, which can be quantified by the
variation of the conditional expectation $\E[\theta_{k}\mid x_{k}=x]$
as a function of $x\in\cX$. Therefore, our analysis is based on understanding
this conditional expectation in steady state, namely $\E\left[\theta_{\infty}\mid x_{\infty}=x\right]$.
We characterize this quantity using the Basic Adjoint Relationship
(BAR) \citep{Harrison1985brownian,Harrison1987RBM,Dai11-bar-paper}
for the steady state with a specific choice of test functions.

Theorem~\ref{thm:bias-characterization} provides an explicit expression
(\ref{eq:bias-coef-def}) for the coefficients $B^{(i)}$ in the bias
expansion. In Sections \ref{sec:mixing-time} and \ref{sec:zero_bias}
below, we use this expression to further characterize the magnitude
of the bias and its relationship to the mixing time of $(x_{k})_{k\ge0}$.
On the other hand, even without knowing the functional form of $B^{(i)}$,
we can still use Richardson-Romberg extrapolation to cancel out the lower
order terms of $\alpha$ in the expansions~\eqref{eq:bias-char}
and~\eqref{eq:bias-expansion}, reducing the bias to a higher order
term of $\alpha$. These results are presented in Section~\ref{sec:implications}.

\subsubsection{Bias and Mixing Time}

\label{sec:mixing-time}

As mentioned, the bias $\E[\theta_{\infty}]-\theta^{*}$ arises due
to the Markovian correlation in $(x_{k})_{k\ge0}$. If the chain $(x_{k})_{k\ge0}$
mixes quickly, the correlation is weak and intuitively one should
expect a small bias. We now rigorously quantify the relationship between
the bias and the mixing time of $(x_{k})_{k\ge0}$. 

We focus on the setting where the chain $(x_{k})_{k\ge0}$ is reversible,
i.e., $\Padj=P$, and relate the bias to the absolute spectral gap $\absgap(P)$.
The gap $\absgap(P)$ is in turn related to the mixing time $\tau_{\epsilon}$
via
\begin{equation}
\frac{1-\absgap(P)}{\absgap(P)}\cdot K'\log(1/\epsilon)\overset{\text{(i)}}{\leq}\tau_{\epsilon}\overset{\text{(ii)}}{\leq}\frac{1}{\absgap(P)}\cdot K''\log(1/\epsilon),\label{eq:mixing_SLEM}
\end{equation}
where $K'$ and $K''$ are independent of $\epsilon$, inequality
(i) holds for general state space $\cX$, and inequality (ii) is valid
for finite $\cX$ \citep[Proposition 3.3]{Paulin2015}. The theorem
below provides upper bounds on the coefficients $B^{(i)}$ in the
bias expansions~\eqref{eq:bias-char}--\eqref{eq:bias-expansion}
in terms of $\absgap(P)$.

\begin{thm} \label{thm:bias-characterization-reversible}

Suppose that Assumptions~\ref{assumption:uniform-ergodic},~\ref{assumption:bounded}
and~\ref{assumption:hurwitz} hold, $\alpha$ satisfies equation~\eqref{eq:alpha-constraint},
and the Markov chain $(x_{k})_{k\geq0}$ is reversible. For each $i=1,2,\ldots,$
we have 
\[
\|B^{(i)}\|\le\left(C\cdot\frac{1-\absgap(P)}{\absgap(P)}\right)^{i}
\]
for some number $C\equiv C(A,b,\pi)>0$ that depends only on $A,b,\pi$.\footnote{The proof of Theorem~\ref{thm:bias-characterization-reversible}
provides an explicit formula for $C$.} \end{thm}

 Theorem \ref{thm:bias-characterization-reversible} together with
the bias expansion (\ref{eq:bias-expansion}) imply the bound 
\[
\|\E\left[\theta_{\infty}\right]-\theta^{*}\|\le2C\cdot\alpha\frac{1-\absgap(P)}{\absgap(P)}
\]
for a small stepsize $\alpha$. In light of the relationship~\eqref{eq:mixing_SLEM},
we see that the bias is roughly proportional to the product of $\alpha$
and $\tau_{\alpha}.$ In the extreme case where the $x_{k}$'s are
independent with distribution $\pi$, i.e., $P=1\otimes\pi$, we have
$1-\absgap(P)=0$, hence $B^{(i)}=0$ for all $i$ and the bias is zero.
This zero-bias property is implicit in the results in \citep[Theorem 1]{Lakshminarayanan18-LSA-Constant-iid}
and \citep[Theorem 1]{Mou20-LSA-iid}, which study LSA in the i.i.d.\ setting.

Theorem~\ref{thm:bias-characterization-reversible} is proved in
Appendix~\ref{sec:proof-reversible}. The proof uses the spectral
property $1-\left\Vert P-\Pi\right\Vert _{\ltwopi}=\absgap(P)$ of a
reversible kernel $P$. The theorem can be readily extended to the
non-reversible setting by considering the multiplicative reversiblization
$P^{*}P$ and the pseudo-spectral gap, $\absgap_{\text{ps}}(P):=\max_{k\geq1}\big\{\absgap\big({(\Padj)}^{k}P^{k}\big)/k\big\}$
\citep{Paulin2015}. We omit the details.

\begin{remark} 

As the bias is due to the correlation in $x_{k}$'s, one may consider
running LSA with the subsampled (and thus less correlated) data $(x'_{k})_{k\ge0},$
where $x'_{k}:=x_{ck}$ and $c\ge2$ is an integer. Doing so reduces
the mixing time of $x'_{k}$ and in turn the bias by a factor of $c$,
but $c$ times more data is used. The overall effect is essentially
equivalent to using the smaller stepsize $\alpha/c$, as we have shown
that the bias and the exponent of geometric convergence in (\ref{eq:w2-thetak-to-mu})
are both proportional to the stepsize. We emphasize that all our results
including Richardson-Romberg extrapolation can be applied \emph{on
top of} such subsampling or stepsize adjustment.

\end{remark}

\subsubsection{Zero Bias with Markovian Data}

\label{sec:zero_bias}

While Markovian LSA has asymptotic bias in general,
there are important special cases where the bias vanishes even when
$(x_{k})_{k\ge0}$ is Markovian. By Theorem \ref{thm:bias-characterization},
the bias is zero for all $\alpha$ if and only if $B^{(i)}=0$
for all $i$. Thanks to the explicit expression (\ref{eq:bias-coef-def})
for $B^{(i)}$, we see that a sufficient condition for zero bias is
$(\Padj-\Pi)(A\theta^{\ast}+b)=0$. This condition can be expressed
in the more familiar notation of conditional expectation.

\begin{cor} \label{cor:zero-bias-sufficient} Under the assumptions
of Theorem~\ref{thm:bias-characterization}, if 
\begin{equation}
\E\Big[A(x_{k})\theta^{\ast}+b(x_{k})\mid x_{k+1}=x\Big]=0,\quad\forall x\in\cX,\label{eq:zero-bias-sufficient}
\end{equation}
then $\E[\theta_{\infty}]-\theta^{\ast}=0$. \end{cor}

It is clear that if $P$ is reversible, then one may replace $x_{k+1}$ by $x_{k-1}$
in the condition (\ref{eq:zero-bias-sufficient}) and the corollary continues to hold.

The condition (\ref{eq:zero-bias-sufficient}) above is trivially
satisfied when $x_{k}\overset{\text{i.i.d.}}{\sim}\pi$, in which
case $P=\Pi$ and $\E_{x_k\sim\pi}[A(x_k)\theta^{\ast}+b(x_k)]=0$ by definition
of $\theta^{*}$ in~\eqref{eq:steady-state-equation}. Importantly,
it is possible for the condition (\ref{eq:zero-bias-sufficient})
to hold even when $(x_{k})_{k\ge0}$ is correlated. We discuss two
such settings in Sections~\ref{sec:td-section} and \ref{sec:sgd-implication}
on the TD(0) and SGD algorithms.

\subsection{Averaging and Extrapolation}

\label{sec:implications}

We exploit the results above to study the performance of LSA in conjunction
with Polyak-Ruppert tail averaging and Richardson-Romberg extrapolation.
We focus on corollaries of the convergence bounds in Theorem~\ref{thm:thm-converge}
and the bias expansion with order $m=1$ in Theorem~\ref{thm:bias-characterization},
namely $\E[\theta_{\infty}^{(\alpha)}]=\theta^{\ast}+\alpha B^{(1)}+\bigO(\alpha^{2})$
(analogous results can be obtained from higher order expansions).
In particular, we decompose the MSE into the optimization error, squared
bias and variance, and study how these quantities interplay with constant
stepsizes, averaging, and extrapolation. 

As our focus is dependence on the stepsize $\alpha$ and iteration
count $k$, we follow the notation convention in Section~\ref{sec:notations}.
In particular, $B,B'$ and $B''$ denote multiplicative factors independent
of $\alpha$ and $k$, whose values may change from line to line;
the big-$\bigO$ notation hides such factors.

\subsubsection{Results for Polyak-Ruppert Tail Averaging\label{sec:averaging}}

Polyak-Ruppert averaging~\citep{Ruppert88-Avg,Polyak92-Avg} is a
classical approach for reducing the variance and accelerating the
convergence of stochastic approximation. Here we consider the tail-averaging
variant of PR averaging~\citep{Jain18-tail-avg}. Given a user-specified
burn-in period $k_{0}\ge0$, define the tail-averaged iterates 
\[
\bar{\theta}_{k_{0},k}:=\frac{1}{k-k_{0}}\sum_{t=k_{0}}^{k-1}\theta_{t},\quad\text{for }k=k_{0}+1,k_{0}+2,\ldots
\]

The following corollary provides non-asymptotic characterization for
the first two moments of $\bar{\theta}_{k_{0},k}$. The proof can
be found in Appendix~\ref{sec:pr-avg-proof}.

\begin{cor} \label{cor:pr-avg-bounds} Under the setting of Theorem~\ref{thm:thm-converge},
the following bounds hold for all $k_{0}\ge\frac{4\gammax}{\alpha}\log\Big(\frac{1}{\alpha\tau_{\alpha}}\Big)$
and $k\geq k_{0}+\tau_{\alpha}$: 
\begin{align}
\E[\bar{\theta}_{k_{0},k}]-\theta^{\ast} & =\alpha B+\bigO\left(\alpha^{2}+\frac{1}{\alpha(k-k_{0})}\exp\left(-\frac{\alpha k_{0}}{4\gammax}\right)\right),\label{eq:pr-avg-first-mom}\\
\E\left[\left(\bar{\theta}_{k_{0},k}-\theta^{\ast}\right)\left(\bar{\theta}_{k_{0},k}-\theta^{\ast}\right)^{\top}\right] & =\alpha^{2}B'+\bigO\left(\alpha^{3}+\frac{\tau_{\alpha}}{k-k_{0}}+\frac{1}{\alpha(k-k_{0})^{2}}\exp\left(-\frac{\alpha k_{0}}{4\gammax}\right)\right).\label{eq:mse-pr-bound}
\end{align}
\end{cor}

To parse the above results, we fix $k_{0}=k/2$ and take the trace
of both sides of equation~\eqref{eq:mse-pr-bound}, which gives the
following bound on the MSE: 
\begin{equation}
\E\left[\|\bar{\theta}_{k/2,k}-\theta^{\ast}\|^{2}\right]=\underbrace{\alpha^{2}B''+\bigO(\alpha^{3})}_{\substack{T_{1}:\text{ asymptotic}\\
\text{squared bias}
}
}+\underbrace{\bigO\left(\frac{\tau_{\alpha}}{k}\right)}_{T_{2}:\text{ variance}}+\underbrace{\bigO\left(\frac{1}{\alpha k^{2}}\exp\left(-\frac{\alpha k}{8\gammax}\right)\right)}_{T_{3}:\text{ optimization error}}.\label{eq:mse-ta}
\end{equation}
Here the term $T_{1}$ corresponds to the asymptotic squared bias
$\|\E\bar{\theta}_{\infty/2,\infty}-\theta^{*}\|^{2}=\|\E\theta_{\infty}-\theta^{*}\|^{2}$,
which is not affected by averaging. The term $T_{2}$ roughly corresponds
to the variance $\var(\bar{\theta}_{k/2,k})$, which enjoys a $1/k$
decay rate due to averaging. The term $T_{3}$ is associated with
the optimization error $\|\E\bar{\theta}_{k/2,k}-\bar{\theta}_{\infty/2,\infty}\|^{2}$,
which decays geometrically in $k$ thanks to using a constant stepsize
$\alpha$ and only averaging the last $k/2$ iterates. Note that for
large values of $k$, the squared bias $T_{1}$ is the dominating
term in the MSE bound~\eqref{eq:mse-ta}.

We make several quick remarks. Firstly, since $\tau_{\alpha}\approx\log(1/\alpha)$,
the burn-in period $k_{0}$ required in Corollary \ref{cor:pr-avg-bounds}
is roughly $\bigO(\tau_{\alpha}/\alpha)$. This is the time by which
both the LSA iterates $(\theta_{k})_{k\geq0}$ and the underlying
chain $(x_{k})_{k\geq0}$ become well mixed, at which point the effect
of tail-averaging kicks in. 

Secondly, our MSE bound (\ref{eq:mse-ta}) for the averaged iterates,
which follows readily from Theorems \ref{thm:thm-converge} and \ref{thm:bias-characterization},
is comparable to the sharp results in \citep{Mou21-optimal-linearSA}
in terms of scaling with $\alpha,k$ and $\tau_{\alpha}$, though
the more complicated analysis in \citep{Mou21-optimal-linearSA} gives
tighter dependence on other parameters. 

Lastly, by setting $k_{0}=k-1$ in Corollary~\ref{cor:pr-avg-bounds}
(and relaxing the requirement $k\geq k_{0}+\tau$ via a more refined
argument), we obtain the following characterization for the raw LSA
iterates, $\bar{\theta}_{k,k+1}=\theta_{k}$:
\begin{equation}
\E\big[\|\theta_{k}-\theta^{\ast}\|^{2}\big]=\alpha^{2}B''+\bigO\left(\alpha\tau_{\alpha}\right)+\bigO\big(e^{-\alpha k/(4\gammax)}\big).\label{eq:mse-raw}
\end{equation}
This result is consistent with existing MSE upper bounds in~\citep{srikant-ying19-finite-LSA,Bhandari21-linear-td,chen20-contract-SA}.
The power of our results in (\ref{eq:pr-avg-first-mom})--(\ref{eq:mse-raw})
lies in that the first right hand side term therein features an equality
rather than merely an upper bound. As such, our results decouple the
contribution of the squared bias $\alpha^{2}B''$ from that of the
variance $\bigO(\alpha\tau_{\alpha})$. This decoupling is crucial
in understanding the effect of tail-averaging (in Corollary~\ref{cor:pr-avg-bounds})
and RR extrapolation (in Corollary~\ref{cor:rr-ext-bounds} to follow).

\subsubsection{Results for Richardson-Romberg Extrapolation\label{sec:extrapolation}}

We next show that RR extrapolation~\citep{bulirsch2002numerical_analysis}
can be used to reduce the bias to a higher order term of $\alpha$.
Let $\bar{\theta}_{k_{0},k}^{(\alpha)}$ and $\bar{\theta}_{k_{0},k}^{(2\alpha)}$
denote the tail-averaged iterates computed using two stepsizes $\alpha$
and $2\alpha$ with the same data stream $(x_{k})_{k\ge0}$. The RR
extrapolated iterates are defined as 
\[
\widetilde{\theta}_{k_{0},k}^{(\alpha)}=2\bar{\theta}_{k_{0},k}^{(\alpha)}-\bar{\theta}_{k_{0},k}^{(2\alpha)}.
\]
With $k_{0},k\to\infty$, Theorems~\ref{thm:thm-converge} and~\ref{thm:bias-characterization}
ensure that $\widetilde{\theta}_{k_{0},k}^{(\alpha)}$ converges to
$2\theta_{\infty}^{(\alpha)}-\theta_{\infty}^{(2\alpha)}$, which
has bias 
\begin{align*}
2\big(\E\theta_{\infty}^{(\alpha)}-\theta^{*}\big)-\big(\E\theta_{\infty}^{(2\alpha)}-\theta^{*}\big)=2\big(\alpha B^{(1)}+\bigO(\alpha^{2})\big)-\big(2\alpha B^{(1)}+\bigO(4\alpha^{2})\big)=\bigO(\alpha^{2}).
\end{align*}
Note that the extrapolation cancels out the first-order term of $\alpha$,
reducing the bias by a factor of~$\alpha$.

The following corollary formalizes the above argument and provides
non-asymptotic characterization for the first two moments of $\widetilde{\theta}_{k_{0},k}^{(\alpha)}$.
The proof can be found in Appendix~\ref{sec:rr-ext-proof}.

\begin{cor} \label{cor:rr-ext-bounds} Under the setting of Theorem~\ref{thm:thm-converge},
the RR extrapolated iterates with stepsizes $\alpha$ and $2\alpha$
satisfy the following bounds for all $k_{0}\ge\frac{4\gammax}{\alpha}\log\Big(\frac{1}{\alpha\tau_{\alpha}}\Big)$
and $k\geq k_{0}+\tau_{\alpha}$: 
\begin{align*}
\E\left[\widetilde{\theta}_{k_{0},k}^{(\alpha)}\right]-\theta^{\ast} & =\bigO(\alpha^{2})+\bigO\left(\frac{1}{\alpha(k-k_{0})}\exp\left(-\frac{\alpha k_{0}}{4\gammax}\right)\right),
\end{align*}
\begin{align}
\E\left[\Big(\widetilde{\theta}_{k_{0},k}^{(\alpha)}-\theta^{*}\Big)\Big(\widetilde{\theta}_{k_{0},k}^{(\alpha)}-\theta^{*}\Big)^{\top}\right] & =\underbrace{\bigO\left(\alpha^{4}\right)}_{\substack{\textnormal{asymptotic}\\
\textnormal{squared bias}
}
}+\underbrace{\bigO\left(\frac{\tau_{\alpha}}{k-k_{0}}\right)}_{\textnormal{variance}}+\underbrace{\bigO\left(\frac{1}{\alpha(k-k_{0})^{2}}\exp\left(-\frac{\alpha k_{0}}{4\gammax}\right)\right)}_{\textnormal{optimization error}}.\label{eq:rr-second-moment}
\end{align}
\end{cor}

Comparing the bound~\eqref{eq:rr-second-moment} with \eqref{eq:mse-pr-bound},
we see that RR extrapolation reduces the squared bias by a factor
of $\alpha^{2}$ while retaining the $1/k$ and $\exp(-k)$ convergence
rates for the variance and optimization error, respectively.

Thanks to higher order expansion in Theorem~\ref{thm:bias-characterization},
RR extrapolation can in fact be applied to more than two stepsizes,
which further reduces the bias. Let $\mathcal{A}=\{\alpha_{1},\alpha_{2},\ldots,\alpha_{m}\}$
be a set of $m\ge2$ distinct stepsizes and $\alpha=\max_{1\le i\le m}\alpha_{i}$.
Let $(h_{1},h_{2},\ldots,h_{m})\in\R^{m}$ be the solution to the
following linear equation system: 
\begin{align}
\sum_{i=1}^{m}h_{i}=1;\qquad\sum_{i=1}^{m}h_{i}\alpha_{i}^{t}=0,\;\;t=1,2,\ldots,m-1.\label{eq:RR_coefficients}
\end{align}
The solution is unique since the coefficient matrix of the system
is a Vandermonde matrix. The RR extrapolated iterates with stepsizes
in $\mathcal{A}$ and the burn-in period $k_{0}$ are given by 
\begin{align}
\widetilde{\theta}_{k_{0},k}^{\mathcal{A}}=\sum_{i=1}^{m}h_{i}\cdot\bar{\theta}_{k_{0},k}^{(\alpha_{i})}\,.\label{eq:RR_m_stepsizes}
\end{align}
This procedure eliminates the first $m-1$ terms in the bias expansion~\eqref{eq:bias-expansion},
reducing the bias to 
\begin{align*}
\E\left[\widetilde{\theta}_{k_{0},\infty}^{\mathcal{A}}\right]-\theta^{*} & =\sum_{i=1}^{m}h_{i}\cdot\left(\E\left[\theta_{\infty}^{(\alpha_{i})}\right]-\theta^{*}\right)=\bigO(\alpha^{m}).
\end{align*}
One can derive non-asymptotic bounds similar to Corollary~\ref{cor:rr-ext-bounds}---we
omit the details. In Section~\ref{sec:experiments}, we numerically
verify the efficacy of this high-order RR extrapolation approach.

\subsection{Implications for TD Learning}
\label{sec:td-section}

TD(0) is an iterative algorithm in RL for evaluating a given policy
for a Markov Decision Process (MDP), or equivalently for computing
the value function of a Markov Reward Process (MRP)~\citep{Bertsekas19-RL-book,Sutton18-RL-book}.
Potentially equipped with function approximation, TD(0) is a special
case of LSA. Consequently, all the results in the previous sections
can be specialized to TD(0), as we show below.

Consider an MRP $(\cS,P^{\cS},r,\gamma)$, with Borel state space
$\cS$, transition kernel $P^{\cS}$, bounded deterministic reward
function $r:\cS\to[-r_{\max},r_{\max}]$, and discount factor $\gamma\in[0,1)$.
We assume that the kernel $P^{\cS}$ is uniformly ergodic with unique
stationary distribution $\pi^{\cS}$. Note that we allow for a general
uncountable state space $\cS$, which generalizes many existing works
that focus on finite or countable state spaces \citep{Tsitsiklis97-td_paper,Bhandari21-linear-td}.
The value function $V:\cS\to\R$ is defined as $V(s)=\E\left[\sum_{t=0}^{\infty}\gamma^{t}r(s_{t})|s_{0}=s\right],$
where $(s_{k})_{k\ge0}$ is the Markov chain with kernel $P^{\cS}$.
It is common to assume that $V$ can be approximated by a linear function
as $V(s)\approx\phi(s)^{\top}\theta,$ where $\phi=(\phi_{1},\ldots,\phi_{d})^{\top}:\cS\to\R^{d}$
is a known feature map and $\theta$ is an unknown weight vector.
We assume that $\phi$ is in $L^{2}(\pi^{\cS})$ and has a finite
rank $d$. The latter means that the functions $\{\phi_{i}\}$ are
linearly independent, i.e., $\sum_{i=1}^{d}c_{i}\phi_{i}=0$ implies
$c_{i}=0,$ $\forall i\in\{1,\ldots,d\}$. When $\cS$ is finite or
countable, our assumption reduces to the matrix $\Phi=\begin{bmatrix}\phi(1) & \phi(2) & \cdots & \phi(|\cS|)\end{bmatrix}^{\top}\in\R^{|\cS|\times d}$
having full column rank, which is standard in literature~\citep{Tsitsiklis97-td_paper,Bhandari21-linear-td,srikant-ying19-finite-LSA}.
Lastly, we assume for simplicity that $s_{0}\sim\pi^{\cS}$ and the
feature map is normalized such that $\phi_{\max}:=\sup_{s\in\cS}\|\phi(s)\|\leq\frac{1}{\sqrt{1+\gamma}}$. 

Given a single Markovian data stream $(s_{k})_{k\geq0}$, the linear
TD(0) algorithm computes the update 
\begin{equation}
\theta_{k+1}=\theta_{k}+\alpha\big[r(s_{k})+\gamma\phi(s_{k+1})^{\top}\theta_{k}-\phi(s_{k})^{\top}\theta_{k}\big]\phi(s_{k}).\label{eq:linear-td-update-rule}
\end{equation}
TD(0) computes an approximation of the solution $\theta^{*}$ of the
projected Bellman equation $\Phi\theta=\Pi_{\phi}(r+\gamma P^{\cS}\Phi\theta),$
where $\Pi_{\phi}$ is the projection operator w.r.t.~$\left\Vert \cdot\right\Vert _{L^{2}(\pi^{\cS})}$
onto the subspace spanned by $\{\phi_{i}\}$. It is easy to see that
the TD(0) update~\eqref{eq:linear-td-update-rule} is a special case
of the LSA update~\eqref{eq:update-rule} with 
\begin{align*}
x_{k}=(s_{k},s_{k+1}),\quad A(x_{k})=\phi(s_{k})\big(\gamma\phi(s_{k+1})-\phi(s_{k})\big)^{\top},\quad b(x_{k})=r(s_{k})\phi(s_{k}),
\end{align*}
and $\cX=\cS\times\cS.$ Below we verify that TD(0) satisfies the
required assumptions. 
\begin{itemize}
\item Assumption~\ref{assumption:uniform-ergodic}: The state space $\cX$
is Borel since the product of Borel spaces remains Borel \citep[Chapter 21.4, Proposition 20]{fristedt1997}.
The uniform ergodicity of the chain $(s_{k})_{k\ge0}$ implies that
of the augmented chain $(x_{k})_{k\ge0}$. The assumption $s_{0}\sim\pi^{\cS}$
implies $x_{0}\sim\pi$, where $\pi$ denotes the stationary distribution
of $(x_{k})_{k\ge0}$. 
\item Assumption~\ref{assumption:bounded} follows from the normalization
$\phi_{\max}\leq\frac{1}{\sqrt{1+\gamma}}$ and direction calculation:
\begin{align*}
\Amax & =\sup_{s,s'\in\cS}\|\phi(s)(\gamma\phi(s')-\phi(s))^{\top}\|\leq(1+\gamma)\phi_{\max}^{2}\quad\text{and}\quad\bmax=\sup_{s\in\cS}\|r(s)\phi(s)\|\leq r_{\max}\phi_{\max}.
\end{align*}
\item Assumption~\ref{assumption:hurwitz}: By direct calculation we have
$\BarA_{ij}=\dotp{\phi_{i},(I-\gamma P^{\cS})\phi_{j}}_{L^{2}(\pi^{\mS})},i,j=1,\ldots,d,$
so $\BarA$ is negative definite and hence Hurwitz. The proof is given
in Appendix~\ref{sec:neg-def-proof}.
\end{itemize}
Consequently, all the results in Sections~\ref{sec:converge-to-limit}--\ref{sec:implications}
apply to TD(0) with linear function approximation, constant stepsizes
and Markovian data. 

Our work generalizes many existing non-asymptotic results on TD(0)
in the i.i.d.\ data setting, e.g., \citep{Dalal18-lineartd0,Bhandari21-linear-td,Khamaru21-td-instance,durmus2021-LSA}.
Under this setting, TD(0) corresponds to the update 
\begin{align}
\theta_{k+1}=\theta_{k}+\alpha\big[r(s_{k})+\gamma\phi(s_{k}^{\text{next}})^{\top}\theta_{k}-\phi(s_{k})^{\top}\theta_{k}\big]\phi(s_{k}),\label{eq:TD_iid}
\end{align}
where the data $x_{k}=(s_{k},s_{k}^{\text{next}})$ is independent
across $k$ and follows the distribution $s_{k}\sim\pi^{\mS}$, $s_{k}^{\text{next}}\sim P^{\mS}(s_{k},\cdot)$.
In this setting, Theorem~\ref{thm:bias-characterization-reversible}
implies that TD(0) with a constant stepsize has no asymptotic bias,
i.e., $\E[\theta_{\infty}]=\theta^{*}.$

\subsubsection{The Semi-Simulator Setting}

Using the explicit bias characterization in Section \ref{sec:bias-expansion},
we show that TD(0) may admit zero bias beyond the i.i.d.~data setting. 

Specifically, consider the update rule (\ref{eq:TD_iid}). This time
we assume that $(s_{k})_{k\ge0}$ is a Markov chain, and conditioned
on $s_{k}$, $s_{k}^{\text{next}}$ is sampled independently from
$P^{\mS}(s_{k},\cdot)$. We call this the \emph{semi-simulator setting}.
For simplicity, we present our results for the \emph{tabular} setting of
TD(0) on a finite state space, i.e., $d=|\mS|<\infty$ and $\phi(s)=e_{s}\in\real^{|\mS|}$
is the $s$-th standard basis vector. In this case, the target vector
$\theta^{*}=V$ is the true value function and satisfies the vanilla
Bellman equation $\theta^{*}=r+\gamma P^{\mS}\theta^{*}$, which can
be explicitly written as
\begin{equation}
\theta^{*}(s)=r(s)+\gamma\E_{s^{\text{next}}\sim P^{\mS}(s,\cdot)}[\theta^{*}(s^{\text{next}})],\quad\forall s\in\mS.\label{eq:Bellman}
\end{equation}

\begin{cor} \label{cor:zero-bias-TD} Under the semi-simulator setting,
the TD(0) update (\ref{eq:TD_iid}) satisfies $\E[\theta_{\infty}]-\theta^{*}=0$.\end{cor}

\begin{proof}We verify the condition (\ref{eq:zero-bias-sufficient})
in Corollary \ref{cor:zero-bias-sufficient}. Recall that $x_{k}=(s_{k},s_{k}^{\text{next}})$.
By definitions of $A,b$ and $\phi$ and the law of total expectation,
we have 
\begin{align*}
\E\big[A(x_{k})\theta^{\ast}+b(x_{k})\mid x_{k+1}\big]= & \E\big[\phi(s_{k})\big(\gamma\phi(s_{k}^{\text{next}})-\phi(s_{k})\big)^{\top}\theta^{\ast}+r(s_{k})\phi(s_{k})\mid x_{k+1}\big]\\
= & \E\big[\phi(s_{k})\big(\gamma\theta^{*}(s_{k}^{\text{next}})-\theta^{*}(s_{k})+r(s_{k})\big)\mid x_{k+1}\big]\\
= & \E\big[\phi(s_{k})\E[\gamma\theta^{*}(s_{k}^{\text{next}})-\theta^{*}(s_{k})+r(s_{k})\mid x_{k+1},s_{k}]\mid x_{k+1}\big]
\end{align*}
With $s_{k}^{\text{next}}\indep x_{k+1}\mid s_{k}$ and the Bellman
equation (\ref{eq:Bellman}), the inner expectation above satisfies
\[
\E[\gamma\theta^{*}(s_{k}^{\text{next}})-\theta^{*}(s_{k})+r(s_{k})\mid x_{k+1},s_{k}]=\gamma\E[\theta^{*}(s_{k}^{\text{next}})\mid s_{k}]-\theta^{*}(s_{k})+r(s_{k})=0.
\]
Combining pieces verifies the desired condition $\E[A(x_{k})\theta^{\ast}+b(x_{k})\mid x_{k+1}]=0.$\end{proof}

\subsection{Implications for Markovian SGD}
\label{sec:sgd-implication} 

SGD is widely used in optimization and machine learning problems.
When minimizing a quadratic objective function, the update step of
SGD is linear in the iterate and can be cast as LSA. Markovian data arise in many sequential decision-making settings, such as those
with experience replay and auto-regressive dynamics \citep{guy2020}. 

Specifically, we consider SGD applied to the minimization of a quadratic
function 
\[
\ell(\theta):=-\frac{1}{2}\,\E_{x\sim\pi}\Big[\theta^{\top}A(x)\theta+b(x)^{\top}\theta+c(x)\Big],
\]
where $A:\cX\to\S^{d\times d}$, $b:\cX\to\R^{d}$ and $c:\cX\to\real$
are deterministic functions, and $\S^{d\times d}$ denotes the set
of $d\times d$ symmetric matrices. Given a Markovian data stream
$(x_{k})_{k\geq0}$ with stationary distribution $\pi$, the constant
stepsize SGD algorithm $\theta_{k+1}=\theta_{k}+\alpha\left(A(x_{k})\theta_{k}+b(x_{k})\right)$
is a special case of the LSA iteration (\ref{eq:update-rule}). We
impose the same Assumptions~\ref{assumption:uniform-ergodic}--\ref{assumption:hurwitz}
on $A,b$ and $(x_{k})_{k\ge0}$. Since the matrix $\BarA$ is symmetric,
Assumption \ref{assumption:hurwitz} implies that $-\BarA$ is positive
definite. In this case, the objective function $\ell$ is strongly
convex, with a unique minimizer $\theta^{*}$ satisfying the stationarity
condition $\nabla\ell(\theta)=0,$ which can be seen to coincide with
the steady-state equation (\ref{eq:steady-state-equation}). Note
that we do not need the matrix $-A(x)$ to be positive definite for
each individual $x$.

All the results in Sections~\ref{sec:converge-to-limit}--\ref{sec:implications}
apply to the above setting. In particular, we establish the distributional convergence of SGD with constant stepsizes and characterize the bias, variance, and optimization error under averaging and extrapolation.
These results generalize several quadratic minimization results in
\citep{bach2013,Dieuleveut20-bach-SGD}, which focus on the i.i.d.~setting.
In the Markovian data setting, our results imply that  in general there
exists an asymptotic bias proportional to the stepsize $\alpha$,
i.e., $\E[\theta_{\infty}]-\theta^{*}=\alpha B^{(1)}+\bigO(\alpha^{2}).$
This result is a generalization and a more precise version of \citep[Theorem 4]{guy2020},
which considers SGD for least squares regression---a special case
of quadratic minimization---and establishes the lower bound $\|\E[\theta_{\infty}]-\theta^{\ast}\|\geq c\alpha$
under Markovian data. Our result explains the root of this error lower
bound and suggests RR extrapolation as a way to reduce the error.

\subsubsection{Least Squares Regression with Independent Additive Noise}

\label{sec:reg-implication}

We specialize our results to the least squares regression problem.
The explicit bias characterization in Section~\ref{sec:bias-expansion}
allows us to identify an interesting setting with Markovian data but
nevertheless zero bias.

Suppose that we observe data pairs $(\obs_{k},y_{k})_{k\geq0}$, where
$\obs_{k}\in\real^{d}$ is the covariate vector sequentially sampled
from a uniformly ergodic Markov chain, $y_{k}=\dotp{\obs_{k},\theta^{\text{reg}}}+n_{k}$
is the scalar response variable, $\theta^{\text{reg}}\in\real^{d}$
is the true regression vector, and $n_{k}$ is the additive noise.
We assume that the noise $n_{t}$ is zero mean and independent from
$\{g_{t}\}$ and $\{n_{t}:t\neq k\}$. The goal is to estimate the
minimizer $\theta^{*}=\theta^{\text{reg}}$ of the least squares objective
$\ell(\theta)=\E_{(g,y)\sim\pi_{0}}\big[(\dotp{\obs,\theta}-y)^{2}\big],$
where $\pi_{0}$ is the stationary distribution of the process $(g_{k},y_{k})_{k\ge0}$.
The corresponding SGD step is 
\begin{equation}
\theta_{t+1}=\theta_{t}-\alpha\obs_{t}(\obs_{t}^{\top}\theta_{t}-y_{t}).\label{eq:SGD_regression}
\end{equation}
This step can be cast into the Markovian LSA iteration (\ref{eq:update-rule})
with 
\[
x_{t}=(g_{t},n_{t}),\quad A(x_{t})=-\obs_{t}\obs_{t}^{\top},\quad\text{and}\quad b(x_{t})=\obs_{t}y_{t}=\obs_{t}\obs_{t}^{\top}\theta^{\ast}+\obs_{t}n_{t}.
\]
Note that the $x_{t}$'s are correlated since the covariates $g_{t}$'s
are. Nevertheless, SGD turns out to have zero asymptotic bias, $\E\theta_{\infty}-\theta^{*}=0$,
due to the structure of the noise component $n_{t}$. We prove this
claim by verifying the condition~(\ref{eq:zero-bias-sufficient})
in Corollary~\ref{cor:zero-bias-sufficient}: 
\begin{align}
\E[A(x_{t})\theta^{\ast}+b(x_{t})\mid x_{t+1}]= & -\E[\obs_{t}\obs_{t}^{\top}\mid\obs_{t+1},n_{t+1}]\theta^{\ast}+\E[\obs_{t}\obs_{t}^{\top}\theta^{\ast}+\obs_{t}n_{t}\mid\obs_{t+1},n_{t+1}]\nonumber \\
= & \E[\obs_{t}\mid\obs_{t+1}]\cdot\E[n_{t}]=0,\label{eq:regression_zero_bias}
\end{align}
where the last two steps follow from the independence and zero mean
assumptions on $n_{t}$. Note that a similar observation is made in
\citep[Theorem 3]{guy2020} using a more specialized argument. 

In fact, the bias is zero under the more general setting of conditionally
independent noise:

\begin{cor} \label{cor:zero-bias-regression} In the regression setting
above, if the additive noise satisfies
\[
n_{t}\indep g_{t}\mid(g_{t+1},n_{t+1})\quad\text{and}\quad\E[n_{t}\mid g_{t+1},n_{t+1}]=0,
\]
then the asymptotic bias of the SGD update (\ref{eq:SGD_regression})
satisfies $\E[\theta_{\infty}]-\theta^{*}=0$. \end{cor}

\begin{proof}The last two steps in equation (\ref{eq:regression_zero_bias})
hold under the assumption of the corollary.\end{proof}

Finally, we reiterate that when the condition in Corollary~\ref{cor:zero-bias-regression} is not satisfied and the bias is nonzero, one may consider employing Richardson-Romberg extrapolation to reduce the bias.

\section{Numerical Experiments}

\label{sec:experiments}

In this section, we provide numerical experiment results for LSA,
TD(0) with linear function approximation, and SGD applied to least squares regression.

\subsection{Experiments for LSA}

\label{sec:expt_LSA}

We consider the LSA update~\eqref{eq:update-rule} in dimension $d=4$
for a finite state, irreducible, and aperiodic Markov chain with $n=8$
states. We construct the transition probability matrix $P$ and the
functions $A$ and $b$ randomly; see Appendix~\ref{sec:expt_detals_LSA}
for the details. Given $P$, we generate a single trajectory of the
Markov chain $(x_{k})_{k=1}^{K}$ of length $K=10^{8}$, and run the
LSA iteration with initialization $\theta_{0}^{(\alpha)}=0$ and stepsizes
$\alpha\in\{0.2,0.4,0.8\}$.

In Figure~\ref{fig:lsa-ta-rr}, we plot the error $\|\theta_{k}^{(\alpha)}-\theta^{*}\|$
for the raw LSA iterates $\theta_{k}^{(\alpha)}$, the error for the
tail-averaged (TA) iterates $\bar{\theta}_{k/2,k}^{(\alpha)}$, and
the error for the RR extrapolated iterates $\widetilde{\theta}_{k}^{(\alpha)}$
with stepsizes $\alpha$ and $2\alpha$. For comparison, we also include
the errors for LSA with a diminishing stepsize $\alpha_{k}=0.2/k^{0.75}$.
We see that the raw LSA iterates with constant stepsizes oscillate,
whereas the tail averaged iterates converge to a limit, with a smaller
error for a smaller stepsize. Moreover, the final TA error, which
corresponds to the asymptotic bias, is roughly proportional to the
stepsize (note the equal spacing in the log scale between the three TA
lines). Finally, RR extrapolation with two stepsizes further reduces
the bias, as can be seen by comparing, e.g., the dashed red line (TA
with $\alpha=0.4$) and the solid red line (RR with $\alpha=0.4$
and $0.8$). These observations are consistent with our theory. Finally,
the tail-averaged iterates with constant stepsizes have significantly
faster initial convergence than the iterates with a diminishing stepsize.

\begin{figure}[htbp]
\centering \subfigure[
The errors of the raw LSA iterates, tail-averaged (TA) iterates, and RR extrapolated iterates with different stepsizes $\alpha$. 
\label{fig:lsa-ta-rr}]{ \includegraphics[width=0.45\textwidth]{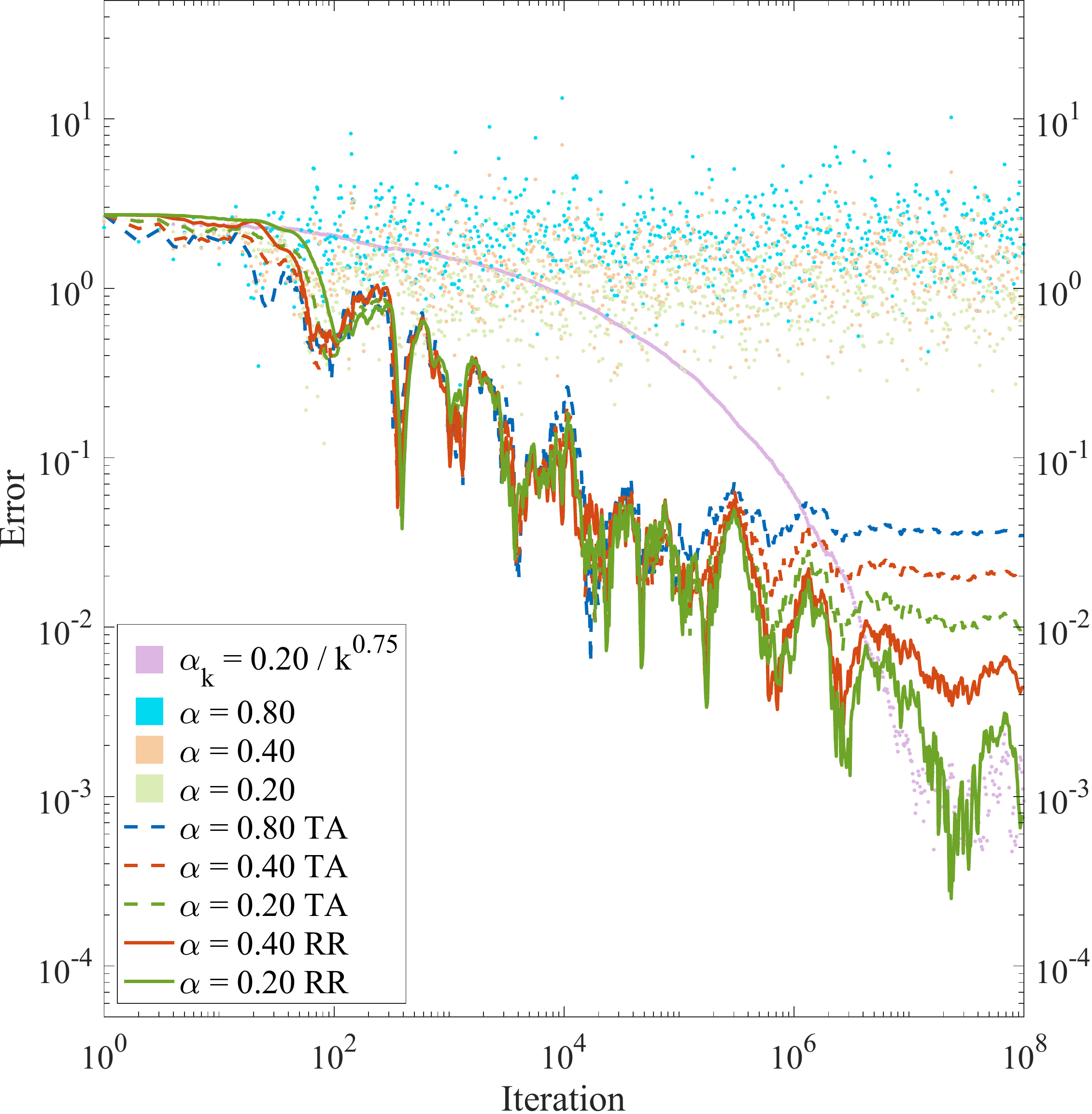}
} \quad\subfigure[
The errors of the raw LSA iterates and tail-averaged (TA) iterates under different SLEM $|\lambda_2|$. The stepsize $\alpha$ is fixed at 0.8. 
\label{fig:lsa-iid-mkv}]{ \includegraphics[width=0.45\textwidth]{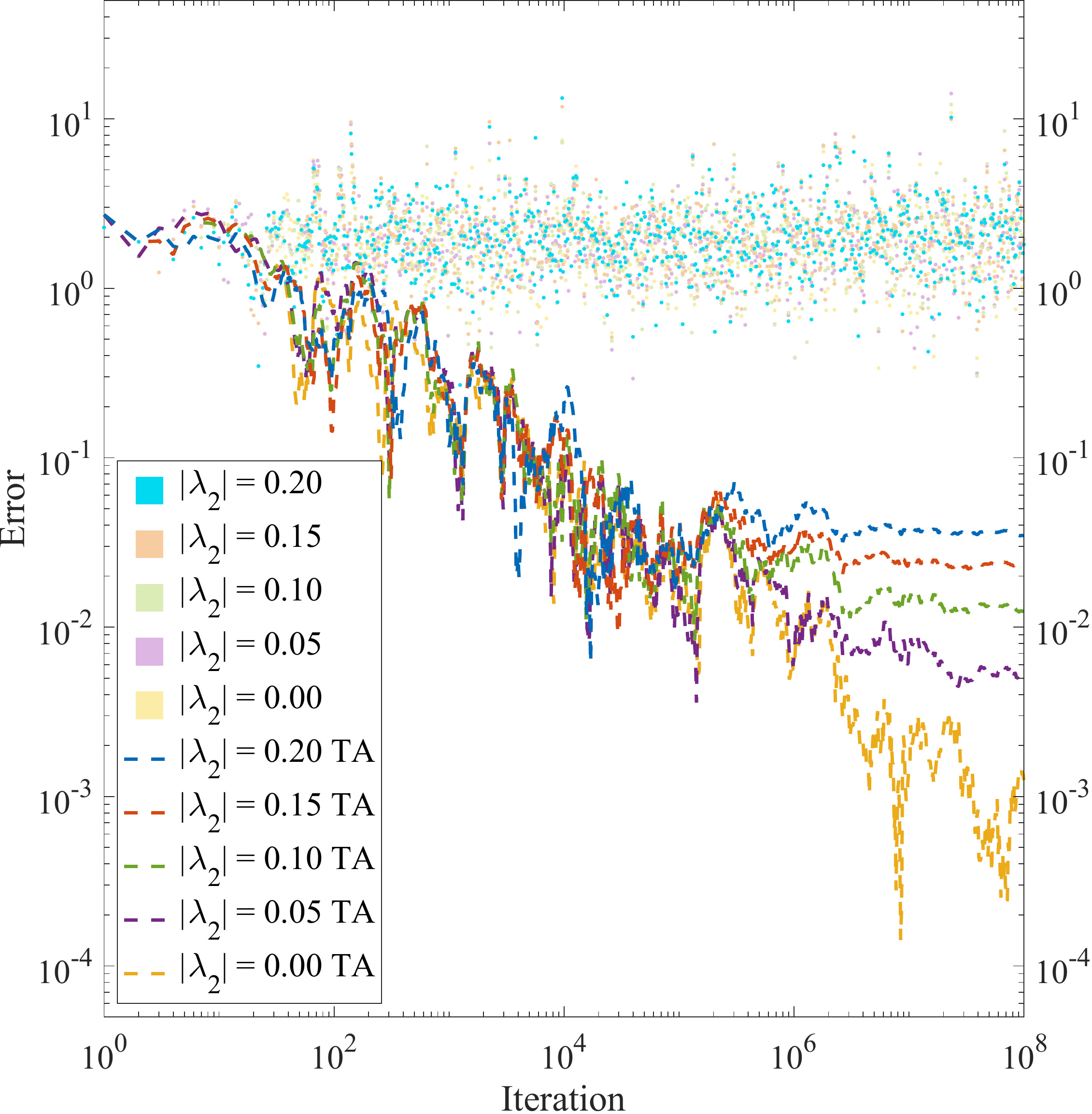} } \caption{Experiment results for LSA}
\label{fig:lsa} 
\end{figure}

We next investigate how the error depends on the spectral gap and
mixing time of $P$. As the state space is finite, we work with the
SLEM $|\lambda_{2}|$ of $P$, which satisfies $\absgap(P)=1-|\lambda_{2}|$.
Given $P$ generated above and its stationary distribution $\pi$,
we construct another transition probability matrix parameterized by
$\beta\in[0,1]$ as follows: 
\[
P^{(\beta)}=\beta\cdot P+(1-\beta)\cdot\onevec\pi^{\top}.
\]
Note that $P^{(1)}=P$, and $P^{(\beta)}$ has $\pi$ as the stationary
distribution for any $\beta$. As $\beta$ decreases from $1$ to
$0$, the SLEM $|\lambda_{2}|$ of $P^{(\beta)}$ decreases towards
$0$. For different values of $\beta$, we run the LSA with $P^{(\beta)}$
as the transition probability matrix of the chain $(x_{k})_{k\ge0}$.
In Figure~\ref{fig:lsa-iid-mkv}, we plot the corresponding errors
of the tail-averaged iterates. We see that a smaller $|\lambda_{2}|$
leads to a smaller final error. Moreover, when $\lambda_{2}=0$, which
corresponds to the i.i.d.\ data setting, the error is converging
to zero, which indicates a vanishing asymptotic bias. These observations
are consistent with Theorem~\ref{thm:bias-characterization-reversible}
on the relationship between the asymptotic bias and mixing time.

\subsection{Experiments for TD(0) with Linear Function Approximation}
\label{sec:expt_TD}

We perform a similar set of experiments as in the previous sub-section
on the TD(0) algorithm. In particular, we consider the classical ``Problematic
MDP'' from \citep{koller2000-pi,Lagoudakis03-LSPI}, and use TD(0)
with linear function approximation to estimate the value function
of a given policy. See Appendix~\ref{sec:expt_details_TD} for the
details of the MDP, the policy, and the choice of the feature vectors.

\begin{figure}[htbp]
\centering \includegraphics[width=0.7\textwidth]{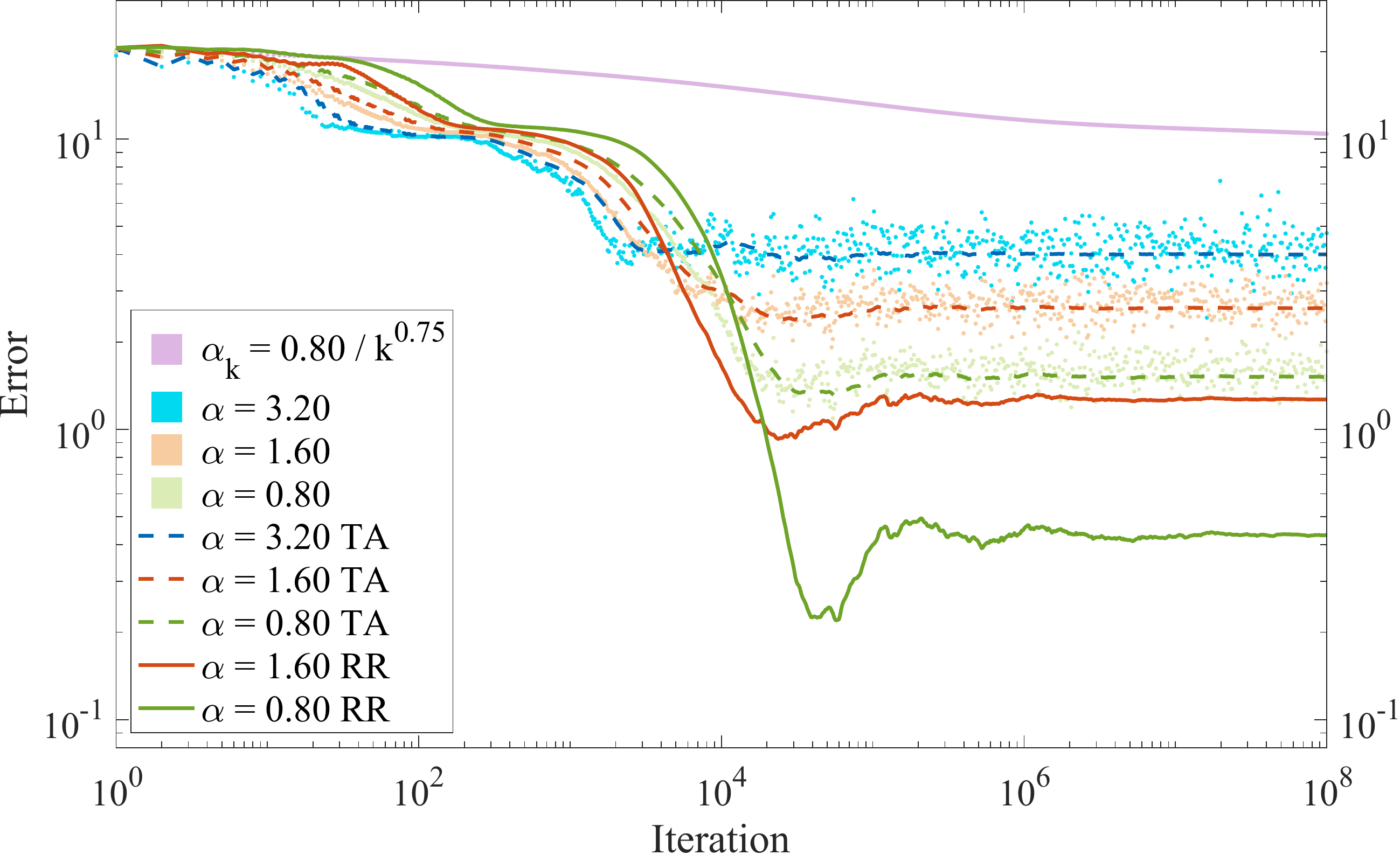}
\caption{The errors of the raw TD(0), tail-averaged (TA), and RR extrapolated
iterates with different stepsizes $\alpha$.}
\label{fig:td-ta-rr} 
\end{figure}

In Figure~\ref{fig:td-ta-rr}, we plot the errors of the raw TD(0)
iterates, tail-averaged iterates, and RR extrapolated iterates with
different stepsizes $\alpha$. The results are qualitatively similar
to those in Figure~\ref{fig:lsa-ta-rr}. In addition, the TA iterates
with a larger stepsize have faster initial convergence, which is consistent
with theoretical prediction in Corollary~\ref{cor:pr-avg-bounds}.

We further investigate the benefit of higher-order RR extrapolation
with more than 2 stepsizes, using the procedure described in equations~\eqref{eq:RR_coefficients}
and~\eqref{eq:RR_m_stepsizes}. Specifically, we compare the errors
of the tail-averaged iterates and the RR extrapolated iterates with
$2$ to $6$ stepsizes. The results are shown in Figure~\ref{fig:TD_6RRs}.
Here we use a set of large stepsizes (of similar magnitudes), which
give fast initial convergence. We see that using more stepsizes in
RR extrapolation reduces the final errors by a significant margin.
In particular, the error of RR extrapolation with 6 stepsizes is smaller
by 3 orders of magnitude than TA with the same stepsizes. We emphasize
that this error reduction is obtained almost for free, as we can run
the six TD(0) iterations in parallel using the same data.

\begin{figure}[htbp]
\centering \includegraphics[width=0.7\textwidth]{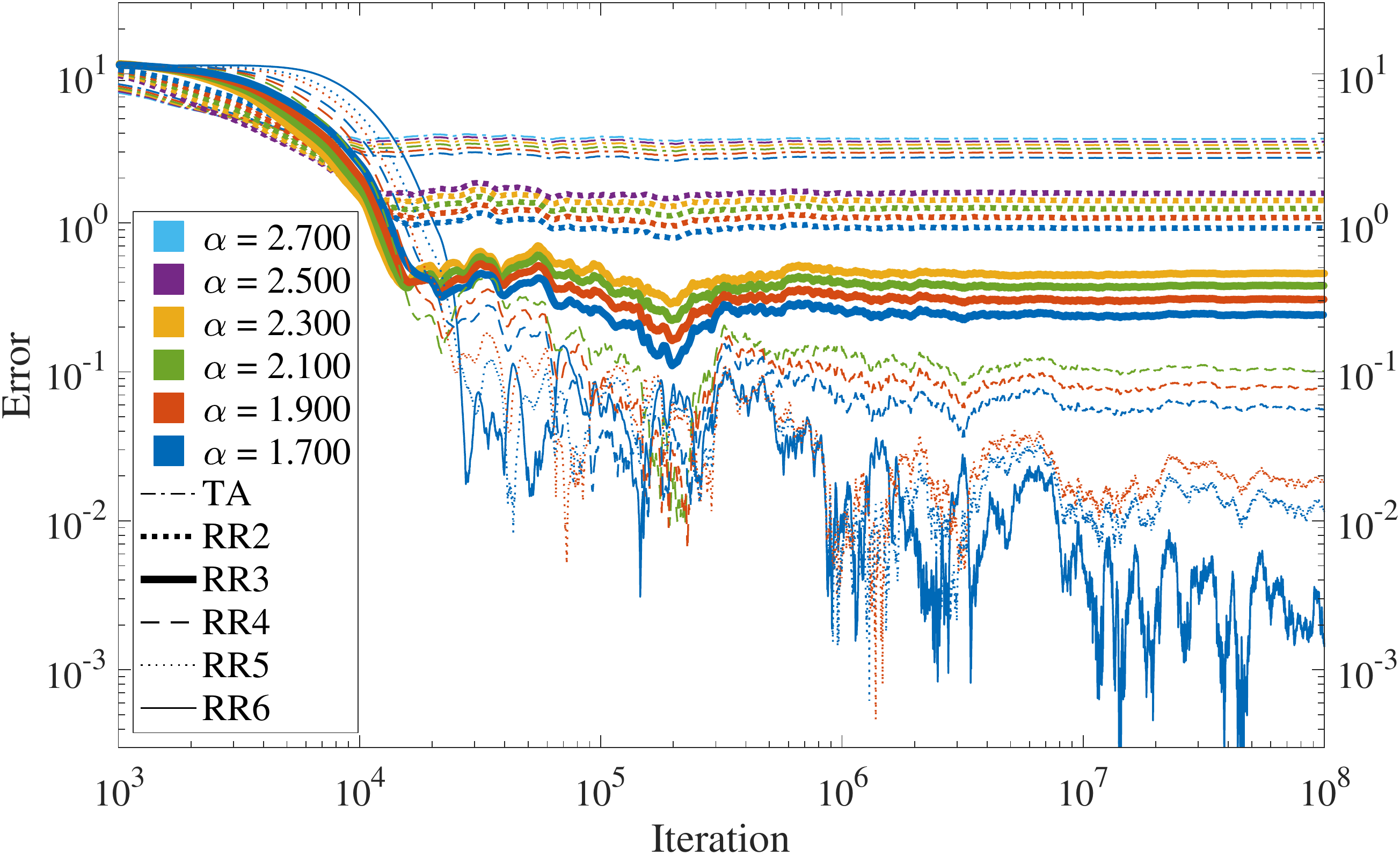} \caption{Comparison between tail-averaging (TA) and RR extrapolation with $m$
stepsizes, for $m=2,\ldots,6$. The setting for each line in the plot
is given by its line style (representing the number of stepsizes used
in RR) and line color (representing the smallest stepsize involved).
For example, the dash-dotted green line corresponds to TA with stepsize
$\alpha=2.1$, and the dashed red line corresponds to RR with four
stepsizes $\alpha\in\{1.9,\,2.1,\,2.3,\,2.5\}$. }
\label{fig:TD_6RRs} 
\end{figure}

\subsection{Experiments for Least Squares Linear Regression}
\label{sec:expt_SGD}

Finally, we consider SGD applied to the regression model $y_{k}=\dotp{\obs_{k},\theta^{\text{reg}}}+n_{k}$
as described in Section~\ref{sec:reg-implication}, where $\theta^{\text{reg}}=0$,
and $g_{k}\in\R^{2}$ is obtained from a Metropolis-Hastings (MH)
sampler on $[-1,1]^{2}$; see Appendix~\ref{sec:expt_detals_SGD}
for the details. We generate a single trajectory of the Markov chain
$(\obs_{k},n_{k})_{k=1}^{K}$ with $K=10^{8}$ and apply SGD with
constant stepsizes $\alpha\in\{0.01,0.02,0.04\}$. The target vector
is the minimizer $\theta^{\ast}$ of the least squares objective $\ell(\theta)=\E_{(g,n)\sim\pi}[\|\dotp{g,\theta}-\left\langle g,\theta^{\text{reg}}\right\rangle -n\|^{2}]$. 

We consider two different settings for the noise $n_{k}$ and study
its impact on the error $\|\theta_{k}-\theta^{*}\|$. In the first
setting, we assume $n_{k}\overset{\text{i.i.d.}}{\sim}\text{Unif}[-1,1]$.
In this case, Corollary \ref{cor:zero-bias-regression} predicts that
the asymptotic bias is zero despite the Markovian correlation in $\{g_{k}\}$.
In Figure~\ref{fig:sgd-ta-no-bias}, we plot the errors for the
TA iterates $\bar{\theta}_{k/2,k}^{(\alpha)}$ with a constant stepsize
and for the raw LSA iterates with a diminishing stepsize $\alpha_{k}=0.01/k^{0.75}$.
We see that the constant stepsize TA iterates converge to zero, as
predicted by our theory, and the convergence speed is faster than
using a diminishing stepsize.

In the second setting, the noise $n_{k}$ is correlated with $\obs_{k}$
as follows: we set $n_{k}=\text{sign}\big(g_{k}(1)+g_{k}(2)\big)$,
where $g_{k}(i)$ denotes the $i$-th coordinate of $g_{k}$. In
Figure~\ref{fig:sgd-ta-with-bias}, we plot the errors for TA iterates
$\bar{\theta}_{k/2,k}^{(\alpha)}$ and RR extrapolated iterates $\widetilde{\theta}_{k}^{(\alpha)}$.
The results here indicate a nonzero bias proportional to the stepsize,
in contrast to the first noise setting. We also see that RR extrapolation
reduces the bias.

\begin{figure}[htbp]
\centering %
\begin{minipage}[c]{0.47\textwidth}%
 \centering \includegraphics[width=1\linewidth]{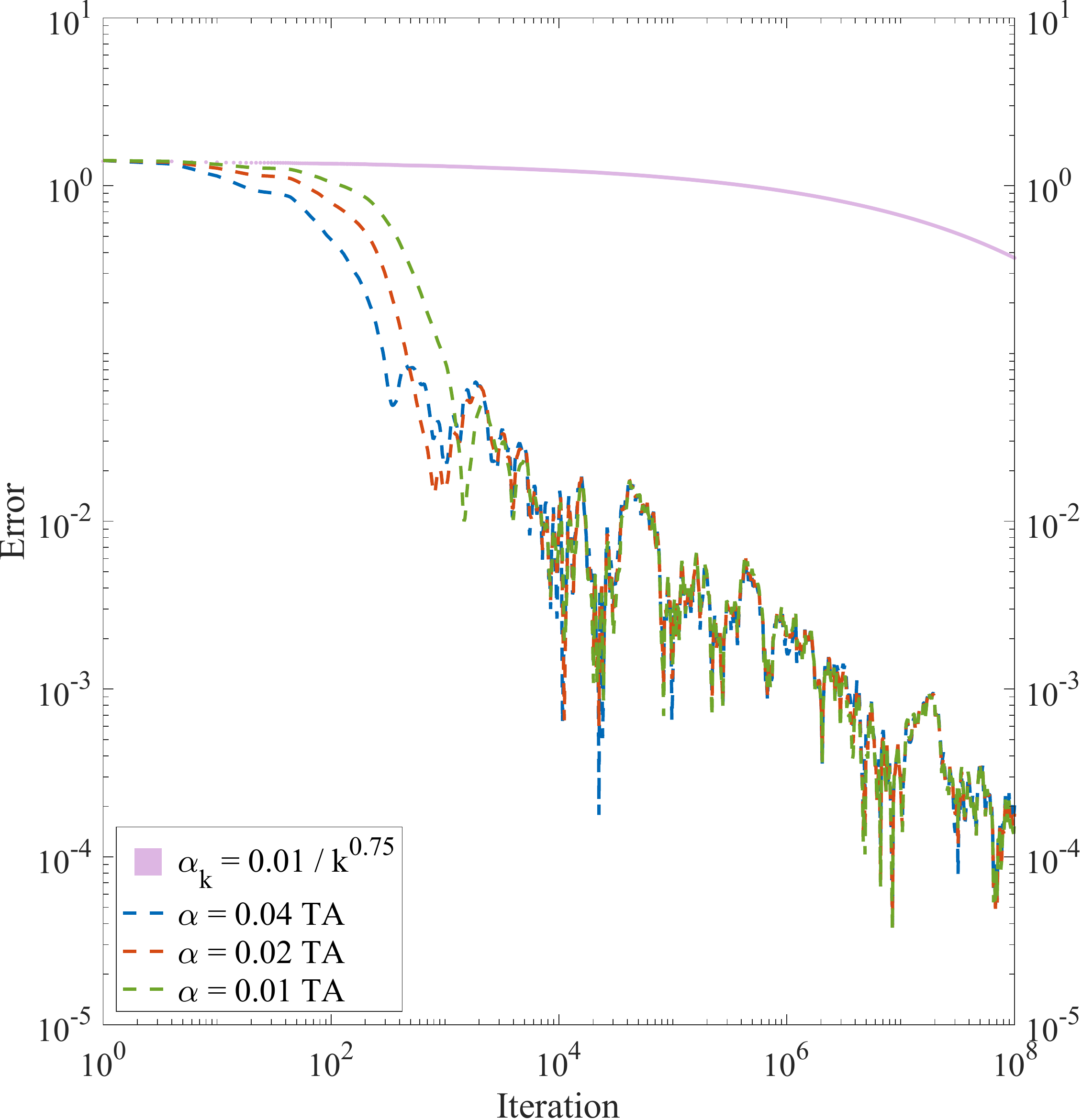}
\captionof{figure}{The errors of the tail-averaged (TA) SGD iterates
with different stepsizes $\alpha$.\\
 } \label{fig:sgd-ta-no-bias} %
\end{minipage}\hspace{0.5cm}%
\begin{minipage}[c]{0.47\textwidth}%
\centering \includegraphics[width=1\linewidth]{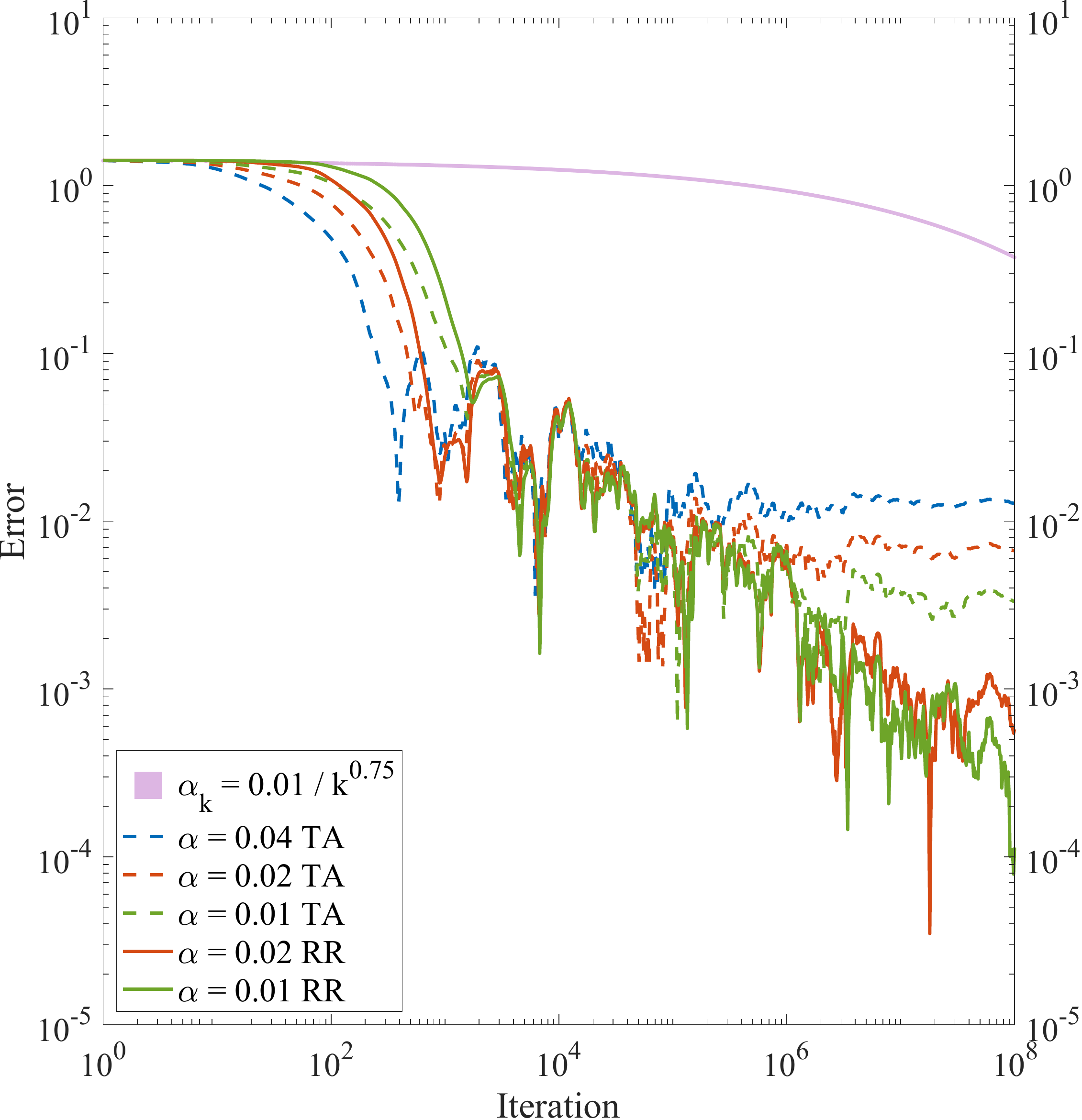}
\captionof{figure}{The errors of the tail-averaged (TA) and RR
extrapolated SGD iterates with different stepsizes $\alpha$.} \label{fig:sgd-ta-with-bias} %
\end{minipage}
\end{figure}

\section{Proof Outline}
\label{sec:proof_sketch}

In this section, we outline the proofs for Theorem~\ref{thm:thm-converge}
(convergence of LSA) and Theorem~\ref{thm:bias-characterization}
(bias expansion). The proofs make use of a pilot result Proposition~\ref{prop:pilot},
stated in Section~\ref{sec:pilot}, which serves as the basis for
subsequent analysis. The complete proofs of these results and other
main theorems/corollaries are given in the appendix.

\subsection{A Pilot Result}
\label{sec:pilot}

We have the following non-asymptotic upper bound on the MSE $\E[\|\theta_{k}-\theta^{\ast}\|^{2}]$.

\begin{proposition} \label{prop:pilot} Under Assumptions~\ref{assumption:uniform-ergodic},~\ref{assumption:bounded}
and~\ref{assumption:hurwitz}, if $\alpha$ satisfies equation~\eqref{eq:alpha-constraint},
then the following bound holds for all $k\geq\tau$, 
\[
\E[\|\theta_{k}-\theta^{\ast}\|^{2}]\leq10\,\frac{\gammax}{\gammin}\left(1-\frac{0.9\alpha}{\gammax}\right)^{k}\Big(\E[\|\theta_{0}-\theta^{\ast}\|^{2}]+s_{\min}^{-2}(\BarA)\bmax^{2}\Big)+\alpha\tau\cdot\kappa,
\]
with $\kappa$ defined in equation~\eqref{eq:kappa-def}. \end{proposition}

Proposition~\ref{prop:pilot} is a moderate improvement of \citep[Theorem 7]{srikant-ying19-finite-LSA}.
When $\bmax=0$ (which means $b(x)=0,\forall x$), Proposition~\ref{prop:pilot}
guarantees that $\kappa=0$, in which case $\theta_{k}$ converges
in mean squared to $\theta^{\ast}$ as $k\to\infty$. This fact plays
an important role in proving the distributional convergence result
in Theorem~\ref{thm:thm-converge} in the setting with a general
$b$ and nonzero $\bmax$. In particular, the proof of Theorem~\ref{thm:thm-converge}
employs a coupling argument that constructs another process with $\bmax=0$.
In comparison, the bound in \citep[Theorem 7]{srikant-ying19-finite-LSA}
gives a non-zero value of $\kappa$ even when $\bmax=0$ and hence
is insufficient for running the coupling argument. Moreover, the stepsize
condition~\eqref{eq:alpha-constraint} required by Proposition~\ref{prop:pilot}
(and by all our other results) does not involve $\bmax$, which correctly
reflects the translation invariance of the LSA update~\eqref{eq:update-rule}.
The stepsize condition in \citep[Theorem 7]{srikant-ying19-finite-LSA},
on the other hand, has a superfluous dependence on $\bmax.$

The proof of Proposition~\ref{prop:pilot} is similar to that of
\citep[Theorem 7]{srikant-ying19-finite-LSA} with more refined arguments.
For completeness, we provide the proof in Appendix~\ref{sec:proof_prop:pilot}.
One key refinement in our proof is to avoid invoking inequalities
of the form $2u\le u^{2}+1$, and to use instead $2u\le\beta^{2}u^{2}+1/\beta^{2}$
with a judicious choice of $\beta$ that respects the translation
invariance of the LSA update~\eqref{eq:update-rule}.

\subsection{Proof Outline of Theorem~\ref{thm:thm-converge}}

\label{sec:thm-converge-proof-sketch}

We sketch the main ideas in proving Theorem~\ref{thm:thm-converge}.
The complete proof is given in Appendix~\ref{sec:proof_thm_conv_limit_dist}.

The proof consists of bounding Wasserstein distances of the form $\Bar W_{2}(\law(x_{k},\theta_{k}),\law(x_{k+1},\theta_{k+1}))$
and $\Bar W_{2}(\law(x_{k},\theta_{k}),\law(x_{\infty},\theta_{\infty}))$.
Since the Wasserstein distance is defined by the optimal coupling,
it can be upper bounded by constructing a particular coupling. With
this strategy in mind, we consider coupling two Markov chains $(x_{k}^{\sampleOne},\theta_{k}^{\sampleOne})_{k\ge0}$
and $(x_{k}^{\sampleTwo},\theta_{k}^{\sampleTwo})_{k\ge0}$, which
are two copies of the LSA iteration~\eqref{eq:update-rule}. We make
use of two types of coupling in the proof.

The first type of coupling is constructed by letting the two Markov
chains above share the same underlying data stream $(x_{k})_{k\geq0}$,
i.e., letting $x_{k}^{\sampleOne}=x_{k}^{\sampleTwo}=x_{k}$ for all
$k\ge0$. Explicitly, the iterates $\theta_{k+1}^{\sampleOne}$ and
$\theta_{k+2}^{\sampleTwo}$ are given by the updates
\[
\begin{aligned}\theta_{k+1}^{\sampleOne} & =\theta_{k}^{\sampleOne}+\alpha\big(A(x_{k})\theta_{k}^{\sampleOne}+b(x_{k})\big),\\
\theta_{k+1}^{\sampleTwo} & =\theta_{k}^{\sampleTwo}+\alpha\big(A(x_{k})\theta_{k}^{\sampleTwo}+b(x_{k})\big),
\end{aligned}
\qquad k=0,1,\ldots
\]
Taking the difference of the two equations above, we see that the
difference $\omega_{k}:=\theta_{k}^{\sampleOne}-\theta_{k}^{\sampleTwo}$
satisfies the following recursion 
\[
\omega_{k+1}=\big(I+\alpha A(x_{k})\big)\cdot\omega_{k},\qquad k=0,1,\ldots
\]
Our key observation is that the above recursion is a special case
of the LSA iteration~\eqref{eq:update-rule} with $\omega_{k}$ as
the variable and $\bmax=\sup_{x\in\cX}\|b(x)\|=0$. Consequently,
the pilot result in Proposition~\ref{prop:pilot} can be invoked
to obtain the following geometric convergence bound for $\omega_{k}$:
\[
\E[\|\omega_{k}\|^{2}]\leq C(A,b,\pi)\left(1-\frac{0.9\alpha}{\gammax}\right)^{k}\E[\|\omega_{0}\|^{2}].
\]

We next judiciously choose the conditional distribution of $\theta_{0}^{\sampleTwo}$
given $(x_{0},\theta_{0}^{\sampleOne})$ such that $(x_{k},\theta_{k}^{\sampleTwo})\overset{\text{d}}{=}(x_{k+1},\theta_{k+1}^{\sampleOne})$
for all $k\geq0$, where $\overset{\text{d}}{=}$ denotes equality
in distribution. Then, it follows from the above geometric convergence
bound that 
\[
\Bar W_{2}^{2}(\law(x_{k},\theta_{k}^{\sampleOne}),\law(x_{k+1},\theta_{k+1}^{\sampleOne}))\le\E\big[\|\theta_{k}^{\sampleOne}-\theta_{k}^{\sampleTwo}\|^{2}\big]\to0\quad\text{as \ensuremath{k\to\infty}}.
\]
As such, $(x_{k},\theta_{k}^{\sampleOne})_{k\ge0}$ is a Cauchy sequence
and hence converges to a unique limit $(x_{\infty},\theta_{\infty})$
with the limiting distribution $\bar{\mu}:=\law((x_{\infty},\theta_{\infty}))$.
This proves Part 1 of Theorem~\ref{thm:thm-converge}.

We next show that $\bar{\mu}$ is the invariant distribution of the
Markov chain $(x_{k},\theta_{k})_{k\ge0}$. This invariance property
would follow easily if one could establish the one-step contraction
property 
\[
\Bar W_{2}\left(\law\big(x_{1},\theta_{1}\big),\law\big(x_{1}',\theta_{1}'\big)\right)\le\rho\cdot\Bar W_{2}\left(\law\big(x_{0},\theta_{0}\big),\law\big(x_{0}',\theta_{0}'\big)\right)
\]
for any two copies of the LSA trajectory~\eqref{eq:update-rule}
and some $\rho\in[0,1).$ In fact, this is the approach taken in \citep{Dieuleveut20-bach-SGD}
for analyzing SGD under i.i.d.\ noise. For our Markovian data setting,
however, establishing one-step contraction is challenging if not impossible.
Thankfully, to prove invariance of $\bar{\mu}$, it suffices to have
the following weaker property 
\begin{align}
 \Bar W_{2}^{2}\left(\law\big(x_{1},\theta_{1}\big),\law\big(x_{1}',\theta_{1}'\big)\right)
\leq & \rho_{1}\cdot\Bar W_{2}^{2}\left(\law\big(x_{0},\theta_{0}\big),\law\big(x_{0}',\theta_{0}'\big)\right)+\sqrt{\rho_{2}\cdot\Bar W_{2}^{2}\left(\law\big(x_{0},\theta_{0}\big),\law\big(x_{0}',\theta_{0}'\big)\right)},\label{eq:near_nonexpansive_sketch}
\end{align}
where $\law\big(x_{0},\theta_{0}\big)=\bar{\mu}$ and the quantities
$\rho_{1}$ and $\rho_{2}$ are finite and independent of $\law(x_{0}',\theta_{0}')$.
We establish the property~\eqref{eq:near_nonexpansive_sketch} by
using a second type of coupling between $\big(x_{k},\theta_{k}\big)_{k\ge0}$
and $\big(x_{k}',\theta_{k}'\big)_{k\ge0}$, such that 
\begin{align*}
\Bar W_{2}^{2}\left(\law\big(x_{0},\theta_{0}\big),\law\big(x_{0}',\theta_{0}'\big)\right) & =\E\left[d_{0}\big(x_{0},x_{0}'\big)+\big\|\theta_{0}-\theta_{0}'\big\|^{2}\right]\quad\text{and}\\
x_{k+1} & =x_{k+1}'\;\;\text{ if }x_{k}=x_{k}',\quad\forall k\ge0.
\end{align*}
That is, the two underlying Markov chains $(x_{k})_{k\ge0}$ and $(x_{k}')_{k\ge0}$
evolve separately until they reach the same state, after which they
coalesce and follow the same trajectory. Given the property~\eqref{eq:near_nonexpansive_sketch},
for any $k\ge0$, if we set $\law(x_{0},\theta_{0})=\bar{\mu}$ and
$\law(x_{0}',\theta_{0}')=\law(x_{k},\theta_{k})$, then 
\[
\Bar W_{2}^{2}(\law(x_{1},\theta_{1}),\law(x_{k+1},\theta_{k+1}))\leq\rho_{1}\cdot\Bar W_{2}^{2}(\Bar{\mu},\law(x_{k},\theta_{k}))+\sqrt{\rho_{2}\cdot\Bar W_{2}^{2}(\Bar{\mu},\law(x_{k},\theta_{k}))}.
\]
It follows from the triangle inequality of Wasserstein distance that
\begin{align*}
\Bar W_{2}(\law(x_{1},\theta_{1}),\Bar{\mu}) & \leq\Bar W_{2}(\law(x_{1},\theta_{1}),\law(x_{k+1},\theta_{k+1}))+W_{2}(\law(x_{k+1},\theta_{k+1}),\Bar{\mu})\\
 & \leq\sqrt{\rho_{1}\cdot\Bar W_{2}^{2}(\Bar{\mu},\law(x_{k},\theta_{k}))+\sqrt{\rho_{2}\cdot\Bar W_{2}^{2}(\Bar{\mu},\law(x_{k},\theta_{k}))}}+W_{2}(\law(x_{k+1},\theta_{k+1}),\Bar{\mu})\\
 & \longrightarrow0\quad\textnormal{as }k\to\infty,
\end{align*}
which establishes the invariance of $\bar{\mu}$ and proves Part 2
of Theorem~\ref{thm:thm-converge}.

Finally, the non-asymptotic bound in Part 3 of Theorem~\ref{thm:thm-converge}
follows from the non-asymptotic bound on $\omega_{k}$ and invariance
property of $\bar{\mu}$ established above.

\subsection{Proof Outline of Theorem~\ref{thm:bias-characterization}}

\label{sec:thm-bias-proof-sketch}

We outline the proof of Theorem~\ref{thm:bias-characterization}.
The complete proof is given in Appendix~\ref{sec:bias-proof}.

As discussed in Section~\ref{sec:bias-expansion}, our proof centers
around the condition expectations $\E\left[\theta_{\infty}\mid x_{\infty}=x\right]$,
$x\in\mathcal{X}.$ To characterize these quantities, we make use
of the Basic Adjoint Relationship (BAR), which states that under the 
 stationary
distribution $\bar{\mu}$ and for any test function $f:\real^d\times\mathcal{X} \to \real $ with at most quadratic growth, it holds that 
\[
\E_{\Bar{\mu}}\left[f(x_{\infty},\theta_{\infty})\right]=\E_{\Bar{\mu}}\left[f(x_{\infty+1},\theta_{\infty+1})\right].
\]
For the purpose of characterizing the first moment of $\theta_\infty$, it suffices to consider the test functions $f^{(E)}(x,\theta):=\theta\cdot\indic\{x\in E\}$
for each measurable subset $E\in\cB(\cX).$ 
This choice allows us
to establish 
\[
\E[\theta_{\infty}\mid x_{\infty}=x]=\E[\theta_{\infty+1}\mid x_{\infty+1}=x],
\]
which leads to the following recursive relationship: 
\[
\E[\theta_{\infty}\mid x_{\infty}=x]=\int_{s\in\cX}P^{\ast}(x,\dd s)\Big(\E[\theta_{\infty}\mid x_{\infty}=s]+\alpha\Big(A(s)\E[\theta_{\infty}|x_{\infty}=s]+b(s)\Big)\Big).
\]
Define the function $z:\cX\to\R^{d}$ such that $z(x)=\E[\theta_{\infty}\mid x_{\infty}=x]$.
The above recursion can be written succinctly as $z=P^{\ast}\big(z+\alpha(\dA z+b)\big).$
Inserting the identity $\Pi(\dA z+b)=(1\otimes\pi)(\dA z+b)=0$, we
obtain 
\begin{equation}
z=P^{\ast}z+\alpha(P^{\ast}-\Pi)(\dA z+b).\label{eq:xi-relationship-step1-sketch-operator2}
\end{equation}

Next, we choose $\pi z=\E[\theta_{\infty}]$ as the reference point
and define the function $\delta:\cX\to\R^{d}$ by $\delta(x)=z(x)-\pi z$.
We subtract $\pi z$ from both sides of \eqref{eq:xi-relationship-step1-sketch-operator2}
to obtain 
\[
\delta=(P^{\ast}-\Pi)\delta+\alpha(P^{\ast}-\Pi)(\dA z+b).
\]
Consolidating the terms involving $\delta$ and using the invertibility
of the operator $(I-\Padj+\Pi)$ under the uniform ergodicity assumption,
we obtain 
\begin{equation}
\delta=\alpha(I-P^{\ast}+\Pi)^{-1}(P^{\ast}-\Pi)(\dA z+b).\label{eq:delta-vec-sketch}
\end{equation}

On the other hand, due to the structure of $\delta$, we can obtain
the following identity:
\begin{equation}
\pi z=\theta^{\ast}-\pi\dAbar\delta.\label{eq:z1-to-thetastar-v1-sketch}
\end{equation}
Substituting (\ref{eq:z1-to-thetastar-v1-sketch}) into the definition
$\delta:=z-\pi z$, we obtain 
\begin{equation}
z=\theta^{\ast}+\big(I-\Pi\dAbar\big)\delta.\label{eq:z-theta-operator-sketch2}
\end{equation}

We now combine \eqref{eq:z-theta-operator-sketch2} with \eqref{eq:delta-vec-sketch}
to eliminate the variable $z$, thereby establishing the following
self-expressing equation for $\delta$:
\[
\delta=\alpha\Upsilon+\alpha\Xi\delta,
\]
where $\Xi$ and $\Upsilon$ are defined in \eqref{eq:xi-def} and
\eqref{eq:upsilon-def}, respectively. Continuing this bootstrapping argument for $m$ steps, we obtain the expansion 
\begin{equation}
\delta=\sum_{i=1}^{m}\alpha^{i}\Xi^{i-1}\Upsilon+\alpha^{m}\Xi^{m}\delta.\label{eq:Delta-expansion-sketch}
\end{equation}

We next note that equation~\eqref{eq:delta-vec-sketch}, together
with the bound $\E\left[\theta_{\infty}\right]=\bigO(1)$ (which follows
from Theorem~\ref{thm:thm-converge}), imply the coarse bound $\|\delta\|_{\ltwopi}=\bigO(\alpha)$.
Applying this bound to the second RHS term in~\eqref{eq:Delta-expansion-sketch}
and substituting the resulting expression for $\delta$ into \eqref{eq:z1-to-thetastar-v1-sketch},
we obtain the desired expansion for $\E[\theta_{\infty}]$ given in
Theorem~\ref{thm:bias-characterization}: 
\begin{equation}
\E[\theta_{\infty}]=\theta^{\ast}+\sum_{i=1}^{m}\alpha^{i}B^{(i)}+\bigO(\alpha^{m+1}).\label{eq:bias-expansion-sketch}
\end{equation}
Moreover, when the stepsize $\alpha$ satisfies $\left\Vert \alpha\Xi\right\Vert _{\ltwopi}<1$,
one can take $m\to\infty$ in the expansions~\eqref{eq:Delta-expansion-sketch}
and~\eqref{eq:bias-expansion-sketch} and obtain the infinite series
expansion.

\section{Conclusion}
\label{sec:conclusion}

In this paper, we study linear stochastic approximation with constant
stepsizes and Markovian data. We analyze the convergence rates to
a limiting distribution and identify the existence of an asymptotic
bias. We characterize the bias as a function of the stepsize and mixing
time, and rigorously establish the benefit of Richardson-Romberg extrapolation.
Our results provide a refined characterization of linear stochastic
approximation, elucidating the effect of stepsize, averaging, and extrapolation
on the optimization error, variance, and bias.

Based on our work, immediate next steps include tightening the dimension
dependence in our bounds and relaxing the uniform ergodicity assumption.
Further future directions include: (a) study higher moments of the
errors and provide high probability bounds; (b) investigate extension
of our results to nonlinear stochastic approximation; (c) exploit
our results to guide the choice and scheduling of the stepsize.

\subsection*{Acknowledgements}

Y.\ Chen is supported in part by NSF CAREER Award CCF-2233152 and
grant CCF-1704828. Q.\ Xie is supported in part by NSF grant CNS-1955997
and a J.P.\ Morgan Faculty Research Award. We would like to thank
Jim Dai for inspiring discussion.

\bibliographystyle{alpha}
\bibliography{cite}

\newcommand{\etalchar}[1]{$^{#1}$}
\begin{thebibliography}{MLW{\etalchar{+}}20}

\bibitem[BCD{\etalchar{+}}21]{Meyn21_ode}
Vivek~S. Borkar, Shuhang Chen, Adithya Devraj, Ioannis Kontoyiannis, and
  Sean~P. Meyn.
\newblock The {ODE} method for asymptotic statistics in stochastic
  approximation and reinforcement learning, 2021.

\bibitem[Ber19]{Bertsekas19-RL-book}
Dimitri~P. Bertsekas.
\newblock {\em Reinforcement learning and Optimal Control}.
\newblock Athena Scientific, Belmont, Massachusetts, USA, 2019.

\bibitem[BJN{\etalchar{+}}20]{guy2020}
Guy Bresler, Prateek Jain, Dheeraj Nagaraj, Praneeth Netrapalli, and Xian Wu.
\newblock Least squares regression with markovian data: Fundamental limits and
  algorithms.
\newblock In {\em Proceedings of the 34th International Conference on Neural
  Information Processing Systems}, NIPS'20, Red Hook, NY, USA, 2020. Curran
  Associates Inc.

\bibitem[Blu54]{Blum54-SA}
Julius~R. Blum.
\newblock Approximation methods which converge with probability one.
\newblock {\em The Annals of Mathematical Statistics}, 25(2):382 -- 386, 1954.

\bibitem[BM00]{borkar2000-ode-sa}
Vivek~S. Borkar and Sean~P. Meyn.
\newblock The {O.D.E.} method for convergence of stochastic approximation and
  reinforcement learning.
\newblock {\em SIAM Journal on Control and Optimization}, 38(2):447–469, Jan
  2000.

\bibitem[BM13]{bach2013}
Francis Bach and Eric Moulines.
\newblock Non-strongly-convex smooth stochastic approximation with convergence
  rate o(1/n).
\newblock In C.J. Burges, L.~Bottou, M.~Welling, Z.~Ghahramani, and K.Q.
  Weinberger, editors, {\em Advances in Neural Information Processing Systems},
  volume~26. Curran Associates, Inc., 2013.

\bibitem[BMP12]{Benveniste12-sa-book}
Albert Benveniste, Michel Metivier, and Pierre Priouret.
\newblock {\em Adaptive Algorithms and Stochastic Approximations}.
\newblock Springer Berlin Heidelberg, 1st edition, 2012.

\bibitem[Bor08]{borkar08-SA-book}
Vivek~S. Borkar.
\newblock {\em Stochastic Approximation: A Dynamical Systems Viewpoint}.
\newblock Hindustan Book Agency Gurgaon, 2008.

\bibitem[BRS21]{Bhandari21-linear-td}
Jalaj Bhandari, Daniel Russo, and Raghav Singal.
\newblock A finite time analysis of temporal difference learning with linear
  function approximation.
\newblock {\em Operations Research}, 69(3):950--973, May 2021.

\bibitem[CBD22]{Chandak22-QLearn}
Siddharth Chandak, Vivek~S. Borkar, and Parth Dodhia.
\newblock Concentration of contractive stochastic approximation and
  reinforcement learning.
\newblock {\em Stochastic Systems}, Jul 2022.

\bibitem[CDLS99]{Cowell07-graph_prob}
Robert~G. Cowell, A.~Philip Dawid, Steffen~L. Lauritzen, and David~J.
  Spiegelhalter.
\newblock {\em Probabilistic Networks and Expert Systems: Exact Computational
  Methods for Bayesian Networks}.
\newblock Information Science and Statistics. Springer Publishing Company,
  Incorporated, New York, NY, USA, 1st edition, 1999.

\bibitem[CMM22]{chen21-siva_asymptotic}
Zaiwei Chen, Shancong Mou, and Siva~Theja Maguluri.
\newblock Stationary behavior of constant stepsize sgd type algorithms: An
  asymptotic characterization.
\newblock {\em Proceedings of the ACM on Measurement and Analysis of Computing
  Systems}, 6(1), 02 2022.

\bibitem[CMSS20]{chen20-contract-SA}
Zaiwei Chen, Siva~Theja Maguluri, Sanjay Shakkottai, and Karthikeyan Shanmugam.
\newblock Finite-sample analysis of contractive stochastic approximation using
  smooth convex envelopes.
\newblock In H.~Larochelle, M.~Ranzato, R.~Hadsell, M.F. Balcan, and H.~Lin,
  editors, {\em Advances in Neural Information Processing Systems}, volume~33,
  pages 8223--8234. Curran Associates, Inc., 2020.

\bibitem[CMSS21a]{chen21-offpolicy}
Zaiwei Chen, Siva~Theja Maguluri, Sanjay Shakkottai, and Karthikeyan Shanmugam.
\newblock Finite-sample analysis of off-policy {TD}-learning via generalized
  {B}ellman operators.
\newblock In M.~Ranzato, A.~Beygelzimer, Y.~Dauphin, P.S. Liang, and J.~Wortman
  Vaughan, editors, {\em Advances in Neural Information Processing Systems},
  volume~34, pages 21440--21452. Curran Associates, Inc., 2021.

\bibitem[CMSS21b]{chen21-finite-td}
Zaiwei Chen, Siva~Theja Maguluri, Sanjay Shakkottai, and Karthikeyan Shanmugam.
\newblock A {L}yapunov theory for finite-sample guarantees of asynchronous
  {Q-Learning} and {TD-Learning} variants, 2021.

\bibitem[Day92]{Dayan92-tdlambda}
Peter Dayan.
\newblock The convergence of {TD}($\lambda$) for general $\lambda$.
\newblock {\em Machine Learning}, 8(3):341--362, May 1992.

\bibitem[DD11]{Dai11-bar-paper}
J.~G. Dai and Antonius~B. Dieker.
\newblock Nonnegativity of solutions to the {B}asic {A}djoint {R}elationship
  for some diffusion processes.
\newblock {\em Queueing Systems}, 68(3):295, Jul 2011.

\bibitem[DDB20]{Dieuleveut20-bach-SGD}
Aymeric Dieuleveut, Alain Durmus, and Francis Bach.
\newblock {Bridging the gap between constant step size stochastic gradient
  descent and {M}arkov chains}.
\newblock {\em The Annals of Statistics}, 48(3):1348 -- 1382, 2020.

\bibitem[DMN{\etalchar{+}}21]{durmus2021-LSA}
Alain Durmus, Eric Moulines, Alexey Naumov, Sergey Samsonov, Kevin Scaman, and
  Hoi-To Wai.
\newblock Tight high probability bounds for linear stochastic approximation
  with fixed stepsize.
\newblock In M.~Ranzato, A.~Beygelzimer, Y.~Dauphin, P.S. Liang, and J.~Wortman
  Vaughan, editors, {\em Advances in Neural Information Processing Systems},
  volume~34, pages 30063--30074. Curran Associates, Inc., 2021.

\bibitem[DMNS22]{durmus22-LSA}
Alain Durmus, Eric Moulines, Alexey Naumov, and Sergey Samsonov.
\newblock Finite-time high-probability bounds for {P}olyak-{R}uppert averaged
  iterates of linear stochastic approximation, 2022.

\bibitem[DMPS18]{Douc2018}
Randal Douc, Eric Moulines, Pierre Priouret, and Philippe Soulier.
\newblock {\em Markov Chains}.
\newblock Springer Cham, 1st edition, 2018.

\bibitem[DS94]{Dayan1994-tdlambda}
Peter Dayan and Terrence~J. Sejnowski.
\newblock {TD}($\lambda$) converges with probability 1.
\newblock {\em Machine Learning}, 14(3):295--301, Mar 1994.

\bibitem[DSTM18]{Dalal18-lineartd0}
Gal Dalal, Bal\'{a}zs Sz\"{o}r\'{e}nyi, Gugan Thoppe, and Shie Mannor.
\newblock Finite sample analyses for {TD}(0) with function approximation.
\newblock In {\em Proceedings of the Thirty-Second AAAI Conference on
  Artificial Intelligence and Thirtieth Innovative Applications of Artificial
  Intelligence Conference and Eighth AAAI Symposium on Educational Advances in
  Artificial Intelligence}, AAAI'18/IAAI'18/EAAI'18, New Orleans, Louisiana,
  USA, Apr 2018. AAAI Press.

\bibitem[DT22]{dong22-sgd}
Jing Dong and Xin~T. Tong.
\newblock Stochastic gradient descent with dependent data for offline
  reinforcement learning, 2022.

\bibitem[EDMB03]{EvenDar04-QLearn-rates}
Eyal Even-Dar, Yishay Mansour, and Peter Bartlett.
\newblock Learning rates for {Q-Learning}.
\newblock {\em Journal of Machine Learning Research}, 5(1):1–25, Dec 2003.

\bibitem[FG97]{fristedt1997}
Bert~E. Fristedt and Lawerence~F. Gray.
\newblock {\em A Modern Approach to Probability Theory}.
\newblock Probability and Its Applications. Springer New York, 1997.

\bibitem[Fol99]{Folland1999}
Gerald~B. Folland.
\newblock {\em Real analysis: modern techniques and their applications}.
\newblock Wiley, New York, 2nd ed. edition, 1999.

\bibitem[Har85]{Harrison1985brownian}
J.~Michael Harrison.
\newblock {\em Brownian motion and Stochastic Flow Systems}.
\newblock Wiley, New York, NY, USA, 1985.

\bibitem[HW87]{Harrison1987RBM}
J.~Michael Harrison and Ruth~J. Williams.
\newblock Multidimensional reflected {Brownian} motions having exponential
  stationary distributions.
\newblock {\em The Annals of Probability}, 15(1):115--137, Jan 1987.

\bibitem[JKK{\etalchar{+}}18]{Jain18-tail-avg}
Prateek Jain, Sham~M. Kakade, Rahul Kidambi, Praneeth Netrapalli, and Aaron
  Sidford.
\newblock Parallelizing stochastic gradient descent for least squares
  regression: Mini-batching, averaging, and model misspecification.
\newblock {\em The Journal of Machine Learning Research}, 18(223):1--42, 2018.

\bibitem[KP00]{koller2000-pi}
Daphne Koller and Ronald Parr.
\newblock Policy iteration for factored {MDPs}.
\newblock In {\em Proceedings of the Sixteenth Conference on Uncertainty in
  Artificial Intelligence}, UAI'00, page 326–334, San Francisco, CA, USA,
  2000. Morgan Kaufmann Publishers Inc.

\bibitem[KPR{\etalchar{+}}21]{Khamaru21-td-instance}
Koulik Khamaru, Ashwin Pananjady, Feng Ruan, Martin~J. Wainwright, and
  Michael~I. Jordan.
\newblock Is temporal difference learning optimal? an instance-dependent
  analysis.
\newblock {\em SIAM Journal on Mathematics of Data Science}, 3(4):1013--1040,
  Jan 2021.

\bibitem[KY03]{kushner2003-yin-sa-book}
Harold~J. Kushner and G.~George Yin.
\newblock {\em Stochastic Approximation and Recursive Algorithms and
  Applications}.
\newblock Stochastic Modelling and Applied Probability. Springer, New York, NY,
  USA, 2nd edition, 2003.

\bibitem[LM22]{meyn22-meanshift}
Caio~Kalil Lauand and Sean~P. Meyn.
\newblock Markovian foundations for quasi-stochastic approximation with
  applications to extremum seeking control, 2022.

\bibitem[LP03]{Lagoudakis03-LSPI}
Michail~G. Lagoudakis and Ronald Parr.
\newblock Least-squares policy iteration.
\newblock {\em The Journal of Machine Learning Research}, 4:1107–1149, Dec
  2003.

\bibitem[LP17]{Levin17-mixing_book}
David~A. Levin and Yuval Peres.
\newblock {\em Markov Chains and Mixing Times}.
\newblock American Mathematical Society, Providence, Rhode Island, USA, 2nd
  edition, 2017.

\bibitem[LS18]{Lakshminarayanan18-LSA-Constant-iid}
Chandrashekar Lakshminarayanan and Csaba Szepesv{\'a}ri.
\newblock Linear stochastic approximation: How far does constant step-size and
  iterate averaging go?
\newblock In Amos Storkey and Fernando Perez-Cruz, editors, {\em Proceedings of
  the Twenty-First International Conference on Artificial Intelligence and
  Statistics}, volume~84 of {\em Proceedings of Machine Learning Research},
  pages 1347--1355. PMLR, 09--11 Apr 2018.

\bibitem[MLW{\etalchar{+}}20]{Mou20-LSA-iid}
Wenlong Mou, Chris~Junchi Li, Martin~J. Wainwright, Peter~L. Bartlett, and
  Michael~I. Jordan.
\newblock On linear stochastic approximation: Fine-grained {P}olyak-{R}uppert
  and non-asymptotic concentration.
\newblock In Jacob Abernethy and Shivani Agarwal, editors, {\em Proceedings of
  Thirty Third Conference on Learning Theory}, volume 125 of {\em Proceedings
  of Machine Learning Research}, pages 2947--2997. PMLR, 09--12 Jul 2020.

\bibitem[MPWB21]{Mou21-optimal-linearSA}
Wenlong Mou, Ashwin Pananjady, Martin~J. Wainwright, and Peter~L. Bartlett.
\newblock Optimal and instance-dependent guarantees for {Markovian} linear
  stochastic approximation, 2021.

\bibitem[MT09]{Meyn12_book}
Sean~P. Meyn and Richard~L. Tweedie.
\newblock {\em Markov Chains and Stochastic Stability}.
\newblock Cambridge Mathematical Library. Cambridge University Press,
  Cambridge, 2nd edition, 2009.

\bibitem[Pau15]{Paulin2015}
Daniel Paulin.
\newblock {Concentration inequalities for Markov chains by Marton couplings and
  spectral methods}.
\newblock {\em Electronic Journal of Probability}, 20(none):1 -- 32, 2015.

\bibitem[PJ92]{Polyak92-Avg}
Boris~T. Polyak and Anatoli~B. Juditsky.
\newblock Acceleration of stochastic approximation by averaging.
\newblock {\em SIAM Journal on Control and Optimization}, 30(4):838–855, Jul
  1992.

\bibitem[Pol90]{polyak90_average}
Boris~T. Polyak.
\newblock New stochastic approximation type procedures.
\newblock {\em Automation and Remote Control}, 51(7):98--107, Jul 1990.

\bibitem[RM51]{Robbins51-Monro-SA}
Herbert Robbins and Sutton Monro.
\newblock A stochastic approximation method.
\newblock {\em The Annals of Mathematical Statistics}, 22(3):400 -- 407, 1951.

\bibitem[Rup88]{Ruppert88-Avg}
David Ruppert.
\newblock Efficient estimations from a slowly convergent {Robbins-Monro}
  process.
\newblock Technical report, Cornell University, February 1988.

\bibitem[SB02]{bulirsch2002numerical_analysis}
Josef Stoer and Roland Bulirsch.
\newblock {\em Introduction to Numerical Analysis}.
\newblock Springer, New York, NY, USA, 3rd edition, 2002.

\bibitem[SB18]{Sutton18-RL-book}
Richard~S. Sutton and Andrew~G. Barto.
\newblock {\em Reinforcement Learning: An Introduction}.
\newblock A Bradford Book, Cambridge, MA, USA, 2018.

\bibitem[Sha74]{Shapiro74-lyapunov-eig-bound}
Eliezer Shapiro.
\newblock On the {L}yapunov matrix equation.
\newblock {\em IEEE Transactions on Automatic Control}, 19(5):594--596, 1974.

\bibitem[Sut88]{Sutton1988-td}
Richard~S. Sutton.
\newblock Learning to predict by the methods of temporal differences.
\newblock {\em Machine Learning}, 3(1):9--44, Aug 1988.

\bibitem[SY19]{srikant-ying19-finite-LSA}
Rayadurgam Srikant and Lei Ying.
\newblock Finite-time error bounds for linear stochastic approximation and {TD}
  learning.
\newblock In Alina Beygelzimer and Daniel Hsu, editors, {\em Proceedings of the
  Thirty-Second Conference on Learning Theory}, volume~99 of {\em Proceedings
  of Machine Learning Research}, pages 2803--2830. PMLR, 25--28 Jun 2019.

\bibitem[Sze97]{Szepesvari97-QLearn-Rates}
Csaba Szepesv\'{a}ri.
\newblock The asymptotic convergence-rate of {Q-Learning}.
\newblock In M.~Jordan, M.~Kearns, and S.~Solla, editors, {\em Advances in
  Neural Information Processing Systems}, volume~10. MIT Press, 1997.

\bibitem[Tsi94]{Tsitsiklis1994-QLearn}
John~N. Tsitsiklis.
\newblock Asynchronous stochastic approximation and {Q-Learning}.
\newblock {\em Machine Learning}, 16(3):185--202, Sep 1994.

\bibitem[TVR97]{Tsitsiklis97-td_paper}
John~N. Tsitsiklis and Benjamin Van~Roy.
\newblock An analysis of temporal-difference learning with function
  approximation.
\newblock {\em IEEE Transactions on Automatic Control}, 42(5):674--690, 1997.

\bibitem[Vil09]{Villani08-ot_book}
C\'{e}dric Villani.
\newblock {\em Optimal Transport: Old and New}.
\newblock Grundlehren der mathematischen Wissenschaften. Springer Berlin
  Heidelberg, 2009.

\bibitem[WD92]{Watkins92-QLearning}
Christopher J. C.~H. Watkins and Peter Dayan.
\newblock {Q-Learning}.
\newblock {\em Machine Learning}, 8(3):279--292, May 1992.

\bibitem[WR22]{WrightRecht2022_OptBook}
Stephen~J. Wright and Benjamin Recht.
\newblock {\em Optimization for Data Analysis}.
\newblock Cambridge University Press, 2022.

\bibitem[YBVE21]{Yu21-stan-SGD}
Lu~Yu, Krishnakumar Balasubramanian, Stanislav Volgushev, and Murat~A. Erdogdu.
\newblock An analysis of constant step size {SGD} in the non-convex regime:
  Asymptotic normality and bias.
\newblock In M.~Ranzato, A.~Beygelzimer, Y.~Dauphin, P.S. Liang, and J.~Wortman
  Vaughan, editors, {\em Advances in Neural Information Processing Systems},
  volume~34, pages 4234--4248. Curran Associates, Inc., 2021.

\end{thebibliography}

\newpage{}

\appendix

\section{Proofs}
\label{sec:proofs}

In this section, we prove our pilot result in Section~\ref{sec:proof_sketch}
and our main results in Section~\ref{sec:main}.

Recall that $\tau\equiv\tau_{\alpha}$ is the $\alpha$-mixing time
defined in Section~\ref{sec:assumption-section}. In the sequel,
we frequently use the fact that $\alpha\tau\le\frac{1}{4}$ if $\alpha$
satisfies the condition~\eqref{eq:alpha-constraint}. This fact follows
from combining the condition~\eqref{eq:alpha-constraint} with the
lower bound 
\begin{equation}
\gammax\geq\gammin\overset{\textnormal{(i)}}{\geq}\frac{1}{2s_{1}(\BarA)}\overset{\textnormal{(ii)}}{\geq}\frac{1}{2},\label{eq:gammax-lower-bound}
\end{equation}
where the inequality (i) is given in \citep{Shapiro74-lyapunov-eig-bound},
and the inequality (ii) holds under Assumption~\ref{assumption:bounded}.

We also repeatedly use the following independence property: 
\begin{equation}
(\theta_{0},x_{0},\theta_{1},x_{1},\ldots,\theta_{k})\indep(x_{k+1},x_{k+2},\ldots)\;\;\big\vert\;\;x_{k},\quad\forall k\ge1.\label{eq:conditional-independence}
\end{equation}
Consequently, we have $\theta_{k}\indep x_{k+1}\;\vert\;x_{k}$ for
all $k\ge1$. These facts can be proved by direct calculation. Alternatively,
one may verify that the joint distribution of $(x_{k},\theta_{k})_{k\ge0}$
obeys the Markov property with respect to the directed acyclic graph
in the right pane of Figure~\ref{fig:dag-lsa}, hence the aforementioned
independence properties follow from standard results on directed probabilistic
graphical models \citep[Corollary 5.11 and Theorem 5.14]{Cowell07-graph_prob}.

\subsection{Proof of Proposition~\ref{prop:pilot}}
\label{sec:proof_prop:pilot}

We prove our pilot result in Proposition~\ref{prop:pilot}, which
upper-bounds the MSE $\E[\|\theta_{k}-\theta^{*}\|^{2}]$.

We argue that it suffices to prove Proposition~\ref{prop:pilot}
in the special case where the expected value $\bar{b}$ defined in
\eqref{eq:bar-limit} is assumed to be 0. When the LSA update rule
in equation~\eqref{eq:update-rule} has a general $\bar{b}$, we
can center the update by subtracting $\theta^{\ast}$ from both sides
of \eqref{eq:update-rule}, which gives 
\begin{equation}
\theta_{k+1}-\theta^{\ast}=\theta_{k}-\theta^{\ast}+\alpha\big[A(x_{k})(\theta_{k}-\theta^{\ast})+b(x_{k})+A(x_{k})\theta^{\ast}\big].\label{eq:equiv-lsa}
\end{equation}
Setting $\theta_{k}':=\theta_{k}-\theta^{\ast}$ and $b'(x_{k}):=b(x_{k})+A(x_{k})\theta^{\ast}$,
we rewrite equation~\eqref{eq:equiv-lsa} as 
\begin{equation}
\theta_{k+1}'=\theta_{k}'+\alpha\big[A(x_{k})\theta_{k}'+b'(x_{k})\big].\label{eq:equiv-update-rule}
\end{equation}
Equation~\eqref{eq:equiv-update-rule} is an LSA update in the variable
$(\theta_{k}')$ and satisfies 
\begin{align*}
\bar{b}' & :=\lim_{k\to\infty}\E[b'(x_{k})]\\
 & =\lim_{k\to\infty}\E[b(x_{k})]+\E[A(x_{k})]\theta^{\ast}\\
 & =\bar{b}+\BarA\theta^{\ast}=0,
\end{align*}
where the last equality holds since $\theta^{\ast}$ is defined as
the solution to $\E_{\pi}[A(x)]\theta+\E_{\pi}[b(x)]=0$. As such,
we have obtained a LSA of $\theta_{k}'$ with $\Bar b'=0$.

Let $\bmax':=\sup_{x\in\cX}\|b'(x)\|$. The convergence rate of the
new LSA update~\eqref{eq:equiv-update-rule} is given in the following
proposition, which is a centered version of Proposition~\ref{prop:pilot}.

\begin{proposition} \label{prop:pilot_equiv} Under Assumptions~\ref{assumption:uniform-ergodic},~\ref{assumption:bounded}
and~\ref{assumption:hurwitz}, if $\alpha$ satisfies equation~\eqref{eq:alpha-constraint},
then the update~\eqref{eq:equiv-update-rule} with $\bar{b}'=0$
satisfies for all $k\ge\tau$, 
\[
\E[\|\theta'_{k}\|^{2}]\leq\frac{\gammax}{\gammin}\left(1-\frac{0.9\alpha}{\gammax}\right)^{k-\tau}\Big(5\E[\|\theta'_{0}\|^{2}]+(\bmax')^{2}\Big)+\frac{\gammax}{0.9\gammin}\cdot\alpha\tau\Big(160\gammax(\bmax')^{2}\Big).
\]
\end{proposition}

We prove the above proposition in Appendix~\ref{sec:proof_pilot_equiv}.
Taking Proposition~\ref{prop:pilot_equiv} as given, we now complete
the proof of the general Proposition~\ref{prop:pilot}.

\begin{proof}[Proof of Proposition~\ref{prop:pilot}] By definition
of $b'$, we have $\|b'(x)\|\leq\|b(x)\|+\|A(x)\|\|\theta^{\ast}\|,\forall x\in\cX,$
whence 
\begin{align*}
\bmax' & \le\bmax+\Amax\|\theta^{\ast}\|\\
 & \leq\big(1+\Amax/\smin{\BarA}\big)\bmax\leq2s_{\min}^{-1}(\BarA)\bmax.
\end{align*}
Substituting $\theta_{k}'=\theta_{k}-\theta^{\ast}$ and the above
bound into Proposition~\ref{prop:pilot_equiv}, we obtain that for
all $k\ge\tau$, 
\begin{align*}
\E[\|\theta_{k}-\theta^{\ast}\|^{2}] & \leq5\,\frac{\gammax}{\gammin}\left(1-\frac{0.9\alpha}{\gammax}\right)^{k-\tau}\Big(\E[\|\theta_{0}-\theta^{\ast}\|^{2}]+s_{\min}^{-2}(\BarA)\bmax^{2}\Big)\\
 & +\frac{\gammax}{0.9\gammin}\cdot\alpha\tau\Big(640\gammax s_{\min}^{-2}(\BarA)\bmax^{2}\Big).
\end{align*}
We can simplify the above expression using the following simple bound,
whose proof is postponed to the end of this sub-sub-section. 

\begin{claim} \label{claim1} We have $\left(1-\frac{0.9\alpha}{\gammax}\right)^{-\tau}\leq2.$
\end{claim} 

Using Claim~\ref{claim1} and the definition of $\kappa$ in equation~\eqref{eq:kappa-def},
we obtain that for all $k\ge\tau$, 
\[
\E[\|\theta_{k}-\theta^{\ast}\|^{2}]\leq10\,\frac{\gammax}{\gammin}\left(1-\frac{0.9\alpha}{\gammax}\right)^{k}\Big(\E[\|\theta_{0}-\theta^{\ast}\|^{2}]+s_{\min}^{-2}(\BarA)\bmax^{2}\Big)+\alpha\tau\cdot\kappa.
\]
As such, we have completed the proof of Proposition~\ref{prop:pilot}.
\end{proof}

\begin{proof}[Proof of Claim~\ref{claim1}]

Observe that 
\begin{equation}
\frac{0.9\alpha}{\gammax}\overset{\textnormal{(i)}}{\leq}\frac{0.9\alpha\tau}{\gammax}\overset{\textnormal{(ii)}}{\leq}2\alpha\tau\overset{\textnormal{(iii)}}{\leq}\frac{1}{2},\label{eq:alpha-gammax-ratio-bound}
\end{equation}
where step (i) holds since $\tau\geq1$, step (ii) follows from the
bound \eqref{eq:gammax-lower-bound}, and step (iii) holds since $\alpha\tau\leq\frac{1}{4}$
under the stepsize condition~\eqref{eq:alpha-constraint}. To proceed,
we use the Bernoulli inequality $(1+x)^{t}\geq1+xt,\forall x\ge-1,t\ge1,$
which is equivalent to $(1-x)^{-t}\leq(1-xt)^{-1},\forall x\in(0,1),t\in[1,1/x).$
 In light of equation~\eqref{eq:alpha-gammax-ratio-bound}, the Bernoulli
inequality holds with $x=\frac{0.9\alpha}{\gammax}$ and $t=\tau$,
hence 
\[
\left(1-\frac{0.9\alpha}{\gammax}\right)^{-\tau}\leq\frac{1}{1-\frac{0.9\alpha\tau}{\gammax}}\leq2,
\]
where the last step follows from \eqref{eq:alpha-gammax-ratio-bound}.
We have completed the proof of Claim~\ref{claim1}. \end{proof}

\subsubsection{Proof of Proposition~\ref{prop:pilot_equiv}}
\label{sec:proof_pilot_equiv}

To prove Proposition~\ref{prop:pilot_equiv}, we need the following technical lemmas.

\begin{lem} \label{lem:s-lemma3} Given any $t\geq1$, if $\alpha\cdot t\leq\frac{1}{4}$,
then the following inequalities hold for all $k\geq t$, 
\begin{align}
\|\theta_{k}-\theta_{k-t}\| & \leq2\alpha t\|\theta_{k-t}\|+2\alpha t\bmax,\label{eq:lem3-1}\\
\|\theta_{k}-\theta_{k-t}\| & \leq4\alpha t\|\theta_{k}\|+4\alpha t\bmax,\label{eq:lem3-2}\\
\|\theta_{k}-\theta_{k-t}\|^{2} & \leq32\alpha^{2}t^{2}\|\theta_{k}\|^{2}+32\alpha^{2}t^{2}\bmax^{2}.\label{eq:lem3-3}
\end{align}
\end{lem}

\begin{lem} \label{lem:s-lemma4} The following inequality holds
for any $k\geq0$, 
\[
\left\vert (\theta_{k+1}-\theta_{k})^{\top}\Gamma(\theta_{k+1}-\theta_{k})\right\vert \leq2\alpha^{2}\gammax\|\theta_{k}\|^{2}+2\alpha^{2}\gammax\bmax^{2}.
\]
\end{lem}

\begin{lem} \label{lem:s-lemma5-a} The following inequality holds
for all $k\geq\tau$, with $\alpha$ chosen sufficiently small such
that $\alpha\tau\leq\frac{1}{4}$, 
\[
\E\left[\theta_{k}^{\top}\Gamma(A(x_{k})-\BarA)\theta_{k}\right]\leq\kappa_{1}\E[\|\theta_{k}\|^{2}]+\kappa_{2},
\]
where 
\[
\kappa_{1}=88\alpha\tau\gammax\quad\text{and}\quad\kappa_{2}=64\alpha\tau\gammax\bmax^{2}.
\]

\end{lem}

\begin{lem} \label{lem:s-lemma5-b} The following inequality holds
for all $k\geq\tau$, with $\alpha$ chosen sufficiently small such
that $\alpha\tau\leq\frac{1}{4}$, 
\[
\E\left[\theta_{k}^{\top}\Gamma(b(x_{k})-\bar{b})\right]\leq\widetilde{\kappa}_{1}\E[\|\theta_{k}\|^{2}]+\widetilde{\kappa}_{2},
\]
where 
\[
\widetilde{\kappa}_{1}=5\alpha\tau\gammax\quad\text{and}\quad\widetilde{\kappa}_{2}=15\alpha\tau\gammax\bmax^{2}.
\]
\end{lem}

The proofs of the technical lemmas above are delayed to Appendix~\ref{sec:proof-technical-lemma}.
Note that all lemmas above hold for the LSA update~\eqref{eq:update-rule}
with general $\bar{b}$. Below we shall apply these lemmas to the
centered LSA update~\eqref{eq:equiv-update-rule} for $\theta_{k}'$
with $\bar{b}'=0$ to prove Proposition~\ref{prop:pilot_equiv}.

\begin{proof}[Proof of Proposition~\ref{prop:pilot_equiv}] Consider
the following drift: 
\begin{align*}
 & \E[{\theta_{k+1}'}^{\top}\Gamma\theta_{k+1}'-{\theta_{k}'}^{\top}\Gamma\theta_{k}']\\
= & 2\E[{\theta_{k}'}^{\top}\Gamma(\theta_{k+1}'-\theta_{k}')]+\E[(\theta_{k+1}'-\theta_{k}')^{\top}\Gamma(\theta_{k+1}'-\theta_{k}')]\\
= & 2\alpha\underbrace{\E[{\theta_{k}'}^{\top}\Gamma(A(x_{k})-\BarA)\theta_{k}']}_{T_{1}}+2\alpha\underbrace{\E[{\theta_{k}'}^{\top}\Gamma b'(x_{k})]}_{T_{2}}+2\alpha\underbrace{\E[{\theta_{k}'}^{\top}\Gamma\BarA\theta_{k}']}_{T_{3}}+\underbrace{\E[(\theta'_{k+1}-\theta'_{k})^{\top}\Gamma(\theta'_{k+1}-\theta'_{k})]}_{T_{4}}.
\end{align*}
We can bound $T_{1}$ using Lemma \ref{lem:s-lemma5-a}, $T_{2}$
using Lemma \ref{lem:s-lemma5-b}, and $T_{4}$ using Lemma \ref{lem:s-lemma4}.
For $T_{3}$, we note that by the property of the Lyapunov equation
in Assumption \ref{assumption:hurwitz}, 
\[
2\alpha\E[{\theta_{k}'}^{\top}\Gamma\BarA\theta_{k}']=\alpha\E[{\theta_{k}'}^{\top}\underbrace{(\BarA^{\top}\Gamma+\Gamma\BarA)}_{=-I}\theta_{k}']=-\alpha\E[\|\theta_{k}'\|^{2}].
\]
Combining the above bounds, we derive that 
\begin{align*}
 & \E\big[{\theta_{k+1}'}^{\top}\Gamma\theta_{k+1}'-{\theta_{k}'}^{\top}\Gamma\theta_{k}'\big]\\
= & T_{1}+T_{2}+T_{3}+T_{4}\\
\leq & 2\alpha\left(\kappa_{1}\E[\|\theta_{k}'\|^{2}]+\kappa_{2}\right)+2\alpha\left(\Tilde{\kappa}_{1}\E[\|\theta_{k}'\|^{2}]+\Tilde{\kappa}_{2}\right)-\alpha\E[\|\theta_{k}'\|^{2}]+\left(2\alpha^{2}\gammax\E[\|\theta_{k}'\|^{2}]+2\alpha^{2}\gammax(\bmax')^{2}\right)\\
= & -\alpha(1-2(\kappa_{1}+\Tilde{\kappa}_{1}+\alpha\gammax))\E[\|\theta_{k}'\|^{2}]+2\alpha(\kappa_{2}+\Tilde{\kappa}_{2}+\alpha\gammax(\bmax')^{2}).
\end{align*}

We simplify the above bound by noting that 
\[
\kappa_{1}+\Tilde{\kappa}_{1}+\alpha\gammax=88\alpha\tau\gammax+5\alpha\tau\gammax+\alpha\gammax\leq95\alpha\tau\gammax,
\]
and 
\[
\kappa_{2}+\Tilde{\kappa}_{2}+\alpha\gammax(\bmax')^{2}=64\alpha\tau\gammax(\bmax')^{2}+15\alpha\tau\gammax(\bmax')^{2}+\alpha\gammax(\bmax')^{2}\leq80\alpha\tau\gammax(\bmax')^{2}.
\]
Combining with the fact that $\alpha$ satisfies \eqref{eq:alpha-constraint},
we obtain that for all $k\geq\tau$, 
\[
\E[{\theta_{k+1}'}^{\top}\Gamma\theta_{k+1}'-{\theta_{k}'}^{\top}\Gamma\theta_{k}']\leq-\frac{0.9\alpha}{\gammax}\E[{\theta_{k}'}^{\top}\Gamma\theta_{k}']+160\alpha^{2}\tau\gammax(\bmax')^{2},
\]
or equivalently 
\[
\E[{\theta_{k+1}'}^{\top}\Gamma\theta_{k+1}']\leq\left(1-\frac{0.9\alpha}{\gammax}\right)\E[{\theta_{k}'}^{\top}\Gamma\theta_{k}']+160\alpha^{2}\tau\gammax(\bmax')^{2}.
\]

Next, we recursively apply the above inequality to obtain 
\begin{align*}
\E[{\theta_{k}'}^{\top}\Gamma\theta_{k}'] & \leq\left(1-\frac{0.9\alpha}{\gammax}\right)^{k-\tau}\E[{\theta_{\tau}'}^{\top}\Gamma\theta_{\tau}']+\sum_{t=0}^{(k-\tau)-1}\left(1-\frac{0.9\alpha}{\gammax}\right)^{t}\cdot\left(160\alpha^{2}\tau\gammax(\bmax')^{2}\right)\\
 & \leq\left(1-\frac{0.9\alpha}{\gammax}\right)^{k-\tau}\E[{\theta_{\tau}'}^{\top}\Gamma\theta_{\tau}']+\frac{\gammax}{0.9}\cdot\left(160\alpha\tau\gammax(\bmax')^{2}\right).
\end{align*}
We then apply the properties in \eqref{eq:gamma-property} to the
above inequality and obtain the following bounds in terms of $\|\theta'_{k}\|^{2}$,
for $k\geq\tau$, 
\[
\E[\|\theta'_{k}\|^{2}]\leq\frac{1}{\gammin}\E[{\theta_{k}'}^{\top}\Gamma\theta_{k}']\leq\frac{\gammax}{\gammin}\left(1-\frac{0.9\alpha}{\gammax}\right)^{k-\tau}\E[\|\theta_{\tau}'\|^{2}]+\frac{\gammax}{0.9\gammin}\cdot\alpha\tau\left(160\gammax(\bmax')^{2}\right).
\]

Lastly, we have 
\begin{align*}
\|\theta'_{\tau}\|_{2}^{2} & \leq\left(\|\theta'_{\tau}-\theta'_{0}\|+\|\theta'_{0}\|\right)^{2}\\
 & \overset{\textnormal{(i)}}{\leq}\left((1+2\alpha\tau)\|\theta'_{0}\|+2\alpha\tau\bmax'\right)^{2}\\
 & \overset{\textnormal{(ii)}}{\leq}(1.5\|\theta'_{0}\|+0.5\bmax')^{2}\leq5\|\theta'_{0}\|^{2}+(\bmax')^{2},
\end{align*}
where in step (i) we make use of Lemma~\ref{lem:s-lemma3} to bound
$\|\theta'_{\tau}-\theta'_{0}\|$ with $\|\theta'_{0}\|$, and step
(ii) holds for $\alpha$ is chosen according to \eqref{eq:alpha-constraint}
such that $\alpha\tau<\frac{1}{4}$. Therefore, we have 
\[
\E[\|\theta'_{k}\|^{2}]\leq\frac{\gammax}{\gammin}\left(1-\frac{0.9\alpha}{\gammax}\right)^{k-\tau}\left(5\E[\|\theta'_{0}\|^{2}]+(\bmax')^{2}\right)+\frac{\gammax}{0.9\gammin}\cdot\alpha\tau\left(160\gammax(\bmax')^{2}\right).
\]
This concludes the proof for Proposition~\ref{prop:pilot_equiv}.
\end{proof}

\subsubsection{Proof of Technical Lemmas}
\label{sec:proof-technical-lemma}

We prove the technical lemmas stated at the beginning of the previous
sub-sub-section.

\begin{proof}[Proof of Lemma~\ref{lem:s-lemma3}] Since $\theta_{k+1}=\theta_{k}+\alpha(A(x_{k})\theta_{k}+b(x_{k}))$,
we have 
\[
\|\theta_{k+1}\|\leq\|I+\alpha A(x_{k})\|\|\theta_{k}\|+\alpha\|b(x_{k})\|\leq(1+\alpha)\|\theta_{k}\|+\alpha\bmax.
\]
Therefore, for $k-t<i\leq k$, we have 
\begin{align}
\|\theta_{i}\| & \leq(1+\alpha)^{i-(k-t)}\|\theta_{k-t}\|+\alpha\bmax\sum_{j=0}^{(i-1)-(k-t)}(1+\alpha)^{j}\nonumber \\
 & \leq(1+\alpha)^{t}\|\theta_{k-t}\|+\alpha\bmax\sum_{j=0}^{t-1}(1+\alpha)^{j}\overset{\textnormal{(i)}}{\leq}(1+2\alpha t)\|\theta_{k-t}\|+2\alpha t\bmax,\label{eq:theta-i-bound}
\end{align}
where step (i) holds since $\alpha t\leq\frac{1}{4}\leq\log2$. It
follows that 
\begin{align*}
\|\theta_{k}-\theta_{k-t}\| & =\bigg\|\sum_{i=k-t}^{k-1}\theta_{i+1}-\theta_{i}\bigg\|\leq\sum_{i=k-t}^{k-1}\|\theta_{i+1}-\theta_{i}\|=\alpha\sum_{i=k-t}^{k-1}\|A(x_{k})\theta_{i}+b(x_{k})\|\\
 & \leq\alpha\Amax\left(\sum_{i=k-t}^{k-1}\|\theta_{i}\|\right)+\alpha t\bmax &  & \text{by Assumption \ref{assumption:bounded} }\\
 & \leq\alpha\Amax\left(\sum_{i=k-t}^{k-1}(1+2\alpha t)\|\theta_{k-t}\|+2\alpha t\bmax\right)+\alpha t\bmax &  & \text{by \eqref{eq:theta-i-bound}}\\
 & =(1+2\alpha t)\left(\alpha t\Amax\|\theta_{k-t}\|+\alpha t\bmax\right)\\
 & \overset{\textnormal{(ii)}}{\leq}2\alpha t(\Amax\|\theta_{k-t}\|+\bmax)<2\alpha t\|\theta_{k-t}\|+2\alpha t\bmax,
\end{align*}
where step (ii) holds since $2\alpha t<1$. As such, we have established \eqref{eq:lem3-1}.

With \eqref{eq:lem3-1}, it is easy to see that 
\[
\|\theta_{k}-\theta_{k-t}\|\leq2\alpha t\|\theta_{k-t}\|+2\alpha t\bmax\leq2\alpha t\left(\|\theta_{k}-\theta_{k-t}\|+\|\theta_{k}\|\right)+2\alpha t\bmax.
\]
Reorganizing the above inequality gives $(1-2\alpha t)\|\theta_{k}-\theta_{k-t}\|\leq2\alpha t\|\theta_{k}\|+2\alpha t\bmax.$
Together with $\alpha t\leq\frac{1}{4}$, we obtain~\eqref{eq:lem3-2}.
Lastly, we use the inequality $(a+b)^{2}\leq2(a^{2}+b^{2})$ to obtain
\[
\|\theta_{k}-\theta_{k-t}\|^{2}\leq\left(4\alpha t\|\theta_{k}\|+4\alpha t\bmax\right)^{2}\le32\alpha^{2}t^{2}\|\theta_{k}\|^{2}+32\alpha^{2}t^{2}\bmax^{2},
\]
thereby proving the desired bound in~\eqref{eq:lem3-3}. \end{proof}

\begin{proof}[Proof of Lemma~\ref{lem:s-lemma4}] We have 
\begin{align*}
\left\vert (\theta_{k+1}-\theta_{k})^{\top}\Gamma(\theta_{k+1}-\theta_{k})\right\vert  & \leq\gammax\|\theta_{k+1}-\theta_{k}\|^{2}\\
 & =\alpha^{2}\gammax\|A(x_{k})\theta_{k}+b(x_{k})\|^{2}\\
 & \leq\alpha^{2}\gammax\left(\Amax\|\theta_{k}\|+\bmax\right)^{2}\\
 & \leq2\alpha^{2}\gammax\|\theta_{k}\|^{2}+2\alpha^{2}\gammax\bmax^{2}.
\end{align*}
This completes the proof of Lemma~\ref{lem:s-lemma4}. \end{proof}

\begin{proof}[Proof of Lemma~\ref{lem:s-lemma5-a}] 

Let us decompose the quantity of interest as 
\begin{align*}
 & \E\left[\theta_{k}^{\top}\Gamma(A(x_{k})-\BarA)\theta_{k}\right]\\
= & \E\left[(\theta_{k}-\theta_{k-\tau}+\theta_{k-\tau})^{\top}\Gamma(A(x_{k})-\BarA)(\theta_{k}-\theta_{k-\tau}+\theta_{k-\tau})\right]\\
= & \underbrace{\E\left[(\theta_{k}-\theta_{k-\tau})^{\top}\Gamma(A(x_{k})-\BarA)(\theta_{k}-\theta_{k-\tau})\right]}_{T_{1}}+\underbrace{\E\left[\theta_{k-\tau}^{\top}\Gamma(A(x_{k})-\BarA)\theta_{k-\tau}\right]}_{T_{2}}\\
+ & \underbrace{\E\left[(\theta_{k}-\theta_{k-\tau})^{\top}\Gamma(A(x_{k})-\BarA)\theta_{k-\tau}\right]}_{T_{3}}+\underbrace{\E\left[\theta_{k-\tau}^{\top}\Gamma(A(x_{k})-\BarA)(\theta_{k}-\theta_{k-\tau})\right]}_{T_{4}}.
\end{align*}
We bound each of the RHS terms respectively.

For $T_{1}$, we have 
\begin{align*}
T_{1} & =\E\left[(\theta_{k}-\theta_{k-\tau})^{\top}\Gamma(A(x_{k})-\BarA)(\theta_{k}-\theta_{k-\tau})\right]\\
 & \overset{\textnormal{(i)}}{\leq}2\gammax\E\left[\|\theta_{k}-\theta_{k-\tau}\|^{2}\right]\\
 & \overset{\textnormal{(ii)}}{\leq}2\gammax\left(32\alpha^{2}\tau^{2}\E[\|\theta_{k}\|^{2}]+32\alpha^{2}\tau^{2}\bmax^{2}\right)\\
 & \leq64\gammax\alpha^{2}\tau^{2}\E[\|\theta_{k}\|^{2}]+64\gammax\alpha^{2}\tau^{2}\bmax^{2},
\end{align*}
where (i) holds since both $\left\Vert A(x_{k})\right\Vert \le1$
and $\left\Vert \BarA\right\Vert \le1$ by Assumption~\ref{assumption:bounded}
and $\Gamma$ is symmetric with top eigenvalue $\gammax$ by Assumption~\ref{assumption:hurwitz},
and (ii) follows from equation \eqref{eq:lem3-3} in Lemma~\ref{lem:s-lemma3}.

For $T_{2}$, we have 
\begin{align*}
T_{2} & =\E\bigg[\E\left[\theta_{k-\tau}^{\top}\Gamma(A(x_{k})-\BarA)\theta_{k-\tau}\mid\theta_{k-\tau},x_{k-\tau}\right]\bigg]\\
 & =\E\bigg[\theta_{k-\tau}^{\top}\Gamma\E\left[A(x_{k})-\BarA\mid\theta_{k-\tau},x_{k-\tau}\right]\theta_{k-\tau}\bigg]\\
 & \overset{\textnormal{(iii)}}{=}\E\bigg[\theta_{k-\tau}^{\top}\Gamma\E\left[A(x_{k})-\BarA\mid x_{k-\tau}\right]\theta_{k-\tau}\bigg].
\end{align*}
where step (iii) follows from  conditional independence property $x_{k}\indep\theta_{k-\tau}\mid x_{k-\tau}$
shown in equation~\eqref{eq:conditional-independence}. Since $\Gamma$
has largest eigenvalue $\gammax$ by Assumption~\ref{assumption:hurwitz}
and $\tau\equiv\tau_{\alpha}$ is the $\alpha$-mixing time, which
by definition ensures that $A(x_k)$ is sufficiently close to $\BarA$
in expectation, it follows that 
\begin{align*}
T_{2} & \leq\alpha\gammax\E\left[\|\theta_{k-\tau}\|^{2}\right]\\
 & \leq\alpha\gammax\E\left[\big(\|\theta_{k}-\theta_{k-\tau}\|+\|\theta_{k}\|\big)^{2}\right]\\
 & \leq\alpha\gammax\E\left[\big(4\alpha\tau\|\theta_{k}\|+4\alpha\tau\bmax+\|\theta_{k}\|\big)^{2}\right]\qquad\textnormal{by \eqref{eq:lem3-2}}\\
 & \overset{\textnormal{(iv)}}{\leq}\alpha\gammax\cdot2\left((1+4\alpha\tau)^{2}\E[\|\theta_{k}\|^{2}]+16\alpha^{2}\tau^{2}\bmax^{2}\right)\\
 & \leq8\alpha\tau\gammax\E[\|\theta_{k}\|^{2}]+32\alpha^{3}\tau^{2}\gammax\bmax^{2},
\end{align*}
where (iv) follows from the inequality $(a+b)^{2}\leq2(a^{2}+b^{2})$,
and the last step holds since $\alpha\tau\leq\frac{1}{4}$ and $\tau\geq1$.

For $T_{3}$, we have 
\begin{align*}
T_{3} & =\E\left[(\theta_{k}-\theta_{k-\tau})^{\top}\Gamma(A(x_{k})-\BarA)\theta_{k-\tau}\right]\\
 & \leq2\gammax\E\big[\|\theta_{k}-\theta_{k-\tau}\|\cdot\left(\|\theta_{k}-\theta_{k-\tau}\|+\|\theta_{k}\|\right)\big]\\
 & \leq2\gammax\E\big[(4\alpha\tau\|\theta_{k}\|+4\alpha\tau\bmax)(4\alpha\tau\|\theta_{k}\|+4\alpha\tau\bmax+\|\theta_{k}\|)\big]\qquad\text{by \eqref{eq:lem3-2}}\\
 & =8\alpha\tau(1+4\alpha\tau)\gammax\E\left[\|\theta_{k}\|^{2}\right]+8\alpha\tau(1+8\alpha\tau)\gammax\bmax\E[\|\theta_{k}\|]+32\alpha^{2}\tau^{2}\gammax\bmax^{2}\\
 & \overset{\textnormal{(v)}}{\leq}8\alpha\tau(1+4\alpha\tau)\gammax\E\left[\|\theta_{k}\|^{2}\right]+4\alpha\tau(1+8\alpha\tau)\gammax\left(\bmax^{2}+\E[\|\theta_{k}\|^{2}]\right)+32\alpha^{2}\tau^{2}\gammax\bmax^{2}\\
 & =4\alpha\tau\gammax\big(2(1+4\alpha\tau)+(1+8\alpha\tau)\big)\E[\|\theta_{k}\|^{2}]+4\alpha\tau\gammax\big((1+8\alpha\tau)+8\alpha\tau\big)\bmax^{2}\\
 & \overset{\textnormal{(vi)}}{\leq}32\alpha\tau\gammax\E[\|\theta_{k}\|^{2}]+20\alpha\tau\gammax\bmax^{2},
\end{align*}
where (v) utilizes the inequality $2\bmax\E[\|\theta_{k}\|]\leq\bmax^{2}+\E[\|\theta_{k}\|^{2}]$,
and (vi) holds with $\alpha\tau\leq\frac{1}{4}$.

Similarly, for $T_{4}$, we have for $\alpha\tau\leq\frac{1}{4}$,
\begin{align*}
T_{4} & =\E\left[\theta_{k-\tau}^{\top}\Gamma(A(x_{k})-\BarA)(\theta_{k}-\theta_{k-\tau})\right]\\
 & \leq32\alpha\tau\gammax\E[\|\theta_{k}\|^{2}]+20\alpha\tau\gammax\bmax^{2}.
\end{align*}

Combining the bounds for $T_{1}$--$T_{4}$, we obtain that 
\begin{align*}
 & \E\left[\theta_{k}^{\top}\Gamma(A(x_{k})-\BarA)\theta_{k}\right]=T_{1}+T_{2}+T_{3}+T_{4}\\
\leq & \left(64\gammax\alpha^{2}\tau^{2}\E[\|\theta_{k}\|^{2}]+64\gammax\alpha^{2}\tau^{2}\bmax^{2}\right)+\left(8\alpha\tau\gammax\E[\|\theta_{k}\|^{2}]+32\alpha^{3}\tau^{2}\gammax\bmax^{2}\right)\\
 &+2\left(32\alpha\tau\gammax\E[\|\theta_{k}\|^{2}]+20\alpha\tau\gammax\bmax^{2}\right)\\
\leq & 88\alpha\tau\gammax\E[\|\theta_{k}\|^{2}]+64\alpha\tau\gammax\bmax^{2},
\end{align*}
where the last step holds with $\alpha\leq1$ and $\alpha\tau\leq\frac{1}{4}$.
This completes the proof of Lemma~\ref{lem:s-lemma5-a}. \end{proof}

\begin{proof}[Proof of Lemma~\ref{lem:s-lemma5-b}] We first make
use of the law of total expectation and obtain that 
\[
\E\left[\theta_{k}^{\top}\Gamma(b(x_{k})-\bar{b})\right]=\E\Big[\E\left[\theta_{k}^{\top}\Gamma(b(x_{k})-\bar{b})\mid\theta_{k-\tau},x_{k-\tau}\right]\Big].
\]
We decompose the inner expectation as
\begin{align*}
 & \E\left[\theta_{k}^{\top}\Gamma(b(x_{k})-\bar{b})\mid\theta_{k-\tau},x_{k-\tau}\right]\\
= & \E\left[(\theta_{k}-\theta_{k-\tau}+\theta_{k-\tau})^{\top}\Gamma(b(x_{k})-\bar{b})\mid\theta_{k-\tau},x_{k-\tau}\right]\\
= & \underbrace{\E\left[(\theta_{k}-\theta_{k-\tau})^{\top}\Gamma(b(x_{k})-\bar{b})\mid\theta_{k-\tau},x_{k-\tau}\right]}_{T_{1}}+\underbrace{\E\left[\theta_{k-\tau}^{\top}\Gamma(b(x_{k})-\bar{b})\mid\theta_{k-\tau},x_{k-\tau}\right]}_{T_{2}}.
\end{align*}
We separately bound $T_{1}$ and $T_{2}$. For $T_{1}$, we have 
\begin{align*}
\E\left[(\theta_{k}-\theta_{k-\tau})^{\top}\Gamma(b(x_{k})-\bar{b})\mid\theta_{k-\tau},x_{k-\tau}\right] & \leq2\bmax\gammax\E\left[\|\theta_{k}-\theta_{k-\tau}\|\mid\theta_{k-\tau},x_{k-\tau}\right]\\
 & \leq2\bmax\gammax\left(2\alpha\tau\|\theta_{k-\tau}\|+2\alpha\tau\bmax\right),
\end{align*}
where we use \eqref{eq:lem3-1} to obtain the last inequality. For
$T_{2}$, we have 
\begin{align*}
\E\left[\theta_{k-\tau}^{\top}\Gamma(b(x_{k})-\bar{b})\mid\theta_{k-\tau},x_{k-\tau}\right] & =\theta_{k-\tau}^{\top}\Gamma\E\left[(b(x_{k})-\bar{b})\mid\theta_{k-\tau},x_{k-\tau}\right]\\
 & \leq\alpha\gammax\bmax\|\theta_{k-\tau}\|.
\end{align*}
Combining the two terms, we have 
\begin{align*}
 & \E\left[\theta_{k}^{\top}\Gamma(b(x_{k})-\bar{b})\mid\theta_{k-\tau},x_{k-\tau}\right]\\
\leq & \alpha\gammax\bmax\|\theta_{k-\tau}\|+2\bmax\gammax\left(2\alpha\tau\|\theta_{k-\tau}\|+2\alpha\tau\bmax\right)\\
= & \alpha\gammax\bmax(1+4\tau)\|\theta_{k-\tau}\|+4\alpha\tau\gammax\bmax^{2}\\
\leq & \alpha\gammax\bmax(1+4\tau)\left(\E[\|\theta_{k}-\theta_{k-\tau}\|\mid\theta_{k-\tau},x_{k-\tau}]+\E[\|\theta_{k}\|\mid\theta_{k-\tau},x_{k-\tau}]\right)+4\alpha\tau\gammax\bmax^{2}\\
\overset{\textnormal{(i)}}{\leq} & \alpha\gammax\bmax(1+4\tau)\left((1+4\alpha\tau)\E[\|\theta_{k}\|\mid\theta_{k-\tau},x_{k-\tau}]+4\alpha\tau\bmax\right)+4\alpha\tau\gammax\bmax^{2}\\
\leq & 10\alpha\tau\gammax\bmax\E\left[\|\theta_{k}\|\mid\theta_{k-\tau},x_{k-\tau}\right]+9\alpha\tau\gammax\bmax^{2},
\end{align*}
where we use \eqref{eq:lem3-2} to obtain (i), and $\alpha\tau\leq\frac{1}{4}$,
$\alpha\leq1$ and $\tau\geq1$ to obtain the last inequality. Using
the inequality $2\bmax\|\theta_{k}\|\leq\bmax^{2}+\|\theta_{k}\|^{2}$,
we simplify the above display equation to 
\begin{align*}
\E\left[\theta_{k}^{\top}\Gamma(b(x_{k})-\bar{b})\mid\theta_{k-\tau},x_{k-\tau}\right] & \leq5\alpha\tau\gammax(\bmax^{2}+\E\left[\|\theta_{k}\|^{2}\mid\theta_{k-\tau},x_{k-\tau}\right])+9\alpha\tau\gammax\bmax^{2}\\
 & \leq5\alpha\tau\gammax\E\left[\|\theta_{k}\|^{2}\mid\theta_{k-\tau},x_{k-\tau}\right]+15\alpha\tau\gammax\bmax^{2}.
\end{align*}
Lastly, we take expectations on both sides of the last display equation
to obtain 
\[
\E\left[\theta_{k}^{\top}\Gamma(b(x_{k})-\bar{b})\right]\leq5\alpha\tau\gammax\E[\|\theta_{k}\|^{2}]+15\alpha\tau\gammax\bmax^{2}.
\]
This completes the proof of Lemma~\ref{lem:s-lemma5-b}.

\end{proof}

\subsection{Proof of Theorem~\ref{thm:thm-converge}}
\label{sec:proof_thm_conv_limit_dist}

In this sub-section, we prove Theorem~\ref{thm:thm-converge} on
the convergence of LSA to a limit.

\subsubsection{Coupling and Geometric Convergence}

Recall that $(x_{k})_{k\ge0}$ is the underlying Markov chain that
drives the LSA iteration~\eqref{eq:update-rule}. We consider a pair
of coupled Markov chains, $(x_{k},\theta_{k}^{\sampleOne})_{k\ge0}$
and $(x_{k},\theta_{k}^{\sampleTwo})_{k\ge0}$, defined as 
\begin{equation}
\begin{aligned}\theta_{k+1}^{\sampleOne} & =\theta_{k}^{\sampleOne}+\alpha\big(A(x_{k})\theta_{k}^{\sampleOne}+b(x_{k})\big),\\
\theta_{k+1}^{\sampleTwo} & =\theta_{k}^{\sampleTwo}+\alpha\big(A(x_{k})\theta_{k}^{\sampleTwo}+b(x_{k})\big),
\end{aligned}
\qquad k=0,1,\ldots\label{eq:coupled_process}
\end{equation}
Note that $(\theta_{k}^{\sampleOne})_{k\ge0}$ and $(\theta_{k}^{\sampleTwo})_{k\ge0}$
are two sample paths of the LSA iteration~\eqref{eq:update-rule},
coupled by sharing the underlying process $(x_{k})_{k\ge0}$. We assume
that the initial iterates $\theta_{0}^{\sampleOne}$ and $\theta_{0}^{\sampleTwo}$
may depend on each other and on $x_{0}$, but are independent of subsequent
$(x_{k})_{k\geq1}$ given $x_{0}$.

It follows from the definition that 
\[
\theta_{k+1}^{\sampleOne}-\theta_{k+1}^{\sampleTwo}=\big(I+\alpha A(x_{k})\big)\cdot(\theta_{k}^{\sampleOne}-\theta_{k}^{\sampleTwo}),\qquad k=0,1,\ldots
\]
If we define the shorthand $\omega_{k}:=\theta_{k}^{\sampleOne}-\theta_{k}^{\sampleTwo}$,
then the above equation becomes 
\begin{equation}
\omega_{k+1}=\big(I+\alpha A(x_{k})\big)\cdot\omega_{k},\qquad k=0,1,\ldots\label{eq:theta-diff-recursion}
\end{equation}
Our key observation is that equation~\eqref{eq:theta-diff-recursion}
is a special case of the LSA iteration~\eqref{eq:update-rule} with
$\omega_{k}$ as the variable and $\bmax=\sup_{x\in\cX}\|b(x)\|=0$.
Applying Proposition~\ref{prop:pilot} to this LSA iteration, we
obtain the following finite-time geometric bound. \begin{cor} \label{delta-cor}
Suppose that $\alpha$ satisfies \eqref{eq:alpha-constraint}. Then,
for all $k\geq\tau$, we have 
\begin{align*}
W_{2}^{2}\Big(\law\big(\theta_{k}^{\sampleOne}\big),\law\big(\theta_{k}^{\sampleTwo}\big)\Big) & \overset{\textnormal{(i)}}{\leq}\Bar W_{2}^{2}\Big(\law\big(x_{k},\theta_{k}^{\sampleOne}\big),\law\big(x_{k},\theta_{k}^{\sampleTwo}\big)\Big)\\
 & \overset{\textnormal{(ii)}}{\leq}\E\Big[\big\|\theta_{k}^{\sampleOne}-\theta_{k}^{\sampleTwo}\big\|^{2}\Big]\\
 & \overset{\textnormal{(iii)}}{\leq}10\,\frac{\gammax}{\gammin}\left(1-\frac{0.9\alpha}{\gammax}\right)^{k}\E\Big[\big\|\theta_{0}^{\sampleOne}-\theta_{0}^{\sampleTwo}\big\|^{2}\Big].
\end{align*}
\end{cor}

\begin{proof}[Proof of Corollary~\ref{delta-cor}] The inequality
(i) follows from the definition of $W_{2}$ and $\Bar W_{2}$. The
inequality (ii) holds since the Wasserstein distance is defined by
an infimum as in equation~\eqref{eq:w2-definition-extended}. Inequality
(iii) follows from applying Proposition~\ref{prop:pilot} with $\bmax=0$
to the LSA iteration~\eqref{eq:theta-diff-recursion}. \end{proof}

With Corollary~\ref{delta-cor}, we are ready to prove Theorem~\ref{thm:thm-converge}
on the convergence of the Markov chain $(x_{k},\theta_{k})_{k\ge0}$.
Theorem~\ref{thm:thm-converge} has three parts, whose proofs are
given in the next three sub-sub-sections.

\subsubsection{Part 1: Existence of Limiting Distribution}

Note that Corollary~\ref{delta-cor} is valid under any joint distribution
of initial iterates $(x_{0},\theta_{0}^{\sampleOne},\theta_{0}^{\sampleTwo}).$
Arbitrarily fix the distribution of $(x_{0},\theta_{0}^{\sampleOne})$.
Given $(x_{0},\theta_{0}^{\sampleOne})$, we shall judiciously choose
the conditional distribution of $\theta_{0}^{\sampleTwo}$ in a way
that ensures $(x_{k},\theta_{k}^{\sampleTwo})\overset{\textup{d}}{=}(x_{k+1},\theta_{k+1}^{\sampleOne})$
for all $k\ge0$, where $\overset{\textup{d}}{=}$ denotes equality
in distribution. Specifically, recall that the adjoint operator $\Padj$
is the transition probability matrix for the time-reversed Markov
chain of $(x_{k})_{k\ge0}$ and that the initial distribution of $x_{0}$
is assumed to be the stationary distribution $\pi$; see Sections~\ref{sec:problem_setup}
and~\ref{sec:assumption-section}. Given $x_{0}$, let $x_{-1}$
be sampled from $\Padj(x_{0},\cdot)$. Let $\theta_{-1}^{\sampleTwo}$
be a random variable which satisfies $\theta_{-1}^{\sampleTwo}\overset{\textup{d}}{=}\theta_{0}^{\sampleOne}$
and is independent of $(x_{k})_{k\ge-1}$. Finally, set $\theta_{0}^{\sampleTwo}$
as 
\begin{equation}
\theta_{0}^{\sampleTwo}=\theta_{-1}^{\sampleTwo}+\alpha\left(A(x_{-1})\theta_{-1}^{\sampleTwo}+b(x_{-1})\right).\label{eq:theta0-sampleTwo-special-construction}
\end{equation}
We argue that this initialization has the desired property.

\begin{claim} \label{claim:theta0-sampleTwo-special-construction}
Under the assumptions in Theorem~\ref{thm:thm-converge} and the
initialization~\eqref{eq:theta0-sampleTwo-special-construction},
we have $(x_{k},\theta_{k}^{\sampleTwo})\overset{\textup{d}}{=}(x_{k+1},\theta_{k+1}^{\sampleOne})$
for all $k\ge0$. \end{claim}

\begin{proof}[Proof of Claim~\ref{claim:theta0-sampleTwo-special-construction}]
From standard results on time-reversed Markov chains, we have $(x_{k})_{k\ge-1}\overset{\textup{d}}{=}(x_{k})_{k\ge0}.$
Since by construction $\theta_{-1}^{\sampleTwo}\overset{\textup{d}}{=}\theta_{0}^{\sampleOne}$
and $\theta_{-1}^{\sampleTwo}$ is independent of $(x_{k})_{k\ge-1}$,
the claim follows from comparing the update rules for $(\theta_{k}^{\sampleOne})_{k\ge0}$
and $(\theta_{k}^{\sampleTwo})_{k\ge-1}$ given in equations~\eqref{eq:coupled_process}
and~\eqref{eq:theta0-sampleTwo-special-construction}. \end{proof}

Using the above claim, we have for all $k\ge\tau$, 
\begin{align*}
\Bar W_{2}^{2}\Big(\law\big(x_{k},\theta_{k}^{\sampleOne}\big),\law\big(x_{k+1},\theta_{k+1}^{\sampleOne}\big)\Big) & =\Bar W_{2}^{2}\Big(\law\big(x_{k},\theta_{k}^{\sampleOne}\big),\law\big(x_{k},\theta_{k}^{\sampleTwo}\big)\Big)\\
 & \leq10\,\frac{\gammax}{\gammin}\left(1-\frac{0.9\alpha}{\gammax}\right)^{k}\E[\|\theta_{0}^{\sampleOne}-\theta_{0}^{\sampleTwo}\|^{2}],
\end{align*}
where in the second step above we use Corollary~\ref{delta-cor}.
It follows that 
\begin{align*}
 & \sum_{k=0}^{\infty}\Bar W_{2}^{2}\Big(\law\big(x_{k},\theta_{k}^{\sampleOne}\big),\law\big(x_{k+1},\theta_{k+1}^{\sampleOne}\big)\Big)\\
\leq & \sum_{k=0}^{\tau-1}\Bar W_{2}^{2}\Big(\law\big(x_{k},\theta_{k}^{\sampleOne}\big),\law\big(x_{k+1},\theta_{k+1}^{\sampleOne}\big)\Big)+10\,\frac{\gammax}{\gammin}\sum_{k=\tau}^{\infty}\left(1-\frac{0.9\alpha}{\gammax}\right)^{k}\E[\|\theta_{0}^{\sampleOne}-\theta_{0}^{\sampleTwo}\|^{2}]\\
< & \infty,
\end{align*}
where the last step holds since $\frac{0.9\alpha}{\gammax}\in(0,1)$
under the assumption~\eqref{eq:alpha-constraint}. The inequality
above means that $\big(\law(x_{k},\theta_{k}^{\sampleOne})\big)_{k\geq0}$
is a Cauchy sequence in the metric $\Bar W_{2}$. Since the space
$\cP_{2}(\cX\times\R^{d})$ endowed with $\bar{W}_{2}$ is a Polish
space \citep[Theorem 6.18]{Villani08-ot_book}, every Cauchy sequence
converges. Furthermore, convergence in Wasserstein distance implies
weak convergence \citep[Theorem 6.9]{Villani08-ot_book}. We conclude
that the sequence $\big(\law(x_{k},\theta_{k}^{\sampleOne})\big)_{k\geq0}$
converges weakly to a limit $\Bar{\mu}\in\cP_{2}(\cX\times\R^{d})$.

We next show that the limit $\Bar{\mu}$ is independent of the initial
distribution of $\theta_{0}^{\sampleOne}$. Recall that $x_{0}$ is
initialized from its unique stationary distribution $\pi$ by Assumption~\ref{assumption:uniform-ergodic}.
Suppose that another sequence $\big(x_{k},\Tilde{\theta}_{k}^{\sampleOne}\big)_{k\geq0}$
with a different initial distribution converges to a limit $\Tilde{\mu}$,
then following from the triangle inequality property for Wasserstein
distance, we obtain 
\begin{equation}
\Bar W_{2}(\Bar{\mu},\Tilde{\mu})\leq\Bar W_{2}\left(\Bar{\mu},\law(x_{k},\theta_{k}^{\sampleOne})\right)+\Bar W_{2}\left(\law(x_{k},\theta_{k}^{\sampleOne}),\law(x_{k},\Tilde{\theta}_{k}^{\sampleOne})\right)+\Bar W_{2}\left(\law(x_{k},\Tilde{\theta}_{k}^{\sampleOne}),\Tilde{\mu}\right)\overset{k\to\infty}{\longrightarrow}0,\label{eq:limit_unique}
\end{equation}
where the last step holds since $\Bar W_{2}\big(\law(x_{k},\theta_{k}^{\sampleOne}),\law(x_{k},\Tilde{\theta}_{k}^{\sampleOne})\big)\overset{k\to\infty}{\longrightarrow}0$
by Corollary~\ref{delta-cor}. Therefore, we have $\Bar W_{2}(\Bar{\mu},\Tilde{\mu})=0$
and hence the limit $\bar{\mu}$ is unique.

Finally, the bound on $\var(\theta_{\infty})$ follows from the lemma
below. Recall that $\kappa$ is defined in \eqref{eq:kappa-def}.

\begin{lem} \label{lem:theta_inf_bound} Under Assumptions~\ref{assumption:uniform-ergodic},~\ref{assumption:bounded}
and~\ref{assumption:hurwitz}, and when $\alpha$ is chosen according
to~\eqref{eq:alpha-constraint}, we have 
\begin{equation}
\tr(\var(\theta_{\infty}))\le\E[\|\theta_{\infty}-\theta^{\ast}\|^{2}]\le\alpha\tau\cdot\kappa\label{eq:theta_inf_var}
\end{equation}
and 
\begin{equation}
\left(\E[\|\theta_{\infty}\|]\right)^{2}\le\E[\|\theta_{\infty}\|^{2}]\le C(A,b,\pi)\label{eq:theta_inf_2nd_moment}
\end{equation}
for some $C(A,b,\pi)$ that is independent of $\alpha$. \end{lem}

\begin{proof}[Proof of Lemma~\ref{lem:theta_inf_bound}] We have
shown that the sequence $(\theta_{k})_{k\geq0}$ converges weakly
to $\theta_{\infty}$ in $\mathcal{P}_{2}(\R^{d})$. It is known that
weak convergence in $\mathcal{P}_{2}(\R^{d})$ is equivalent to convergence
in distribution and the convergence of the first two moments \citep[Definition 6.8]{Villani08-ot_book}.
Consequently, we have 
\begin{equation}
\E[\|\theta_{\infty}-\theta^{\ast}\|^{2}]=\lim_{k\to\infty}\E[\|\theta_{k}-\theta^{\ast}\|^{2}].\label{eq:exchange-limit}
\end{equation}
Proposition~\ref{prop:pilot} presents the following upper bound
on $\E[\|\theta_{k}-\theta^{\ast}\|^{2}]$ that 
\[
\E[\|\theta_{k}-\theta^{\ast}\|^{2}]\leq10\,\frac{\gammax}{\gammin}\left(1-\frac{0.9\alpha}{\gammax}\right)^{k}\Big(\E[\|\theta_{0}-\theta^{\ast}\|^{2}]+s_{\min}^{-2}(\BarA)\bmax^{2}\Big)+\alpha\tau\cdot\kappa.
\]
Taking $k\to\infty$ and combining with equation~\eqref{eq:exchange-limit}
give $\E[\|\theta_{\infty}-\theta^{\ast}\|^{2}]\leq\alpha\tau\cdot\kappa\le\frac{1}{4}\kappa,$
where the last step holds since $\alpha\tau\leq\frac{1}{4}$. Equation~\eqref{eq:theta_inf_var}
follows since $\theta^{*}$ is a deterministic quantity.

Furthermore, we have 
\begin{equation}
\left(\E[\|\theta_{\infty}\|]\right)^{2}\le\E[\|\theta_{\infty}\|^{2}]\leq2\E[\|\theta_{\infty}-\theta^{\ast}\|^{2}]+2\|\theta^{\ast}\|^{2}\leq\frac{1}{2}\kappa+2\|\theta^{\ast}\|^{2}.\label{eq:theta-inf-c-bound}
\end{equation}
Equation~\eqref{eq:theta_inf_2nd_moment} then follows from noting
that $\gammax,\gammin,\kappa$ and $\theta^{\ast}$ only depend on
$A,b$, and $\pi$. \end{proof}

We have proved part 1 of Theorem~\ref{thm:thm-converge}.

\subsubsection{Part 2: Invariance}

We next show that $\Bar{\mu}$ is the unique invariant distribution.
Suppose that the initial distribution of $(x_{0},\theta_{0})$ is
$\Bar{\mu}$. By the triangle inequality of Wasserstein distance,
we have 
\begin{equation}
\Bar W_{2}(\law(x_{1},\theta_{1}),\Bar{\mu})\leq\Bar W_{2}(\law(x_{1},\theta_{1}),\law(x_{k+1},\theta_{k+1}))+W_{2}(\law(x_{k+1},\theta_{k+1}),\Bar{\mu}).\label{eq:2.76}
\end{equation}
We proceed by noting the following lemma, whose proof is given at
the end of this sub-sub-section. 

\begin{lem} \label{lem:near-non-expansive} Let $(x_{k},\theta_{k})_{\ge0}$
and $(x'_{k},\theta'_{k})_{k\ge0}$ be two copies of the LSA trajectory~\eqref{eq:update-rule},
where $\law(x_{0},\theta_{0})=\Bar{\mu}$ and $\law(x'_{0},\theta'_{0})\in\cP_{2}(\cX\times\R^{d})$
is arbitrary. Under Assumptions~\ref{assumption:uniform-ergodic},~\ref{assumption:bounded}
and~\ref{assumption:hurwitz}, and when $\alpha$ is chosen according
to equation~\eqref{eq:alpha-constraint}, we have 
\begin{equation}
\begin{aligned}\Bar W_{2}^{2}(\law(x_{1},\theta_{1}),\law(x'_{1},\theta'_{1}))\leq\rho_{1}\cdot\Bar W_{2}^{2}(\law(x_{0},\theta_{0}),\law(x'_{0},\theta'_{0}))+\sqrt{\rho_{2}\cdot\Bar W_{2}^{2}(\law(x_{0},\theta_{0}),\law(x'_{0},\theta'_{0}))},\end{aligned}
\label{eq:near-non-expansive1}
\end{equation}
where the quantities $\rho_{1}:=1+2(1+\alpha)^{2}+16\alpha^{2}\bmax^{2}<\infty$
and $\rho_{2}:=16\alpha^{2}\cdot\E_{\theta_{0}\sim\mu}\left[\|\theta_{0}\|^{4}\right]<\infty$
are independent of $\law(x'_{0},\theta'_{0})$. In particular, for
any $k\ge0$, if we set $\law(x'_{0},\theta'_{0})=\law(x_{k},\theta_{k})$,
then 
\begin{align}
\Bar W_{2}^{2}(\law(x_{1},\theta_{1}),\law(x_{k+1},\theta_{k+1}))\leq\rho_{1}\cdot\Bar W_{2}^{2}(\Bar{\mu},\law(x_{k},\theta_{k}))+\sqrt{\rho_{2}\cdot\Bar W_{2}^{2}(\Bar{\mu},\law(x_{k},\theta_{k}))}.\label{eq:near-non-expansive2}
\end{align}

\end{lem}

Applying Lemma~\ref{lem:near-non-expansive} to bound the first term
on the RHS of equation~\eqref{eq:2.76}, we obtain that 
\begin{align*}
\Bar W_{2}(\law(x_{1},\theta_{1}),\Bar{\mu}) & \leq\sqrt{\rho_{1}\cdot\Bar W_{2}(\Bar{\mu},\law(x_{k},\theta_{k}))+\sqrt{\rho_{2}\cdot\Bar W_{2}^{2}(\Bar{\mu},\law(x_{k},\theta_{k}))}}+\Bar W_{2}(\law(x_{k+1},\theta_{k+1}),\Bar{\mu})\\
 & \overset{k\to\infty}{\longrightarrow}\;0,
\end{align*}
where the last step follows from the weak convergence result established
in the last sub-sub-section. We therefore conclude that $W_{2}(\law(x_{1},\theta_{1}),\Bar{\mu})=0$
and hence $\Bar{\mu}$ is an invariant distribution of the Markov
chain $(x_{k},\theta_{k})_{k\ge0}.$ The uniqueness of the invariant
distribution follows from a similar argument as in equation~\eqref{eq:limit_unique}.
We have proved part 2 of Theorem~\ref{thm:thm-converge}.

\begin{proof}[Proof of Lemma~\ref{lem:near-non-expansive}]

We choose a coupling between the two processes $(x_{k},\theta_{k})_{k\ge0}$
and $(x'_{k},\theta'_{k})_{k\ge0}$ such that 
\begin{align}
\Bar W_{2}^{2}(\law(x_{0},\theta_{0}),\law(x'_{0},\theta'_{0})) & =\E\left[d_{0}(x_{0},x'_{0})+\|\theta_{0}-\theta'_{0}\|^{2}\right]\quad\text{and}\label{eq:initial_coupling_invariance}\\
x_{k+1} & =x'_{k+1}\;\;\text{ if }x_{k}=x'_{k},\quad\forall k\ge0.\label{eq:subsequent_coupling_invariance}
\end{align}
The existence of a coupling satisfying equation~\eqref{eq:initial_coupling_invariance}
at step $k=0$ is a standard result in optimal transport \citep[Theorem 4.1]{Villani08-ot_book}.
We can ensure equation~\eqref{eq:subsequent_coupling_invariance}
by further coupling the two processes for the subsequent steps $k\ge1$,
such that the two underlying Markov chains $(x_{k})_{k\ge0}$ and
$(x'_{k})_{k\ge0}$ evolve separately (subject to the above coupling
at step $k=0$) until they reach the same state, after which they
coalesce and follow the same trajectory.

To prove Lemma~\ref{lem:near-non-expansive}, we begin by observing that 
\begin{align}
\Bar W_{2}^{2}(\law(x_{1},\theta_{1}),\law(x'_{1},\theta'_{1})) & \leq\E\left[d_{0}(x_{1},x'_{1})+\|\theta_{1}-\theta'_{1}\|^{2}\right].\label{eq:invariance1}
\end{align}
thanks to the definition~\eqref{eq:w2-definition-extended} of $\Bar W_{2}$
using an infimum. Recalling the definition of the discrete metric
$d_{0}(x'_{0},x_{0}):=\indic\left\{ x'_{0}\neq x_{0}\right\} $, we
have the identities 
\begin{align*}
A(x_{0}) & =A(x'_{0})+d_{0}(x'_{0},x_{0})\cdot\big(A(x_{0})-A(x'_{0})\big)\quad\text{and}\\
b(x_{0}) & =b(x'_{0})+d_{0}(x'_{0},x_{0})\cdot\big(b(x_{0})-b(x'_{0})\big).
\end{align*}
The update rule~\eqref{eq:update-rule} together with the above identities
implies that 
\begin{align*}
\theta_{1}-\theta'_{1}= & \theta_{0}+\alpha\big(A(x_{0})\theta_{0}+b(x_{0})\big)-\theta'_{0}-\alpha\big(A(x'_{0})\theta'_{0}+b(x'_{0})\big)\\
= & \big(I+\alpha A(x'_{0})\big)\cdot\big(\theta_{0}-\theta'_{0}\big)+\alpha d_{0}(x'_{0},x_{0})\cdot\big[\big(A(x_{0})-A(x'_{0})\big)\theta_{0}+b(x_{0})-b(x'_{0})\big],
\end{align*}
whence 
\begin{align*}
\left\Vert \theta_{1}-\theta'_{1}\right\Vert  & \le\left\Vert I+\alpha A(x'_{0})\right\Vert \cdot\left\Vert \theta_{0}-\theta'_{0}\right\Vert +\alpha d_{0}(x'_{0},x_{0})\cdot\left\Vert \big(A(x_{0})-A(x'_{0})\big)\theta_{0}+b(x_{0})-b(x'_{0})\right\Vert \\
 & \le(1+\alpha)\left\Vert \theta_{0}-\theta'_{0}\right\Vert +\alpha d_{0}(x'_{0},x_{0})\cdot2\big(\|\theta_{0}\|+\bmax\big),
\end{align*}
where the last step follows from the boundedness Assumption~\ref{assumption:bounded}.
Also note that $d_{0}(x_{1},x'_{1})\le d_{0}(x_{0},x'_{0})$ thanks
to the coupling in equation~\eqref{eq:subsequent_coupling_invariance}.
Combining the above inequalities gives 
\begin{align}
 & \E\left[d_{0}(x_{1},x'_{1})+\left\Vert \theta_{1}-\theta'_{1}\right\Vert ^{2}\right]\nonumber \\
\le & \E\left[d_{0}(x_{0},x'_{0})\right]+2(1+\alpha)^{2}\cdot\E\left[\left\Vert \theta_{0}-\theta'_{0}\right\Vert ^{2}\right]+2\alpha^{2}\cdot\E\left[d_{0}(x'_{0},x_{0})\cdot8(\|\theta_{0}\|^{2}+\bmax^{2})\right].\label{eq:invariance-one-step-upper-bound}
\end{align}
By Cauchy-Schwarz's inequality, we have 
\begin{align}
\E\left[d_{0}(x'_{0},x_{0})\cdot\|\theta_{0}\|^{2}\right]\le\sqrt{\E\left[d_{0}(x'_{0},x_{0})\right]}\sqrt{\E_{\theta_{0}\sim\mu}\left[\|\theta_{0}\|^{4}\right]}.\label{eq:Cauchy-Schwarz-bound}
\end{align}
Moreover, we claim that 
\begin{equation}
\E_{\theta_{0}\sim\mu}\left[\|\theta_{0}\|^{4}\right]=\E\left[\|\theta_{\infty}\|^{4}\right]<\infty.\label{eq:4th-moment}
\end{equation}
This claim follows from a moderate tightening of the result in \citep[Theorem 9]{srikant-ying19-finite-LSA},
which provides sufficient conditions for the existence of higher moments
of $\theta_{\infty}$. In Appendix~\ref{sec:higher_moments}, we
explain how to tighten their result to show that the $4$th moment
exists under our stepsize condition~\eqref{eq:alpha-constraint}.

Combining equations~\eqref{eq:invariance-one-step-upper-bound} and~\eqref{eq:Cauchy-Schwarz-bound}
and recalling the values of $\rho_{1}$ and $\rho_{2}$ given in the
statement of the lemma, we obtain that 
\begin{align}
 & \E\left[d_{0}(x_{1},x'_{1})+\left\Vert \theta_{1}-\theta'_{1}\right\Vert ^{2}\right]\nonumber \\
\le & \rho_{1}\cdot\E\left[d_{0}(x_{0},x'_{0})+\left\Vert \theta_{0}-\theta'_{0}\right\Vert ^{2}\right]+\sqrt{\rho_{2}\cdot\E\left[d_{0}(x_{0},x'_{0})+\left\Vert \theta_{0}-\theta'_{0}\right\Vert ^{2}\right]}\nonumber \\
= & \rho_{1}\cdot\bar{W}_{2}^{2}\left(\law(x_{0},\theta_{0}),\law(x'_{0},\theta'_{0})\right)+\sqrt{\rho_{2}\cdot\bar{W}_{2}^{2}\left(\law(x_{0},\theta_{0}),\law(x'_{0},\theta'_{0})\right)},\label{eq:invariance2}
\end{align}
where the last step from our choice of coupling in equation~\eqref{eq:initial_coupling_invariance}.
Combining equations~\eqref{eq:invariance1} and~\eqref{eq:invariance2}
proves the first equation~\eqref{eq:near-non-expansive1} in Lemma~\ref{lem:near-non-expansive}.
The second equation~\eqref{eq:near-non-expansive2} is then immediate.
\end{proof}

\subsubsection{Part 3: Convergence Rate}

We have established that the joint sequence $\big(\law(x_{k},\theta_{k}^{\sampleOne})\big)_{k\geq0}$
converges weakly to the invariant distribution $\bar{\mu}\in\cP_{2}(\cX\times\R^{d})$.
Consequently, $\big(\law(\theta_{k}^{\sampleOne})\big)_{k\geq0}$
converges weakly to ${\mu}\in\cP_{2}(\R^{d})$, where $\mu$ is the
marginal distribution of $\Bar{\mu}$ over $\R^{d}$. We now characterize
the convergence rate.

Again consider the coupled processes defined in equation~\eqref{eq:coupled_process}.
Suppose that the initial distribution of $(x_{0},\theta_{0}^{\sampleTwo})$
is the invariant distribution $\Bar{\mu}$, hence $\law(x_{k},\theta_{k}^{\sampleTwo})=\Bar{\mu}$
and $\law(\theta_{k}^{\sampleTwo})=\mu$ for all $k\ge0$. Applying
Corollary~\ref{delta-cor}, we have for all $k\geq\tau$, 
\begin{align*}
W_{2}^{2}(\law(\theta_{k}^{\sampleOne}),\mu) & =W_{2}^{2}(\law(\theta_{k}^{\sampleOne}),\law(\theta_{k}^{\sampleTwo}))\\
 & \le\Bar W_{2}^{2}(\law(x_{k},\theta_{k}^{\sampleOne}),\law(x_{k},\theta_{k}^{\sampleTwo}))\\
 & \leq10\,\frac{\gammax}{\gammin}\left(1-\frac{0.9\alpha}{\gammax}\right)^{k}\E[\|\theta_{0}^{\sampleOne}-\theta_{0}^{\sampleTwo}\|^{2}]\\
 & \leq20\,\frac{\gammax}{\gammin}\left(1-\frac{0.9\alpha}{\gammax}\right)^{k}\left(\E[\|\theta_{0}^{\sampleOne}-c\|^{2}]+\E[\|\theta_{\infty}-c\|^{2}]\right),
\end{align*}
where $c$ is an arbitrary constant, and the last step above holds
since the chain $(x_{k},\theta_{k}^{\sampleTwo})_{k\ge0}$ is at stationarity
and hence $\E\|\theta_{0}^{\sampleTwo}\|^{2}=\E\|\theta_{\infty}\|^{2}$.

Hence, taking $c=\E[\theta_{\infty}]$, we have now proven equation~\eqref{eq:w2-thetak-to-mu}
in part 3 of the theorem, 
\begin{align*}
W_{2}^{2}\Big(\law(\theta_{k}^{\sampleOne}),\law(\theta_{\infty})\Big) & \leq20\,\frac{\gammax}{\gammin}\left(1-\frac{0.9\alpha}{\gammax}\right)^{k}\left(\E[\|\theta_{0}^{\sampleOne}-\E[\theta_{\infty}]\|^{2}]+\E[\|\theta_{\infty}-\E[\theta_{\infty}]\|^{2}]\right)\\
 & \leq20\,\frac{\gammax}{\gammin}\left(1-\frac{0.9\alpha}{\gammax}\right)^{k}\left(\E[\|\theta_{0}^{\sampleOne}-\E[\theta_{\infty}]\|^{2}]+\tr(\var(\theta_{\infty}))\right).
\end{align*}

\subsection{Proof of Corollary~\ref{cor:non-asymptotic-bounds}}
\label{sec:non-asymptotic-bounds-proof}

\begin{proof}[Proof of Corollary~\ref{cor:non-asymptotic-bounds}]
By Lemma~\ref{lem:theta_inf_bound}, we have $\E[\|\theta_{\infty}\|^{2}]=\bigO(1)$.
Combining this bound with equation~\eqref{eq:w2-thetak-to-mu} in
Theorem~\ref{thm:thm-converge}, we obtain that for $k\ge\tau$,
\[
W_{2}^{2}(\law(\theta_{k}),\mu)\leq C(A,b,\pi)\cdot\left(1-\frac{0.9\alpha}{\gammax}\right)^{k}.
\]
By \citep[Theorem 4.1]{Villani08-ot_book}, there exists a coupling
between $\theta_{k}$ and $\theta_{\infty}$ such that $W_{2}^{2}(\law(\theta_{k}),\mu)=\E[\|\theta_{k}-\theta_{\infty}\|^{2}].$
Utilizing the above bounds and applying Jensen's inequality twice,
we obtain that 
\[
\|\E[\theta_{k}-\theta_{\infty}]\|^{2}\leq\big(\E[\|\theta_{k}-\theta_{\infty}\|]\big)^{2}\leq\E\left[\|\theta_{k}-\theta_{\infty}\|^{2}\right]\leq C(A,b,\pi)\cdot\left(1-\frac{0.9\alpha}{\gammax}\right)^{k}.
\]
The above bound then
implies the first moment bound in equation~\eqref{eq:first-moment-geometric}.

Turning to the second moment, we observe that 
\begin{align}
 & \left\Vert \E\left[\theta_{k}\theta_{k}^{\top}\right]-\E\left[\theta_{\infty}\theta_{\infty}^{\top}\right]\right\Vert \nonumber \\
= & \left\Vert \E\left[(\theta_{k}-\theta_{\infty})(\theta_{k}-\theta_{\infty})^{\top}\right]+\E\left[\theta_{\infty}(\theta_{k}-\theta_{\infty})^{\top}\right]+\E\left[(\theta_{k}-\theta_{\infty})\theta_{\infty}^{\top}\right]\right\Vert \nonumber \\
\leq & \left\Vert \E\left[(\theta_{k}-\theta_{\infty})(\theta_{k}-\theta_{\infty})^{\top}\right]\right\Vert +\left\Vert \E\left[\theta_{\infty}(\theta_{k}-\theta_{\infty})^{\top}\right]\right\Vert +\left\Vert \E\left[(\theta_{k}-\theta_{\infty})\theta_{\infty}^{\top}\right]\right\Vert \nonumber \\
\leq & \E\left[\left\Vert (\theta_{k}-\theta_{\infty})(\theta_{k}-\theta_{\infty})^{\top}\right\Vert \right]+\E\left[\left\Vert \theta_{\infty}(\theta_{k}-\theta_{\infty})^{\top}\right\Vert \right]+\E\left[\left\Vert (\theta_{k}-\theta_{\infty})\theta_{\infty}^{\top}\right\Vert \right]\nonumber \\
\leq & \E\left[\left\Vert \theta_{k}-\theta_{\infty}\right\Vert ^{2}\right]+2\big(\E\left[\left\Vert \theta_{k}-\theta_{\infty}\right\Vert ^{2}\right]\E\left[\|\theta_{\infty}\|^{2}\right]\big)^{1/2},\label{eq:second-moment-decompose}
\end{align}
where the last inequality~\eqref{eq:second-moment-decompose} holds
true by Cauchy-Schwarz inequality. On the other hand, we have already
established that for $k\ge\tau$, 
\[
\E[\|\theta_{k}-\theta_{\infty}\|^{2}]\leq C(A,b,\pi)\left(1-\frac{0.9\alpha}{\gammax}\right)^{k}\quad\text{and}\quad\E[\|\theta_{\infty}\|^{2}]\leq C'(A,b,\pi).
\]
Substituting the above bounds into the right-hand side of inequality~\eqref{eq:second-moment-decompose},
we obtain equation~\eqref{eq:first-moment-geometric} in Corollary~\ref{cor:non-asymptotic-bounds}.
\end{proof}

\subsection{Proof of Theorem~\ref{thm:bias-characterization}}
\label{sec:bias-proof}

In this sub-section, we prove Theorem~\ref{thm:bias-characterization}
on characterizing the asymptotic bias of LSA. The proof is divided
into four steps, which are given in Appendices~\ref{sec:BAR}--\ref{sec:characterize-z}
to follow.

\subsubsection{Step 1: Basic Adjoint Relationship}

\label{sec:BAR}

Following the strategy discussed after Theorem~\ref{thm:bias-characterization},
we begin by deriving a recursive relationship for the function $z:\mathcal{X}\to\mathbb{R}^{d}$
given by
\[
z(x):=\E\left[\theta_{\infty}\mid x_{\infty}=x\right],
\]
which is well-defined by the Doob-Dynkin Lemma. To put the derivation
in context and to avoid measurability issues on general state space,
we present using the language of Basic Adjoint Relationship (BAR).

Recall that under Assumption~\ref{assumption:uniform-ergodic}, $(x_{k})_{k\ge0}$
is a time-homogeneous Markov chain with transition kernel $P$ and
unique stationary distribution $\pi$. Theorem~\ref{thm:thm-converge}
has demonstrated that the Markov chain $(x_{k},\theta_{k})_{k\ge0}$
also has a unique stationary distribution $\Bar{\mu}$, and $(x_{k},\theta_{k})$
converges in distribution to a limit $(x_{\infty},\theta_{\infty})\sim\Bar{\mu}$,
where $\theta_{\infty}\sim\mu$ and $x_{\infty}\sim\pi$. Given $(x_{\infty},\theta_{\infty})$,
let $x_{\infty+1}$ be the random variable with conditional distribution
$P(x_{\infty},\cdot)$, and $\theta_{\infty+1}=\theta_{\infty}+\alpha\left(A(x_{\infty})\theta_{\infty}+b(x_{\infty})\right)$;
that is, $(x_{\infty+1},\theta_{\infty+1})$ is the state following
$(x_{\infty},\theta_{\infty})$.

Denote by $Q$ the transition kernel of $(x_{k},\theta_{k})_{k\ge0}$.
Since $\Bar{\mu}$ is invariant for $Q$, they satisfy BAR 
\[
\Bar{\mu}(I-Q)f=0
\]
for any test function $f:\cX\times\real^{d}\to\real^{d}$ satisfying
$\|f(x,\theta)\|\le C(1+\|\theta\|^{2}),\forall(x,\theta)$ for some
$C\in\R$ \citep[Definition 6.8 and Theorem 6.9]{Villani08-ot_book}.
The above BAR can be written equivalently as 
\begin{equation}
\E\left[f(x_{\infty},\theta_{\infty})\right]=\Bar{\mu}f=\Bar{\mu}Qf=\E\left[f(x_{\infty+1},\theta_{\infty+1})\right].\label{eq:BAR}
\end{equation}
It is known that equation (\ref{eq:BAR}) with a sufficiently large
class of test functions $f$ completely characterizes the invariant
distribution $\Bar{\mu}$ \citep{Harrison1985brownian,Harrison1987RBM,Dai11-bar-paper}.

For characterization of the first moment $\E[\theta_{\infty}]$, it
suffices to consider the test functions of the form 
\[
f^{(E)}(x,\theta)=\theta\cdot\indic\{x\in E\}\qquad\text{and}\qquad f^{(E,S)}(x,\theta)=\indic\{\theta\in S\}\cdot\indic\{x\in E\}
\]
 for $E\in\cB(\cX)$ and $S\in\cB(\real^{d})$. Substituting $f^{(E)}$
into the BAR (\ref{eq:BAR}) gives 
\begin{equation}
\E\left[\theta_{\infty}\cdot\indic\{x_{\infty}\in E\}\right]=\E\left[\theta_{\infty+1}\cdot\indic\{x_{\infty+1}\in E\}\right].\label{eq:BAR_test}
\end{equation}
We now compute the left and right-hand sides of equation~\eqref{eq:BAR_test}
above. For the left-hand side, we have 
\begin{align*}
\E\left[\theta_{\infty}\cdot\indic\{x_{\infty}\in E\}\right] & =\E\Big[\E\left[\theta_{\infty}\cdot\indic\{x_{\infty}\in E\}\mid x_{\infty}\right]\Big]\\
 & =\int_{\cX}\E[\theta_{\infty}\mid x_{\infty}](x)\,\indic\{x_{\infty}\in E\}(x)\,\pi(\dd x)=\int_{E}\E[\theta_{\infty}\mid x_{\infty}](x)\,\pi(\dd x).
\end{align*}
For the right-hand side, we similarly obtain
\[
\E\left[\theta_{\infty+1}\cdot\indic\{x_{\infty+1}\in E\}\right]=\int_{E}\E[\theta_{\infty+1}\mid x_{\infty+1}](x)\,\pi(\dd x).
\]
Plugging back to equation~\eqref{eq:BAR_test}, we obtain that $\int_{E}\E\left[\theta_{\infty}\mid x_{\infty}\right](x)\pi(\dd x)=\int_{E}\E\left[\theta_{\infty+1}\mid x_{\infty+1}\right](x)\pi(\dd x)$
for any $E\in\cB(\cX)$. By \citep[Proposition 2.23(b)]{Folland1999},
we conclude that 
\begin{equation}
\E\left[\theta_{\infty}\mid x_{\infty}\right](x)=\E\left[\theta_{\infty+1}\mid x_{\infty+1}\right](x)\qquad\text{\ensuremath{\pi}-a.e.}\label{eq:theta-condition-stationary}
\end{equation}

Repeating the above argument for the test function $f^{(E,S)}$, we
obtain that for all $S\in\cB(\real^{d})$:
\begin{equation}
\E\left[\indic\{\theta_{\infty}\in S\}\mid x_{\infty}\right](x)=\E\left[\indic\{\theta_{\infty+1}\in S\}\mid x_{\infty+1}\right](x)\qquad\text{\ensuremath{\pi}-a.e.}\label{eq:theta-indic-condition-stationary}
\end{equation}

\subsubsection{Step 2: Set up System of \texorpdfstring{$z$}{Z}}

We derive another relationship between $\E[\theta_{\infty}\mid x_{\infty}]$
and $\E[\theta_{\infty+1}\mid x_{\infty+1}]$ using the update rule
$\theta_{\infty+1}=\theta_{\infty}+\alpha\big(A(x_{\infty})\theta_{\infty}+b(x_{\infty})\big).$ 

As the state space $(\R^{d},\cB(\R^{d}))$ is Borel, the conditional
expectation $\E[\indic\{\theta_{\infty}\in\cdot\}\mid x_{\infty}=x]$
can be taken as a regular conditional probability measure, which we
denote as $\nu(\cdot,x_{\infty}=x)$. Similarly define the regular
conditional probability measure $\nu'(\cdot,x_{\infty}=x,x_{\infty+1}=x')$.
By~\eqref{eq:theta-indic-condition-stationary}, we have that for
all $S\in\cB(\real^{d})$:
\[
\nu(\theta_{\infty}\in S,x_{\infty}=x)=\nu(\theta_{\infty+1}\in S,x_{\infty+1}=x).
\]
Using the time-reversed kernel $\Padj$, we can write 
\begin{align*}
\nu(\theta_{\infty+1}\in S,x_{\infty+1}=x) & =\int_{s\in\cX}\nu'(\theta_{\infty}\in S',x_{\infty}=s,x_{\infty+1}=x)\,\Padj(x,\dd s)\\
 & \overset{\text{(i)}}{=}\int_{s\in\cX}\nu(\theta_{\infty}\in S',x_{\infty}=s)\,\Padj(x,\dd s),
\end{align*}
where the set $S'\in\cB(\R^{d})$ is determined by $S$ and $x$ via
the aforementioned update rule, and the equality (i) follows from
$\theta_{\infty}\indep x_{\infty+1}\mid x_{\infty}$. Using the above
property of $\nu$, we obtain 
\begin{align*}
\E[\theta_{\infty+1}\mid x_{\infty+1}=x] & =\int_{\theta\in\R^{d}}\theta\,\nu(\dd\theta,x)\\
 & =\int_{\theta'\in\R^{d}}\int_{s\in\cX}\Big(\theta'+\alpha(A(s)\theta'+b(s))\Big)\,\nu(\dd\theta',s)\Padj(x,\dd s)\\
 & =\int_{s\in\cX}\Padj(x,\dd s)\int_{\theta'\in\R^{d}}\Big(\theta'+\alpha(A(s)\theta'+b(s)\Big)\nu(\dd\theta',s)\\
 & =\int_{s\in\cX}\Padj(x,\dd s)\left(\E[\theta_{\infty}|x_{\infty}=s]+\alpha\Big(A(s)\E[\theta_{\infty}|x_{\infty}=s]+b(s)\Big)\right).
\end{align*}
Applying \eqref{eq:theta-condition-stationary} to the left-hand side
above, we obtain the following equation for $\E\left[\theta_{\infty}\mid x_{\infty}\right]$:
\begin{equation}
\E[\theta_{\infty}\mid x_{\infty}=x]=\int_{s\in\cX}\Padj(x,ds)\left(\E[\theta_{\infty}\mid x_{\infty}=s]+\alpha\big(A(s)\E[\theta_{\infty}\mid x_{\infty}=s]+b(s)\big)\right).\label{eq:z-recursion}
\end{equation}
Using the function $z$ and the operator $D$ defined in Section~\ref{sec:bias-expansion},
we can write equation \eqref{eq:z-recursion} compactly as
\begin{equation}
z=\Padj\big(z+\alpha(\dA z+b)\big).\label{eq:xi-relationship-step1}
\end{equation}

\subsubsection{Step 3: Setting up System of \texorpdfstring{$\delta$}{Delta}}

Using $\pi z=\int_{\cX}\pi(\dd x)z(x)=\E[\theta_{\infty}]$ as a reference
point, we define the function $\delta:\cX\to\R^{d}$ by 
\begin{equation}
\delta(x)=z(x)-\pi z,\label{eq:new-delta-v2}
\end{equation}
which is a centered version of $z$. Applying the operator $(\Padj-\Pi)$
to both sides of \eqref{eq:new-delta-v2} gives 
\begin{equation}
(\Padj-\Pi)z=(\Padj-\Pi)\delta,\label{eq:z-to-delta-relationship1}
\end{equation}
Subtracting $\Pi z$ from both sides of \eqref{eq:xi-relationship-step1}
and invoking \eqref{eq:z-to-delta-relationship1}, we obtain 
\begin{equation}
\delta=(\Padj-\Pi)z+\alpha\Padj(\dA z+b)=(\Padj-\Pi)\delta+\alpha\Padj(\dA z+b).\label{eq:delta-to-delta-relationship}
\end{equation}

On the other hand, applying $\pi$ to both sides of equation \eqref{eq:xi-relationship-step1}
gives 
\[
\pi z=\pi\left(\Padj\big(z+\alpha(\dA z+b)\big)\right)=\pi z+\alpha\pi(\dA z+b),
\]
which implies that 
\begin{equation}
\pi(\dA z+b)=0.\label{eq:pi-z-zero}
\end{equation}
Therefore, we can subtract the vanishing quantity $\pi(\dA z+b)$
from the right-hand side of \eqref{eq:delta-to-delta-relationship}
and obtain
\[
\delta=(\Padj-\Pi)\delta+\alpha(\Padj-\Pi)(\dA z+b).
\]
Consolidating the $\delta$ terms, we have 
\begin{equation}
(I-\Padj+\Pi)\delta=\alpha(\Padj-\Pi)(\dA z+b).\label{eq:delta-vec}
\end{equation}
To proceed, we observe that:

\begin{claim} \label{clm:i-p+pi-invertible} $(I-\Padj+\Pi)^{-1}$
exists as a bounded operator in $\ltwopi$. \end{claim}

When the state space $\cX$ is finite or countable, Claim~\ref{clm:i-p+pi-invertible}
is a well-known fact. For completeness, below we provide the proof
of Claim~\ref{clm:i-p+pi-invertible} for uniformly ergodic chains
on general state spaces. Taking Claim~\ref{clm:i-p+pi-invertible}
as given, we can rewrite \eqref{eq:delta-vec} as 
\begin{equation}
\delta=\alpha(I-\Padj+\Pi)^{-1}(\Padj-\Pi)(\dA z+b).\label{eq:delta-to-alpha-x}
\end{equation}
We then obtain the following bound on $\delta$:
\begin{align}
\|\delta\|_{L^{2}(\pi)} & =\alpha\left\Vert (I-\Padj+\Pi)^{-1}\Padj(\dA z+b)\right\Vert _{L^{2}(\pi)}\nonumber \\
 & \leq\alpha\left\Vert (I-\Padj+\Pi)^{-1}\right\Vert _{L^{2}(\pi)}\left\Vert \Padj\right\Vert _{L^{2}(\pi)}\left(\left\Vert \dA\right\Vert _{L^{2}(\pi)}\left\Vert z\right\Vert _{L^{2}(\pi)}+\left\Vert b\right\Vert _{L^{2}(\pi)}\right)\nonumber \\
 & \leq\alpha\left\Vert (I-\Padj+\Pi)^{-1}\right\Vert _{L^{2}(\pi)}\left(\Amax\left\Vert z\right\Vert _{L^{2}(\pi)}+\bmax\right)\nonumber \\
 & \leq\alpha\,C(A,b,\pi),\label{eq:delta_bound}
\end{align}
where in the step we use the bound
\[
\|z\|_{\ltwopi}^{2}=\int_{\cX}\pi(\dd x)\Big\|\E[\theta_{\infty}\mid x_{\infty}=x]\Big\|^{2}\leq\int_{\cX}\pi(\dd x)\E\Big[\|\theta_{\infty}\|^{2}\mid x_{\infty}=x\Big]=\E[\|\theta_{\infty}\|^{2}]\leq C'(A,b,\pi),
\]
with the last inequality following from Theorem~\ref{thm:thm-converge}.

\begin{proof}[Proof of Claim~\ref{clm:i-p+pi-invertible}] Since
$P$ is uniformly ergodic by Assumption~\ref{assumption:uniform-ergodic},
$P$ is $L^{2}(\pi)$-exponentially convergent \citep[Proposition 22.3.5]{Douc2018}.
That is, there exist $r\in[0,1)$ and $R<\infty$ such that $\|P^{n}-\Pi\|_{L^{2}(\pi)}\leq Rr^{n}$
for all $n\in\N$. Recall that $\Padj$ is the adjoint operator of
$P$, and note that $\Pi$ is self-adjoint. We therefore have 
\[
\dotp{f,(P^{n}-\Pi)g}_{L^{2}(\pi)}=\dotp{(P^{n}-\Pi)^{\ast}f,g}_{L^{2}(\pi)}=\dotp{((\Padj)^{n}-\Pi)f,g}_{L^{2}(\pi)},
\]
which implies that $\|(\Padj)^{n}-\Pi\|_{L^{2}(\pi)}=\|P^{n}-\Pi\|_{L^{2}(\pi)}.$
Combining pieces, we have the following bound on the Neumann series:
\begin{align*}
\Big\|\sum_{n=0}^{\infty}(\Padj-\Pi)^{n}\Big\|_{L^{2}(\pi)}\leq\sum_{n=0}^{\infty}\|(\Padj-\Pi)^{n}\|_{L^{2}(\pi)}=\sum_{n=0}^{\infty}\|(\Padj)^{n}-\Pi\|_{L^{2}(\pi)}=\sum_{n=0}^{\infty}\|P^{n}-\Pi\|_{L^{2}(\pi)}\leq\frac{R}{1-r}<\infty.
\end{align*}
Therefore, $I-\Padj+\Pi$ is invertible and its inverse is given by
the Neumann series. \end{proof}

\subsubsection{Step 4: Bootstrapping}
\label{sec:characterize-z}

We derive another relationship between $z$ and $\delta$, which can
be combined with \eqref{eq:delta-to-alpha-x} to give an equation
for~$\delta$.

Recall that $\theta^{\ast}$ is the unique solution to $\BarA\theta^{\ast}+\Bar b=0$.
Together with \eqref{eq:pi-z-zero}, using which we can write $\bar{b}=\pi b$
in terms of $A$ and $z$, we obtain the following relationship between
$\theta^{\ast}$ and $z$:
\begin{equation}
\theta^{\ast}=\pi\dAbar z,\label{eq:theta-ast-to-z-v2}
\end{equation}
where the operator $\dAbar$ is defined in Section \ref{sec:bias-expansion}.
Substituting $z(\cdot)=\delta(\cdot)+\pi z$ into \eqref{eq:theta-ast-to-z-v2},
we have 
\[
\theta^{\ast}=\int_{x\in\cX}\pi(\dd x)\BarA^{-1}A(x)(\delta(x)+\pi z)=\pi\dAbar\delta+\pi z.
\]
Reorganizing the equation above, we obtain 
\begin{equation}
\pi z=\theta^{\ast}-\pi\dAbar\delta.\label{eq:a-38-v2}
\end{equation}
It follows that 
\[
z(x)=\delta(x)+\pi z=\delta(x)+\Big(\theta^{\ast}-\pi\dAbar\delta\Big)=\theta^{\ast}+\Big(\delta(x)-\pi\dAbar\delta\Big),\qquad\forall x\in\cX,
\]
which can be written compactly as 
\begin{equation}
z=\theta^{\ast}+(I-\Pi\dAbar)\delta.\label{eq:a-40-v2}
\end{equation}
Substituting \eqref{eq:a-40-v2} into the RHS of \eqref{eq:delta-to-alpha-x},
and recalling the definitions of $\Upsilon$ and $\Xi$ in (\ref{eq:bias-coef-def}),
we obtain 
\begin{align}
\delta & =\alpha(I-\Padj+\Pi)^{-1}(\Padj-\Pi)(\dA z+b)\nonumber \\
 & =\alpha(I-\Padj+\Pi)^{-1}(\Padj-\Pi)(A\theta^{\ast}+\dA(I-\Pi\dAbar)\delta+b)\nonumber \\
 & =\alpha(I-\Padj+\Pi)^{-1}(\Padj-\Pi)(A\theta^{\ast}+b)+\alpha(I-\Padj+\Pi)^{-1}(\Padj-\Pi)\dA(I-\Pi\dAbar)\delta,\nonumber \\
 & =\alpha\Upsilon+\alpha\Xi\delta.\label{eq:Delta_self_expressing}
\end{align}
Equation (\ref{eq:Delta_self_expressing}), which expresses $\delta$
in terms of itself, plays a crucial role in the sequel.

Substituting~\eqref{eq:Delta_self_expressing} into the RHS of equation~\eqref{eq:a-38-v2},
we obtain
\[
\E[\theta_{\infty}]=\theta^{\ast}-\pi\dAbar\delta=\theta^{\ast}-\alpha\pi\dAbar\Upsilon-\alpha\pi\dAbar\Xi\delta.
\]
Using the bound (\ref{eq:delta_bound}), we obtain that 
\[
\|\pi\dAbar\Xi\delta\|\leq\|\dAbar\Xi\|\|\pi\delta\|\leq\|\dAbar\Xi\|\|\delta\|_{\ltwopi}=\bigO(\alpha).
\]
Combining the last two equations, and recalling the definition of
$B^{(1)}$ in (\ref{eq:Bi-def}), we prove the base case in Theorem~\ref{thm:bias-characterization},
i.e., $\E[\theta_{\infty}]=\theta^{\ast}+\alpha B^{(1)}+\bigO(\alpha^{2}).$ 

We now bootstrap the above argument. Plugging \eqref{eq:Delta_self_expressing}
back to the RHS of itself, we obtain the following equation for $\delta$:
\[
\delta=\sum_{i=1}^{m}\alpha^{i}\Xi^{i-1}\Upsilon+\alpha^{m}\Xi^{m}\delta.
\]
Plugging the above equation into the RHS of equation~\eqref{eq:a-38-v2},
and using again the bound $\|\pi\delta\|\leq\|\delta\|_{\ltwopi}=\bigO(\alpha)$,
we obtain the $m$-th order bias expansion:
\begin{align*}
\E[\theta_{\infty}]=\theta^{\ast}-\pi\dAbar\delta= & \theta^{\ast}-\pi\dAbar\left(\sum_{i=1}^{m}\alpha^{i}\Xi^{i-1}\Upsilon+\alpha^{m}\Xi^{m}\delta\right)\\
= & \theta^{\ast}+\pi\dAbar\sum_{i=1}^{m}\alpha^{i}\Xi^{i-1}\Upsilon+\bigO(\alpha^{m+1}).
\end{align*}
We have proven the first part of Theorem~\ref{thm:bias-characterization}.

To prove the infinite series expansion in the second part, we use
the assumption $\alpha<1/\|\Xi\|_{\ltwopi}$ to establish that the following
Neumann series converges:
\[
\lim_{m\to\infty}\alpha\cdot\Big\Vert\sum_{i=0}^{m}(\alpha\Xi)^{i}\Upsilon\Big\Vert\leq\lim_{m\to\infty}\alpha\cdot\Big(\sum_{i=0}^{m}\|\alpha\Xi\|^{i}\|\Upsilon\|\Big)<\infty.
\]
Therefore, the inverse operator $(I-\alpha\Xi)^{-1}$ is well defined
and given by the above Neumann series. We can then solve equation~\eqref{eq:Delta_self_expressing} for $\delta$ to obtain 
\begin{align*}
\delta & =\alpha(I-\alpha\Xi)^{-1}\Upsilon=\alpha\left(\sum_{i=0}^{\infty}(\alpha\Xi)^{i}\right)\Upsilon.
\end{align*}
Finally, we substitute the above expansion for $\delta$ into the
RHS of equation~\eqref{eq:a-38-v2}, which gives the desired infinite
expansion for $\E[\theta_{\infty}]$:
\begin{align*}
\E[\theta_{\infty}]=\theta^{\ast}-\pi\dAbar\delta= & \theta^{\ast}-\alpha\pi\dAbar\Big(I-\alpha\Xi\Big)^{-1}\Upsilon\\
= & \theta^{\ast}-\alpha\Big(\pi\dAbar\sum_{i=0}^{\infty}(\alpha\Xi)^{i}\Upsilon\Big)\\
= & \theta^{\ast}+\sum_{i=1}^{\infty}\alpha^{i}B^{(i)},
\end{align*}
where the last step follows from the definition of $B^{(i)}$ in (\ref{eq:Bi-def}).
This completes the proof of Theorem~\ref{thm:bias-characterization}.

Before concluding the section, we derive an explicit upper bound on
$\|\Xi\|_{L^{2}(\pi)}$. Recall that the uniform ergodicity of $P$
implies the bound $\|P^{n}-\Pi\|_{L^{2}(\pi)}\leq Rr^{n},\forall n\in\N$
for some $r<1$. It follows that 
\begin{align*}
\|\Xi\|_{L^{2}(\pi)}= & \left\Vert (I-\Padj+\Pi)^{-1}(\Padj-\Pi)\dA(I-\Pi\dAbar)\right\Vert _{L^{2}(\pi)}\\
\leq & \left\Vert (I-\Padj+\Pi)^{-1}\right\Vert _{L^{2}(\pi)}\left\Vert \Padj-\Pi\right\Vert _{L^{2}(\pi)}\left\Vert \dA\right\Vert _{L^{2}(\pi)}\left\Vert I-\Pi\dAbar\right\Vert _{L^{2}(\pi)}\\
\leq & \frac{Rr}{1-r}\Amax(1+s_{\min}(\BarA)^{-1}\Amax).
\end{align*}
Therefore, a sufficient condition for $\alpha<1/\|\Xi\|_{\ltwopi}$
is 
\begin{equation}
\alpha<\left(\frac{Rr}{1-r}\Amax(1+s_{\min}(\BarA)^{-1}\Amax)\right)^{-1}.\label{eq:alpha_cond_inf_series}
\end{equation}

\subsection{Proof of Theorem~\ref{thm:bias-characterization-reversible}}

\label{sec:proof-reversible}

In this section, we prove Theorem~\ref{thm:bias-characterization-reversible}
on the relationship between the bias and the absolute spectral gap
of the chain $(x_{k})_{k\ge0}$.

Under the assumption that $P$ is uniformly ergodic and reversible,
i.e., $\Padj=P$, we have the following relationship \citep[Proposition 22.2.5]{Douc2018}:
\begin{equation}
1-\|P-\Pi\|_{L^{2}(\pi)}=\absgap(P)>0.\label{eq:reversible-absgap}
\end{equation}
Therefore, the operator norm of $(I-P+\Pi)^{-1}$ satisfies, 
\begin{equation}
\|(I-P+\Pi)^{-1}\|_{L^{2}(\pi)}=\left\Vert \sum_{i=0}^{\infty}(P-\Pi)^{i}\right\Vert \leq\sum_{i=0}^{\infty}\|P-\Pi\|_{L^{2}(\pi)}^{i}=\sum_{i=0}^{\infty}\Big(1-\absgap(P)\Big)^{i}=\frac{1}{\absgap(P)}.\label{eq:neuman-bound}
\end{equation}
Recall the definitions of $\Xi,\Upsilon$ and $B^{(i)}$ in (\ref{eq:bias-coef-def}).
Using the submultiplicativity of the norm $\left\Vert \cdot\right\Vert _{\ltwopi}$,
we obtain 
\begin{align*}
\left\Vert \Xi\right\Vert _{\ltwopi} & =\left\Vert (I-P+\Pi)^{-1}(P-\Pi)\dA(I-\Pi\dAbar)\right\Vert _{\ltwopi}\\
 & \le\left\Vert (I-P+\Pi)^{-1}\right\Vert _{\ltwopi}\left\Vert P-\Pi\right\Vert _{\ltwopi}\left\Vert \dA(I-\Pi\dAbar)\right\Vert _{\ltwopi}\\
 & \le\frac{1}{\absgap(P)}\cdot\left(1-\absgap(P)\right)\cdot\left\Vert \dA(I-\Pi\dAbar)\right\Vert _{\ltwopi}
\end{align*}
and
\begin{align*}
\left\Vert \Upsilon\right\Vert _{\ltwopi} & =\left\Vert (I-P+\Pi)^{-1}(P-\Pi)(A\theta^{\ast}+b)\right\Vert _{\ltwopi}\\
 & \le\|(I-P+\Pi)^{-1}\|_{L^{2}(\pi)}\|P-\Pi\|_{L^{2}(\pi)}\|A\theta^{\ast}+b\|_{L^{2}(\pi)}\\
 & \le\frac{1}{\absgap(P)}\cdot\left(1-\absgap(P)\right)\cdot\left\Vert A\theta^{\ast}+b\right\Vert _{\ltwopi}.
\end{align*}
It follows that for each $i=1,2,\ldots,$
\begin{align*}
\left\Vert B^{(i)}\right\Vert =\left\Vert \pi\dAbar\Xi^{i-1}\Upsilon\right\Vert  & \overset{\text{(i)}}{\le}\left\Vert D\Xi^{i-1}\Upsilon\right\Vert _{\ltwopi}\\
 & \le\left\Vert D\right\Vert _{\ltwopi}\left\Vert \Xi\right\Vert _{\ltwopi}^{i-1}\left\Vert \Upsilon\right\Vert _{\ltwopi}\\
 & \le\left\Vert D\right\Vert _{\ltwopi}\left(\frac{1-\absgap(P)}{\absgap(P)}\right)^{i}\left\Vert \dA(I-\Pi\dAbar)\right\Vert _{\ltwopi}^{i-1}\left\Vert A\theta^{\ast}+b\right\Vert _{\ltwopi},
\end{align*}
where step (i) follows from Jensen's inequality. Setting
\[
C(A,b,\pi)=\max\left\{ \left\Vert D\right\Vert _{\ltwopi}\left\Vert A\theta^{\ast}+b\right\Vert _{\ltwopi},\left\Vert \dA(I-\Pi\dAbar)\right\Vert _{\ltwopi}\right\} ,
\]
which only depends on $A,b$ and $\pi$, we obtain $\left\Vert B^{(i)}\right\Vert \le C(A,b,\pi)^{i}\left(\frac{1-\absgap(P)}{\absgap(P)}\right)^{i}$
as claimed.

\subsection{Proof of Corollary~\ref{cor:pr-avg-bounds}}

\label{sec:pr-avg-proof}

We prove the first and second moment bounds in Corollary~\ref{cor:pr-avg-bounds}.

\subsubsection{First Moment}

\label{sec:pr-avg-first} We first have 
\[
\E[\bar{\theta}_{k_{0},k}]-\theta^{\ast}=\left(\E[\theta_{\infty}]-\theta^{\ast}\right)+\frac{1}{k-k_{0}}\underbrace{\sum_{t=k_{0}}^{k-1}\E[\theta_{t}-\theta_{\infty}]}_{T_{1}}.
\]
To bound $T_{1}$, we recall the inequality proven in\eqref{eq:first-moment-geometric},
that for $k\geq\tau$, 
\[
\|\E[\theta_{k}]-\E[\theta_{\infty}]\|\leq C(A,b,\pi)\cdot\left(1-\frac{0.9\alpha}{\gammax}\right)^{k/2}.
\]
As the burn-in period satisfies $k_{0}\geq\tau$, we have the following
bound, 
\begin{align}
\|T_{1}\|=\bigg\Vert\sum_{t=k_{0}}^{k-1}\E[\theta_{t}-\theta_{\infty}]\bigg\Vert & \leq\sum_{t=k_{0}}^{k-1}\left\Vert \E[\theta_{t}]-\E[\theta_{\infty}]\right\Vert \nonumber \\
 & \leq C(A,b,\pi)\left(1-\frac{0.9\alpha}{\gammax}\right)^{k_{0}/2}\frac{\gammax}{0.9\alpha}\nonumber \\
 & \leq C'(A,b,\pi)\cdot\frac{1}{\alpha}\cdot\exp\left(-\frac{\alpha k_{0}}{4\gammax}\right).\label{eq:t1-sum-bound}
\end{align}
Together with \eqref{eq:bias-char}, we obtain the desired equation~\eqref{eq:pr-avg-first-mom},
that is, 
\[
\E[\bar{\theta}_{k_{0},k}]-\theta^{\ast}=\alpha B(A,b,P)+\bigO(\alpha^{2})+\bigO\left(\frac{1}{\alpha(k-k_{0})}\exp\left(-\frac{\alpha k_{0}}{4\gammax}\right)\right).
\]

\subsubsection{Second Moment}

To bound the second moment of the tail-averaged iterate, we make use
of the following decomposition: 
\begin{align}
 & \E\left[\left(\bar{\theta}_{k_{0},k}-\theta^{\ast}\right)\left(\bar{\theta}_{k_{0},k}-\theta^{\ast}\right)^{\top}\right]\nonumber \\
= & \E\left[\left(\bar{\theta}_{k_{0},k}-\E[\theta_{\infty}]+\E[\theta_{\infty}]-\theta^{\ast}\right)\left(\bar{\theta}_{k_{0},k}-\E[\theta_{\infty}]+\E[\theta_{\infty}]-\theta^{\ast}\right)^{\top}\right]\nonumber \\
= & \underbrace{\E\left[\left(\bar{\theta}_{k_{0},k}-\E[\theta_{\infty}]\right)\left(\bar{\theta}_{k_{0},k}-\E[\theta_{\infty}]\right)^{\top}\right]}_{T_{1}}+\underbrace{\E\left[\left(\bar{\theta}_{k_{0},k}-\E[\theta_{\infty}]\right)\left(\E[\theta_{\infty}]-\theta^{\ast}\right)^{\top}\right]}_{T_{2}}\label{eq:2nd-mom-decomp1}\\
+ & \underbrace{\E\left[\left(\E[\theta_{\infty}]-\theta^{\ast}\right)\left(\bar{\theta}_{k_{0},k}-\E[\theta_{\infty}]\right)^{\top}\right]}_{T_{3}}+\underbrace{\E\left[\left(\E[\theta_{\infty}]-\theta^{\ast}\right)\left(\E[\theta_{\infty}]-\theta^{\ast}\right)^{\top}\right]}_{T_{4}},\label{eq:2nd-mom-decomp2}
\end{align}
and we analyze the four terms $T_{1}$--$T_{4}$ respectively.

We start with $T_{2}$, which satisfies the bound
\begin{align*}
\E\left[\left(\bar{\theta}_{k_{0},k}-\E[\theta_{\infty}]\right)\left(\E[\theta_{\infty}]-\theta^{\ast}\right)^{\top}\right] & =\frac{1}{k-k_{0}}\left(\sum_{t=k_{0}}^{k-1}\E\left[\theta_{t}-\theta_{\infty}\right]\right)\left(\E[\theta_{\infty}]-\theta^{\ast}\right)^{\top}\\
 & \overset{\textnormal{(iv)}}{=}\bigO\left(\frac{1}{\alpha(k-k_{0})}\exp\left(-\frac{\alpha k_{0}}{4\gammax}\right)\right)\left(\alpha B(A,b,P)+\bigO(\alpha^{2})\right)\\
 & =\bigO\left(\frac{1}{k-k_{0}}\exp\left(-\frac{\alpha k_{0}}{4\gammax}\right)\right),
\end{align*}
where step (iv) is due to equations~\eqref{eq:t1-sum-bound} and
\eqref{eq:bias-char}. The term $T_{3}$ can be analyzed in the same
fashion.

For $T_{4}$, we observe that 
\begin{align*}
\E\left[\left(\E[\theta_{\infty}]-\theta^{\ast}\right)\left(\E[\theta_{\infty}]-\theta^{\ast}\right)^{\top}\right] & =\left(\E[\theta_{\infty}]-\theta^{\ast}\right)\left(\E[\theta_{\infty}]-\theta^{\ast}\right)^{\top}\\
 & \overset{\textnormal{(v)}}{=}(\alpha B(A,b,P)+\bigO(\alpha^{2}))(\alpha B(A,b,P)+\bigO(\alpha^{2}))^{\top}\\
 & =\alpha^{2}B'(A,b,P)+\bigO(\alpha^{3}),
\end{align*}
where step~(v) holds by equation~\eqref{eq:bias-char}.

It remains to bound $T_{1}$. We have 
\begin{align*}
T_{1}= & \frac{1}{(k-k_{0})^{2}}\E\left[\left(\sum_{t=k_{0}}^{k-1}(\theta_{t}-\E[\theta_{\infty}])\right)\left(\sum_{t=k_{0}}^{k-1}(\theta_{t}-\E[\theta_{\infty}])\right)^{\top}\right]\\
= & \underbrace{\frac{1}{(k-k_{0})^{2}}\sum_{t=k_{0}}^{k-1}\E\left[\left(\theta_{t}-\E[\theta_{\infty}]\right)\left(\theta_{t}-\E[\theta_{\infty}]\right)^{\top}\right]}_{T'_{1}}\\
 & +\underbrace{\frac{1}{(k-k_{0})^{2}}\sum_{t=k_{0}}^{k-1}\sum_{l=t+1}^{k-1}\bigg(\E\left[\left(\theta_{t}-\E[\theta_{\infty}]\right)\left(\theta_{l}-\E[\theta_{\infty}]\right)^{\top}\right]+\E\left[\left(\theta_{l}-\E[\theta_{\infty}]\right)\left(\theta_{t}-\E[\theta_{\infty}]\right)^{\top}\right]\bigg)}_{T'_{2}}.
\end{align*}

Below we control $T'_{1}$ and $T'_{2}$ respectively. For $T'_{1}$,
we start with the following decomposition, 
\begin{align}
 & \E\Big[\left(\theta_{t}-\E[\theta_{\infty}]\right)\left(\theta_{t}-\E[\theta_{\infty}]\right)^{\top}\Big]\nonumber \\
= & \left(\E[\theta_{t}\theta_{t}^{\top}]-\E[\theta_{\infty}\theta_{\infty}^{\top}]\right)+\left(\E[\theta_{\infty}\theta_{\infty}^{\top}]-\E[\theta_{\infty}]\E[\theta_{\infty}^{\top}]\right)-\left(\E[\theta_{t}]\E[\theta_{\infty}^{\top}]+\E[\theta_{\infty}]\E[\theta_{t}^{\top}]-2\E[\theta_{\infty}]\E[\theta_{\infty}^{\top}]\right)\nonumber \\
= & \left(\E[\theta_{t}\theta_{t}^{\top}]-\E[\theta_{\infty}\theta_{\infty}^{\top}]\right)+\var(\theta_{\infty})-\E[\theta_{t}-\theta_{\infty}]\E[\theta_{\infty}^{\top}]-\E[\theta_{\infty}]\E[(\theta_{t}-\theta_{\infty})^{\top}].\label{eq:thetat-sec-mom-decomp}
\end{align}
By Corollary~\ref{cor:non-asymptotic-bounds} and Lemma~\ref{lem:theta_inf_bound},
the following bounds hold when $t\geq\tau$: 
\begin{align}
\E[\|\theta_{t}-\theta_{\infty}\|] & \leq C(A,b,\pi)\cdot\left(1-\frac{0.9\alpha}{\gammax}\right)^{t/2},\label{eq:1st-moment-bound}\\
\left\Vert \E\left[\theta_{t}\theta_{t}^{\top}\right]-\E\left[\theta_{\infty}\theta_{\infty}^{\top}\right]\right\Vert  & \leq C'(A,b,\pi)\cdot\left(1-\frac{0.9\alpha}{\gammax}\right)^{t/2},\nonumber \\
\E[\|\theta_{\infty}\|] & \leq C''(A,b,\pi),\nonumber \\
\var(\theta_{\infty}) & \le C'''(A,b,\pi)\cdot\alpha\tau.\label{eq:var-bound}
\end{align}
Plugging these bounds into equation~\eqref{eq:thetat-sec-mom-decomp},
we obtain that 
\begin{equation}
\E\Big[\left(\theta_{t}-\E[\theta_{\infty}]\right)\left(\theta_{t}-\E[\theta_{\infty}]\right)^{\top}\Big]=\bigO\bigg(\left(1-\frac{0.9\alpha}{\gammax}\right)^{t/2}+\alpha\tau\bigg).\label{eq:2nd-moment-bound1}
\end{equation}
Now with~\eqref{eq:2nd-moment-bound1} on hand, we obtain the following
bound for $T'_{1}$:
\begin{align*}
T'_{1} & =\frac{1}{(k-k_{0})^{2}}\sum_{t=k_{0}}^{k-1}\bigO\bigg(\left(1-\frac{0.9\alpha}{\gammax}\right)^{t/2}+\alpha\tau\bigg)\\
 & =\bigO\left(\frac{1}{(k-k_{0})^{2}}\sum_{t=k_{0}}^{\infty}\left(1-\frac{0.9\alpha}{\gammax}\right)^{t/2}\right)+\bigO\left(\frac{\alpha\tau}{k-k_{0}}\right)\\
 & =\bigO\left(\frac{1}{(k-k_{0})^{2}}\cdot\frac{2\gammax}{0.9\alpha}\cdot\left(1-\frac{0.9\alpha}{\gammax}\right)^{k_{0}/2}\right)+\bigO\left(\frac{\alpha\tau}{k-k_{0}}\right)\\
 & =\bigO\left(\frac{1}{\alpha(k-k_{0})^{2}}\exp\left(-\frac{\alpha k_{0}}{4\gammax}\right)+\frac{\alpha\tau}{k-k_{0}}\right).
\end{align*}

To analyze $T'_{2}$, we first study each term in the summation. Observe
that for $l>t$, we have 
\begin{align*}
\E\left[\left(\theta_{t}-\E[\theta_{\infty}]\right)\left(\theta_{l}-\E[\theta_{\infty}]\right)^{\top}\right] & =\E\left[\E\big[\left(\theta_{t}-\E[\theta_{\infty}]\right)\left(\theta_{l}-\E[\theta_{\infty}]\right)^{\top}\mid\theta_{t}\big]\right]\\
 & =\E\left[\left(\theta_{t}-\E[\theta_{\infty}]\right)\left(\E[\theta_{l}\mid\theta_{t}\big]-\E[\theta_{\infty}]\right)^{\top}\right].
\end{align*}
We now make the following claim, whose proof we delay to the end of
this sub-sub-section.

\begin{claim} \label{clm:cross-term-nice} For $t\geq\frac{4\gammax}{\alpha}\log\Big(\frac{1}{\alpha\tau}\Big)$
and $l\geq t+\tau$, we have 
\begin{equation}
\left\Vert \E\left[\left(\theta_{t}-\E[\theta_{\infty}]\right)\left(\theta_{l}-\E[\theta_{\infty}]\right)^{\top}\right]\right\Vert =\bigO\Big(\alpha\tau\cdot\Big(1-\frac{0.9\alpha}{\gammax}\Big)^{(l-t)/2}\Big)
\end{equation}
\end{claim}

Taking this claim as given, we now complete the analysis of $T'_{2}$,
which has the form $T'_{2}=T''_{2}+(T''_{2})^{\top}$. We observe
that
\begin{align*}
T''_{2}= & \frac{1}{(k-k_{0})^{2}}\Bigg[\sum_{t=k_{0}}^{k-\tau-1}\Big(\sum_{l=t+1}^{t+\tau}\E\left[\left(\theta_{t}-\E[\theta_{\infty}]\right)\left(\theta_{l}-\E[\theta_{\infty}]\right)^{\top}\right]+\sum_{l=t+\tau+1}^{k-1}\E\left[\left(\theta_{t}-\E[\theta_{\infty}]\right)\left(\theta_{l}-\E[\theta_{\infty}]\right)^{\top}\right]\Big)\\
 & \qquad\qquad\quad+\sum_{t=k-\tau}^{k-1}\sum_{l=t+1}^{k-1}\E\left[\left(\theta_{t}-\E[\theta_{\infty}]\right)\left(\theta_{l}-\E[\theta_{\infty}]\right)^{\top}\right]\Bigg]\\
= & \frac{1}{(k-k_{0})^{2}}\left(\sum_{t=k_{0}}^{k-\tau-1}\sum_{l=t+1}^{t+\tau}\E\left[\left(\theta_{t}-\E[\theta_{\infty}]\right)\left(\theta_{l}-\E[\theta_{\infty}]\right)^{\top}\right]+\sum_{t=k-\tau}^{k-1}\sum_{l=t+1}^{k-1}\E\left[\left(\theta_{t}-\E[\theta_{\infty}]\right)\left(\theta_{l}-\E[\theta_{\infty}]\right)^{\top}\right]\right)\\
 & +\frac{1}{(k-k_{0})^{2}}\sum_{t=k_{0}}^{k-\tau-1}\sum_{l=t+\tau+1}^{k-1}\E\left[\left(\theta_{t}-\E[\theta_{\infty}]\right)\left(\theta_{l}-\E[\theta_{\infty}]\right)^{\top}\right]\\
=: & T_{a}+T_{b}.
\end{align*}
Above, we group the summand based on the $\tau$ burn-in requirements.
For the $T_{b}$ terms, which satisfies the burn-in requirements,
we use Claim~\ref{clm:cross-term-nice} to obtain 
\begin{align*}
T_{b}= & \frac{1}{(k-k_{0})^{2}}\sum_{t=k_{0}}^{k-\tau-1}\sum_{l=t+\tau+1}^{k-1}\bigO\left(\alpha\tau\cdot\left(1-\frac{0.9\alpha}{\gammax}\right)^{(l-t)/2}\right)\\
= & \frac{\alpha\tau}{(k-k_{0})^{2}}\sum_{t=k_{0}}^{k-\tau-1}\sum_{l=\tau+1}^{k-1-t}\bigO\left(\left(1-\frac{0.9\alpha}{\gammax}\right)^{l/2}\right)\\
\le & \frac{\alpha\tau}{k-k_{0}}\sum_{l=0}^{\infty}\bigO\left(\left(1-\frac{0.9\alpha}{\gammax}\right)^{l/2}\right)=\bigO\Big(\frac{\tau}{k-k_{0}}\Big).
\end{align*}
The $T_{a}$ term involves a finite number of cross terms. We bound
them using the Cauchy-Schwarz inequality, 
\[
\E\left[\left\Vert \left(\theta_{t}-\E[\theta_{\infty}]\right)\left(\theta_{l}-\E[\theta_{\infty}]\right)^{\top}\right\Vert \right]\leq\sqrt{\E[\|\theta_{t}-\E[\theta_{\infty}]\|^{2}]}\sqrt{\E[\|\theta_{l}-\E[\theta_{\infty}]\|^{2}]}.
\]
As $l\geq t\geq k_{0}$ and $k_{0}\geq\frac{4\gammax}{\alpha}\log(\frac{1}{\alpha\tau})$,
we can bound the above RHS using (\ref{eq:2nd-moment-bound1}) and
obtain
\[
\E\left[\left\Vert \left(\theta_{t}-\E[\theta_{\infty}]\right)\left(\theta_{l}-\E[\theta_{\infty}]\right)^{\top}\right\Vert \right]=\bigO(\alpha\tau).
\]
Hence, we have 
\[
T_{a}\leq\frac{1}{k-k_{0}}\bigO(\tau\cdot\alpha\tau)+\frac{\tau^{2}}{(k-k_{0})^{2}}\bigO(\alpha\tau)=\bigO\Big(\frac{\tau}{k-k_{0}}\Big),
\]
where in the last step we make use of $\alpha\tau\leq\bigO(1)$ and
$\tau\leq k-k_{0}$.

Therefore, by combining the above pieces of information, we obtain the
following bound for $T'_{2}$:
\[
T'_{2}=T''_{2}+(T''_{2})^{\top}=(T_{a}+T_{b})+(T_{a}+T_{b})^{\top}=\bigO\Big(\frac{\tau}{k-k_{0}}\Big).
\]

With the bounds for $T'_{1}$ and $T'_{2}$, we can bound the term
$T_{1}$ in \eqref{eq:2nd-mom-decomp1} for $k_{0}\geq\frac{4\gammax}{\alpha}\log\Big(\frac{1}{\alpha\tau}\Big)$
and $k\geq k_{0}+\tau$:
\begin{align}
T_{1}=T'_{1}+T'_{2}= & \bigO\left(\frac{1}{\alpha(k-k_{0})^{2}}\exp\left(-\frac{\alpha k_{0}}{4\gammax}\right)+\frac{\alpha\tau}{k-k_{0}}\right)+\bigO\left(\frac{\tau_{\alpha}}{k-k_{0}}\right)\nonumber \\
= & \bigO\left(\frac{\tau_{\alpha}}{k-k_{0}}+\frac{1}{\alpha(k-k_{0})^{2}}\exp\left(-\frac{\alpha k_{0}}{4\gammax}\right)\right),\label{eq:theta-avg-to-inf-mse}
\end{align}
where in the second step we use $\alpha\leq\alpha\tau_{\alpha}<1$
from equations~\eqref{eq:alpha-constraint} and~\eqref{eq:gammax-lower-bound}.

Combining all the pieces, we obtain 
\begin{align*}
\E\left[\left(\bar{\theta}_{k_{0},k}-\theta^{\ast}\right)\left(\bar{\theta}_{k_{0},k}-\theta^{\ast}\right)^{\top}\right]= & T_{1}+T_{2}+T_{3}+T_{4}\\
= & \alpha^{2}B'(A,b,P)+\bigO(\alpha^{3})+\bigO\left(\frac{1}{k-k_{0}}\exp\left(-\frac{\alpha k_{0}}{4\gammax}\right)\right)\\
 & +\bigO\left(\frac{1}{\alpha(k-k_{0})^{2}}\exp\left(-\frac{\alpha k_{0}}{4\gammax}\right)+\frac{\tau}{k-k_{0}}\right),
\end{align*}
which is the desired equation~\eqref{eq:mse-pr-bound}. We have completed
the proof of Corollary~\ref{cor:pr-avg-bounds}.

\begin{proof}[Proof of Claim~\ref{clm:cross-term-nice}] First
note that when $t\geq\frac{4\gammax}{\alpha}\log\Big(\frac{1}{\alpha\tau}\Big)$,
which is the assumption in Corollary~\ref{cor:pr-avg-bounds}, we
have $(1-\frac{0.9\alpha}{\gammax})^{t/2}\lesssim\alpha\tau$. Therefore,
\eqref{eq:2nd-moment-bound1} simplifies to 
\[
\E\bigg[\left(\theta_{t}-\E[\theta_{\infty}]\right)\left(\theta_{t}-\E[\theta_{\infty}]\right)^{\top}\bigg]=\bigO\left(\alpha\tau\right).
\]
With the above bound, we can use the convergence rate~\eqref{eq:w2-thetak-to-mu}
in Theorem~\ref{thm:thm-converge} to derive the following: for $t\geq\frac{4\gammax}{\alpha}\log\Big(\frac{1}{\alpha\tau}\Big)$
and $k\geq\tau$, 
\[
W_{2}^{2}\Big(\law(\theta_{k+t}),\mu\Big)\leq20\,\frac{\gammax}{\gammin}\Big(\E[\|\theta_{t}-\E[\theta_{\infty}]\|^{2}]+\tr(\var(\theta_{\infty}))\Big)\cdot\left(1-\frac{0.9\alpha}{\gammax}\right)^{k}\leq\bigO\Big(\alpha\tau\cdot\Big(1-\frac{0.9\alpha}{\gammax}\Big)^{k}\Big).
\]
Taking the optimal coupling that achieves the Wasserstein distance
between $(\theta_{k+t},x_{k+t})$ and $(\theta_{\infty},x_{\infty})$,
we obtain 
\[
\Big\|\E[\theta_{k+t}]-\E[\theta_{\infty}]\Big\|^{2}\leq\E\left[\left\Vert \theta_{k+t}-\theta_{\infty}\right\Vert ^{2}\right]\leq\bigO(\alpha\tau\cdot\left(1-\frac{0.9\alpha}{\gammax}\right)^{k}\Big).
\]
Consequently, for $l\geq t+\tau$ and $t\geq\frac{4\gammax}{\alpha}\log\Big(\frac{1}{\alpha\tau}\Big)$,
we obtain 
\begin{align*}
 & \left\Vert \E\left[\left(\theta_{t}-\E[\theta_{\infty}]\right)\left(\theta_{l}-\E[\theta_{\infty}]\right)^{\top}\right]\right\Vert \\
\le & \E\Big[\|\theta_{t}-\E[\theta_{\infty}]\|\cdot\|\E[\theta_{l}|\theta_{t}]-\E[\theta_{\infty}]\|\Big]\\
\leq & \E\left[\|\theta_{t}-\E[\theta_{\infty}]\|\cdot\sqrt{20\,\frac{\gammax}{\gammin}\Big(\|\theta_{t}-\E[\theta_{\infty}]\|^{2}+\tr(\var(\theta_{\infty}))]\Big)\cdot\left(1-\frac{0.9\alpha}{\gammax}\right)^{l-t}}\,\,\right]\\
\leq & C(A,b,\pi)\cdot\E\left[\|\theta_{t}-\E[\theta_{\infty}]\|\cdot\Big(\|\theta_{t}-\E[\theta_{\infty}]\|+\sqrt{\tr(\var(\theta_{\infty}))}]\Big)\cdot\left(1-\frac{0.9\alpha}{\gammax}\right)^{(l-t)/2}\,\right]\\
\leq & C(A,b,\pi)\cdot\left(\E\left[\|\theta_{t}-\E[\theta_{\infty}]\|^{2}\right]\cdot\left(1-\frac{0.9\alpha}{\gammax}\right)^{(l-t)/2}+\sqrt{\alpha\tau}\cdot\E\left[\|\theta_{t}-\E[\theta_{\infty}]\|\,\right]\cdot\left(1-\frac{0.9\alpha}{\gammax}\right)^{(l-t)/2}\right)\\
= & \bigO\Big(\alpha\tau\cdot\Big(1-\frac{0.9\alpha}{\gammax}\Big)^{(l-t)/2}\Big).
\end{align*}
As such, we have proven the desired claim. \end{proof}

\subsection{Proof of Corollary~\ref{cor:rr-ext-bounds}}

\label{sec:rr-ext-proof}

We prove the first and second moment bounds in Corollary~\ref{cor:rr-ext-bounds}.

\subsubsection{First Moment}

We have 
\begin{align*}
\E[\widetilde{\theta}_{k_{0},k}^{(\alpha)}]-\theta^{\ast} & =\left(2\bar{\theta}_{k_{0},k}^{(\alpha)}-\bar{\theta}_{k_{0},k}^{(2\alpha)}\right)-\theta^{\ast}=2\left(\bar{\theta}_{k_{0},k}^{(\alpha)}-\theta^{\ast}\right)-\left(\bar{\theta}_{k_{0},k}^{(2\alpha)}-\theta^{\ast}\right)\\
 & \overset{\textnormal{(i)}}{=}2\left(\alpha B(A,b,P)+\bigO(\alpha^{2})+\bigO\left(\frac{1}{\alpha(k-k_{0})}\exp\left(-\frac{\alpha k_{0}}{4\gammax}\right)\right)\right)\\
 & \quad-\left(2\alpha B(A,b,P)+\bigO(\alpha^{2})+\bigO\left(\frac{1}{\alpha(k-k_{0})}\exp\left(-\frac{\alpha k_{0}}{2\gammax}\right)\right)\right)\\
 & =\bigO(\alpha^{2})+\bigO\left(\frac{1}{\alpha(k-k_{0})}\exp\left(-\frac{\alpha k_{0}}{4\gammax}\right)\right),
\end{align*}
where (i) holds following from equation~\eqref{eq:pr-avg-first-mom}.

\subsubsection{Second Moment}

We first introduce the following shorthands: 
\[
u_{1}:=\bar{\theta}_{k_{0},k}^{(\alpha)}-\E\left[\theta_{\infty}^{(\alpha)}\right],\quad u_{2}:=\bar{\theta}_{k_{0},k}^{(2\alpha)}-\E\left[\theta_{\infty}^{(2\alpha)}\right]\quad\text{and}\quad v:=2\E\left[\theta_{\infty}^{(\alpha)}\right]-\E\left[\theta_{\infty}^{(2\alpha)}\right]+\theta^{*}.
\]
With these notations, we write $\widetilde{\theta}_{k_{0},k}-\theta^{*}=2u_{1}-u_{2}+v$
and observe the bound 
\begin{align*}
\left\Vert \E\left[\left(\tilde{\theta}_{k_{0},k}-\theta^{*}\right)\left(\tilde{\theta}_{k_{0},k}-\theta^{*}\right)^{\top}\right]\right\Vert  & =\left\Vert \E\left[\left(2u_{1}-u_{2}+v\right)\left(2u_{1}-u_{2}+v\right)^{\top}\right]\right\Vert \\
 & \le\E\left\Vert 2u_{1}\right\Vert ^{2}+3\E\left\Vert u_{2}\right\Vert ^{2}+3\left\Vert v\right\Vert ^{2}.
\end{align*}

By equation~\eqref{eq:theta-avg-to-inf-mse} we have 
\[
\E\left\Vert u_{1}\right\Vert ^{2}=\tr\Big(\E\left[u_{1}u_{1}^{\top}\right]\Big)=\bigO\left(\frac{\tau_{\alpha}}{k-k_{0}}+\frac{1}{\alpha(k-k_{0})^{2}}\exp\left(-\frac{\alpha k_{0}}{4\gammax}\right)\right),
\]
and similarly, 
\[
\E\left\Vert u_{2}\right\Vert ^{2}=\bigO\left(\frac{\tau_{2\alpha}}{k-k_{0}}+\frac{1}{\alpha(k-k_{0})^{2}}\exp\left(-\frac{\alpha k_{0}}{4\gammax}\right)\right).
\]
Furthermore, by equation~\eqref{eq:bias-char} we have $\left\Vert v\right\Vert ^{2}=\bigO(\alpha^{4})$.

Combining these bounds and noting that $\tau_{2\alpha}\le\tau_{\alpha}$,
we obtain 
\[
\E\left[\left(\tilde{\theta}_{k-k_{0}}-\theta^{*}\right)\left(\tilde{\theta}_{k-k_{0}}-\theta^{*}\right)^{\top}\right]=\bigO\left(\frac{\tau_{\alpha}}{k-k_{0}}\right)+\bigO\left(\frac{1}{\alpha(k-k_{0})^{2}}\exp\left(-\frac{\alpha k_{0}}{4\gammax}\right)\right)+\bigO\left(\alpha^{4}\right).
\]
We have completed the proof of Corollary~\ref{cor:rr-ext-bounds}.

\subsection{Proof of Negative Definiteness in TD Learning}

\label{sec:neg-def-proof}

In this subsection, we show that 
\[
-\BarA=-\E_{(s_{k},s_{k+1})\sim\pi}[\phi(s_{k})\big(\gamma\phi(s_{k+1})-\phi(s_{k})\big)^{\top}]\in\real^{d\times d}
\]
is a positive definite matrix. We start by noting the following property
of $I-\gamma P^{\cS}$ on $L^{2}(\pi^{\mS})$.

\begin{claim} \label{clm:positive-op} For any $v\in L^{2}(\pi^{\mS})$
and $v\neq0$, we have $\dotp{v,(I-\gamma P^{\cS})v}_{L^{2}(\pi^{\mS})}\geq(1-\gamma)\|v\|_{L^{2}(\pi^{\mS})}^{2}$.
Therefore, $I-\gamma P^{\cS}$ is a positive operator on $L^{2}(\pi^{\mS})$.
\end{claim}

\begin{proof}[Proof of Claim~\ref{clm:positive-op}] When $v\neq0$,
we have 
\[
\dotp{v,(\gamma P-I)v}_{L^{2}(\pi^{\mS})}=\gamma\dotp{v,P^{\cS}v}-\|v\|_{L^{2}(\pi^{\mS})}^{2}\leq(\gamma-1)\|v\|_{L^{2}(\pi^{\mS})}^{2}<0,
\]
where the last step holds since $\gamma<1.$\end{proof}

We next observe the following consequence of the linear independence
of $\{\phi_{i}\}$.

\begin{claim} \label{clm:non-zero-lower} There exists some $\rho>0$
such that $\Big\|\sum_{i=1}^{d}u_{i}\phi_{i}\Big\|_{L^{2}(\pi^{\mS})}\geq\rho\|u\|,\forall u\in\real^{d}$.\end{claim}

\begin{proof}[Proof of Claim~\ref{clm:non-zero-lower}] Define
the matrix $\Bar{\Phi}\in\R^{d\times d}$ by $\Bar{\Phi}_{ij}=\dotp{\phi_{i},\phi_{j}}_{L^{2}(\pi^{\mS})}$.
We can write 
\[
\Big\|\sum_{i=1}^{d}u_{i}\phi_{i}\Big\|_{L^{2}(\pi^{\mS})}^{2}=\dotp{\sum_{i=1}^{d}u_{i}\phi_{i},\sum_{i=1}^{d}u_{i}\phi_{i}}_{L^{2}(\pi^{\mS})}=\sum_{i=1}^{d}\sum_{j=1}^{d}u_{i}u_{j}\Bar{\Phi}_{ij}=u^{\top}\Bar{\Phi}u.
\]
It follows that $\Big\|\sum_{i=1}^{d}u_{i}\phi_{i}\Big\|_{L^{2}(\pi^{\mS})}^{2}\ge\rho\left\Vert u\right\Vert ^{2},\forall u\in\real^{d}$,
where $\rho$ is the smallest eigenvalue of $\Bar{\Phi}$. Note that
$\rho\ge0$ since the left hand side of the last display equation
is non-negative. For the sake of deriving a contradiction, assume that
$\rho=0$ and let $a\in\R^{d}$ be the corresponding unit-norm eigenvector.
Then, we have $\Big\|\sum_{i=1}^{d}a_{i}\phi_{i}\Big\|_{L^{2}(\pi^{\mS})}^{2}=a^{\top}\bar{\Phi}a=0,$
which implies that $\sum_{i=1}^{d}a_{i}\phi_{i}=0$ (see footnote
\ref{fn:quotient_space}), contradicting the linear independence of
$\{\phi_{i}\}$. \end{proof}

We are ready to prove $-\BarA$ is positive definite. Using the definitions
of $A$ and $\BarA$, we have 
\begin{align*}
-u\BarA u & =\sum_{i=1}^{d}\sum_{j=1}^{d}u_{i}u_{j}\dotp{\phi_{i},(I-\gamma P^{\cS})\phi_{j}}_{L^{2}(\pi^{\mS})}\\
 & =\dotp{\sum_{i=1}^{d}u_{i}\phi_{i},(I-\gamma P^{\cS})\left(\sum_{i=1}^{d}u_{i}\phi_{i}\right)}_{L^{2}(\pi^{\mS})}\\
 & \geq(1-\gamma)\Big\|\sum_{i=1}^{d}u_{i}\phi_{i}\Big\|_{L^{2}(\pi^{\mS})}^{2}\\
 & \ge(1-\gamma)\rho\|u\|_{L^{2}(\pi^{\mS})}^{2}.
\end{align*}
where the last two steps follow from Claims~\ref{clm:positive-op}
and \ref{clm:non-zero-lower}, respectively. Since $\rho>0$, the
claim follows.

\section{Existence of Higher Moments}
\label{sec:higher_moments}

The result in \citep[Theorem 9]{srikant-ying19-finite-LSA} provides
a sufficient condition for the existence of the $m$-th moment of
the LSA iterates $\theta_{k}$. Their condition turns out to be more
restrictive than necessary. By tightening several intermediate steps
in their proof, we can establish the following Proposition~\ref{prop:higher_moments},
which gives a more relaxed condition. In Appendix~\ref{sec:proof_higher_moments}
to follow, we explain how to modify the proof of \citep[Theorem 9]{srikant-ying19-finite-LSA}
to prove Proposition~\ref{prop:higher_moments}.

\begin{proposition} \label{prop:higher_moments} Assume the stepsize
$\alpha$ satisfies the condition~\eqref{eq:alpha-constraint}. Then,
for each positive integer $m$ obeying 
\begin{equation}
m\cdot\alpha\tau<\frac{1}{4\sqrt{\gammax}}\left(\frac{1}{\sqrt{\gammin}}+1\right)^{-1},\label{eq:high-mom-constraint}
\end{equation}
it holds for all $k\geq k_{m}$ that 
\[
\E[\|\theta_{k}\|^{2m}]\leq(2m-1)!!(c\alpha\tau)^{m},
\]
where 
\[
k_{m}=m\tau+\frac{\Tilde c}{\alpha}\left(\log\frac{1}{\alpha}\right)\sum_{t=1}^{m}\frac{1}{t},
\]
and both $c$ and $\Tilde c$ are constants independent of $\alpha$
and $m$. \end{proposition}

In the proof of Theorem~\ref{thm:thm-converge}, we make use of the
existence of the $4$th moment. Taking $m=2$ in Proposition~\ref{prop:higher_moments},
we see that the condition~\eqref{eq:high-mom-constraint} becomes
\[
\alpha\tau<\frac{1}{8\sqrt{\gammax}}\left(\frac{1}{\sqrt{\gammin}}+1\right)^{-1}.
\]
Recall our stepsize condition~\eqref{eq:alpha-constraint}: $\alpha\tau\leq\frac{0.05}{95\gammax}$.
Using the inequality $\gammax\geq\gammin\geq\frac{1}{2}$ established
in equation~\eqref{eq:gammax-lower-bound}, we have 
\[
\frac{0.05}{95\gammax}=\frac{0.05}{95\sqrt{\gammax}}\cdot\frac{1}{\sqrt{\gammax}}\leq\frac{0.1}{95\sqrt{\gammax}}\leq\frac{1}{32\sqrt{\gammax}}<\frac{1}{8\sqrt{\gammax}}\left(\frac{1}{\sqrt{\gammin}}+1\right)^{-1}.
\]
Therefore, the condition~\eqref{eq:alpha-constraint} implies that
the condition~\eqref{eq:high-mom-constraint} holds with $m=2$,
which in turn ensures the existence of a finite 4th moment and proves
the claim in equation~\eqref{eq:4th-moment}.

\subsection{Proof of Proposition~\ref{prop:higher_moments}}

\label{sec:proof_higher_moments}

The proof is similar to that of \citep[Theorem 9]{srikant-ying19-finite-LSA}.
We only point out the differences. In the proof of \citep[Theorem 9]{srikant-ying19-finite-LSA},
the key constraint on $\alpha\tau$ and $m$ to ensure a finite
$m$-th moment arises when bounding $\E[\|\Psi_{0}\|^{2m}]$, where
$\Psi_{k}=\Gamma^{1/2}\theta_{k+\tau}$; see \citep[Appendix D.4]{srikant-ying19-finite-LSA}.
Below we provide a refinement of the arguments therein.

We start with the following decomposition 
\begin{align}
\|\Psi_{0}\|^{2m}-\|\Psi_{k}\|^{2m} & =\sum_{t=0}^{2m-1}\left(\|\Psi_{0}\|^{2m}-t\|\Psi_{k}\|^{t}-\|\Psi_{0}\|^{2m-(t+1)}\|\Psi_{k}\|^{t+1}\right)\nonumber \\
 & =\sum_{t=0}^{2m-1}\|\Psi_{0}\|^{2m-(t+1)}\|\Psi_{k}\|^{t}\left(\|\Psi_{0}\|-\|\Psi_{k}\|\right).\label{eq:telescoping-term}
\end{align}
Note that 
\begin{align*}
\|\Psi_{0}\|-\|\Psi_{k}\|\leq\|\Psi_{k}-\Psi_{0}\| & \leq\sqrt{\gammax}\|\theta_{k}-\theta_{0}\|\\
 & \overset{\text{(i)}}{\leq}2\alpha k\sqrt{\gammax}(\|\theta_{0}\|+\bmax)\\
 & \leq2\alpha k\sqrt{\gammax}\left(\frac{1}{\sqrt{\gammin}}\|\Psi_{0}\|+\bmax\right),
\end{align*}
where we use Lemma~\ref{lem:s-lemma3} in step~(i). Hence, for the
$t$-th summand on the RHS of equation~\eqref{eq:telescoping-term},
we have 
\begin{align*}
 & \|\Psi_{0}\|^{2m-(t+1)}\|\Psi_{k}\|^{t}\left(\|\Psi_{0}\|-\|\Psi_{k}\|\right)\\
\leq & 2\alpha k\sqrt{\gammax}\|\Psi_{0}\|^{2m-(t+1)}\|\Psi_{k}\|^{t}\left(\frac{1}{\sqrt{\gammin}}\|\Psi_{0}\|+\bmax\right)\\
\leq & 2\alpha k\sqrt{\gammax}\left(\frac{1}{\sqrt{\gammin}}\|\Psi_{0}\|^{2m-t}\|\Psi_{k}\|^{t}+\bmax\|\Psi_{0}\|^{2m-(t+1)}\|\Psi_{k}\|^{t}\right)\\
\leq & 2\alpha k\sqrt{\gammax}\left(\frac{1}{\sqrt{\gammin}}(\|\Psi_{0}\|^{2m}+\|\Psi_{k}\|^{2m})+\bmax(\|\Psi_{0}\|^{2m-1}+\|\Psi_{k}\|^{2m-1})\right).
\end{align*}
We further note the following bound: 
\begin{align}
\frac{1}{\sqrt{\gammin}}\|\Psi_{0}\|^{2m}+\bmax\|\Psi_{0}\|^{2m-1} & =\|\Psi_{0}\|^{2(m-1)}\left(\frac{1}{\sqrt{\gammin}}\|\Psi_{0}\|^{2}+\bmax\|\Psi_{0}\|\right)\nonumber \\
 & \leq\|\Psi_{0}\|^{2(m-1)}\left(\frac{1}{\sqrt{\gammin}}\|\Psi_{0}\|^{2}+(\bmax^{2}+\|\Psi_{0}\|^{2})\right)\nonumber \\
 & =\left(\frac{1}{\sqrt{\gammin}}+1\right)\|\Psi_{0}\|^{2m}+\bmax^{2}\|\Psi_{0}\|^{2(m-1)}.\label{eq:4th-mom-to-2nd-part1}
\end{align}
Similarly, we have 
\begin{equation}
\frac{1}{\sqrt{\gammin}}\|\Psi_{k}\|^{2m}+\bmax\|\Psi_{k}\|^{2m-1}\leq\left(\frac{1}{\sqrt{\gammin}}+1\right)\|\Psi_{k}\|^{2m}+\bmax^{2}\|\Psi_{k}\|^{2(m-1)}.\label{eq:4th-mom-to-2nd-part2}
\end{equation}
Combining~\eqref{eq:4th-mom-to-2nd-part1} and \eqref{eq:4th-mom-to-2nd-part2},
the $t$-th summand on the RHS of \eqref{eq:telescoping-term} admits
the following upper bound: 
\begin{align*}
 & \|\Psi_{0}\|^{2m-(t+1)}\|\Psi_{k}\|^{t}\left(\|\Psi_{0}\|-\|\Psi_{k}\|\right)\\
\leq & 2\alpha k\sqrt{\gammax}\left(\left(\frac{1}{\sqrt{\gammin}}+1\right)(\|\Psi_{0}\|^{2m}+\|\Psi_{k}\|^{2m})+\bmax^{2}(\|\Psi_{0}\|^{2(m-1)}+\|\Psi_{k}\|^{2(m-1)})\right)\\
\leq & 2\alpha k\sqrt{\gammax}\left(\left(\frac{1}{\sqrt{\gammin}}+1\right)(\|\Psi_{0}\|^{2m}+\|\Psi_{k}\|^{2m})+\bmax^{2}(\|\Psi_{0}\|^{2(m-1)}+\|\Psi_{k}\|^{2(m-1)})\right).
\end{align*}

Substituting the above bound back into equation~\eqref{eq:telescoping-term},
we have 
\[
\|\Psi_{0}\|^{2m}-\|\Psi_{k}\|^{2m}\leq4m\alpha k\sqrt{\gammax}\left(\left(\frac{1}{\sqrt{\gammin}}+1\right)(\|\Psi_{0}\|^{2m}+\|\Psi_{k}\|^{2m})+\bmax^{2}(\|\Psi_{0}\|^{2(m-1)}+\|\Psi_{k}\|^{2(m-1)})\right).
\]
Set $C\equiv C(A,b,\pi)=4\sqrt{\gammax}\left(\frac{1}{\sqrt{\gammin}}+1\right)$
and $C'\equiv C'(A,b,\pi)=\sqrt{\gammax}\bmax^{2}$. We have the inequalities
\begin{align*}
\|\Psi_{0}\|^{2m}-\|\Psi_{k}\|^{2m} & \leq m\alpha kC(\|\Psi_{0}\|^{2m}+\|\Psi_{k}\|^{2m})+m\alpha kC'(\|\Psi_{0}\|^{2(m-1)}+\|\Psi_{k}\|^{2(m-1)}),\\
(1-m\alpha kC)\|\Psi_{0}\|^{2n} & \leq(1+m\alpha kC)\|\Psi_{k}\|^{2m}+m\alpha kC'(\|\Psi_{0}\|^{2(m-1)}+\|\Psi_{k}\|^{2(m-1)}),\\
\|\Psi_{0}\|^{2m} & \leq\frac{1+m\alpha kC}{1-m\alpha kC}\|\Psi_{k}\|^{2m}+\frac{n\alpha kC'}{1-m\alpha kC}(\|\Psi_{0}\|^{2(m-1)}+\|\Psi_{k}\|^{2(m-1)}).
\end{align*}

Therefore, the constraint on $m$ arises as we set $\tau=k$ and require
$m\alpha\tau C<1$. Hence, to ensure a finite $m$-th moment, we require
$m\alpha\tau<\frac{1}{C}$, which corresponds to the condition~\eqref{eq:high-mom-constraint}
in the statement of Proposition~\ref{prop:higher_moments}.

\section{Details for Numerical Experiments}
\label{sec:expt_details}

In this section, we provide the details for the setup of the numerical
experiments in Section~\ref{sec:experiments}.

\subsection{Setup for LSA Experiments}
\label{sec:expt_detals_LSA}

For the experiments on LSA, we generate the transition probability
matrix $P$ and functions $A$ and $b$ randomly on an 8-state finite
state space as follows.

We first illustrate the steps we take to generate the transition matrix
$P$. For a given $n\,(=|\cX|)$, we start with a random matrix $M^{(P)}\in[0,1]^{n\times n}$
with entries $m_{ij}^{(P)}\overset{\text{i.i.d.\ }}{\sim}U[0,1]$,
and normalize it to obtain a stochastic matrix $\hat{M}^{(P)}=\left(\hat{m}_{ij}^{(P)}\right)$
with $\hat{m}_{ij}^{(P)}=\frac{m_{ij}^{(P)}}{\sum_{k=1}^{n}m_{ik}^{(P)}}$.
We then examine the period and reducibility of the stochastic matrix
$\hat{M}^{(P)}$ to ensure that it is aperiodic and irreducible, which
gives a uniformly ergodic Markov chain as required in Assumption~\ref{assumption:uniform-ergodic}.
If $\hat{M}^{(P)}$ is not aperiodic or irreducible, we then repeat
the above procedure until we obtain one, and set $P:=\hat{M}^{(P)}$.
Now with $P$ generated, we compute the stationary distribution $\pi$.

Next, we proceed to generate $A(x)$ for $x\in\cX$. As we also need
$\BarA=\E_{\pi}[A(x)]$ Hurwitz as required in Assumption~\ref{assumption:hurwitz},
we start with generating the Hurwitz matrix $\BarA$ and then add
noise to obtain the respective $A(x)$. We first generate a random
matrix $M^{(A)}\in\R^{d\times d}$ with $m_{ij}^{(A)}\overset{\text{i.i.d.\ }}{\sim}\mathcal{N}(0,1)$,
and examine the eigenvalues $\lambda_{i}(M^{(A)})$, as Hurwitz matrix
has eigenvalues all with strictly negative real parts. If $\text{Re}(\lambda_{i}(M^{(A)}))<0$
for all $i=1,\ldots,d$, then $M^{(A)}$ is Hurwitz and we set it
as $\BarA:=M^{(A)}$. Otherwise, we adjust $M^{(A)}$ to obtain a
Hurwitz matrix, $\BarA:=M^{(A)}-2\max(\text{Re}(\lambda_{i}(M^{(A)})))\cdot I_{d}$.
With $\BarA$ generated, we add a noise matrix $E(x)\in[-1,1]^{d\times d}$
to $\BarA$ to obtain $A(x)$, i.e., $A(x)=\BarA+E(x)$. As $\E_{\pi}[E(x)]=0$,
we only generate $E(x)$ with $e(x)_{ij}\overset{\text{i.i.d.\ }}{\sim}U[-1,1]$
for $x=1,\ldots,n-1$, and set $A(n)=\BarA-\sum_{x=1}^{n-1}\pi_{x}E(x)$.
Lastly, to align with our boundedness Assumption~\ref{assumption:bounded},
we normalize $A(x)$ by the following procedure, 
\[
A(x)\leftarrow A(x)/\max_{x}\|A(x)\|,\quad\BarA\leftarrow\BarA/\max_{x}\|A(x)\|,
\]
which ensures that $\Amax:=1$.

Lastly, we generate $b(x)\in\R^{d}$ with $b(x)_{i}\overset{\text{i.i.d.\ }}{\sim}[-1,1]$
and obtain $\bar{b}=\sum_{x}\pi_{x}b(x)$ and $\bmax=\max_{x}\|b(x)\|$.

\subsection{Setup for TD(0) Experiments}
\label{sec:expt_details_TD}

We consider the TD(0) algorithm applied to the so-called ``problematic
MDP'' considered in the work \citep{koller2000-pi,Lagoudakis03-LSPI}.
This MDP involves $n^{\cS}=4$ states, $\cS=\{1,2,3,4\}$, arranged
from left to right. At each state, there are two available actions,
``Left" (L) and ``Right" (R). When the action L is chosen, with
probability $0.9$ the state transitions to the left (or stays at
the same position if the current state is the leftmost state $1$),
and with probability $0.1$ the state transitions in the opposite
direction (or stay at the same position if the current state is the
rightmost state $4$). The dynamics under the action R is defined
symmetrically. The reward function is given by $r(1)=0,r(2)=1,r(3)=3,r(4)=0$,
with a discount factor $\gamma=0.9$. We consider evaluating the policy
that takes the actions R, R, L, and L at states $1,2,3,4$, respectively
(this policy is the optimal policy for this MDP). The induced MRP
is illustrated in Figure~\ref{fig:chain-4}. 
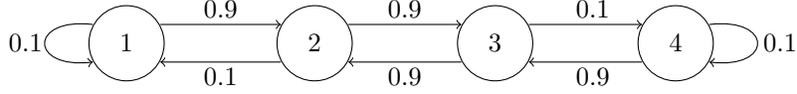
\begin{figure}[htbp]
\centering \begin{tikzpicture}
    \draw (-2.5, 0) circle (.5) node{$1$};
    \node at (-1.25, 0.45) {$0.9$};
    \draw [->] (-2.95, 0.25) .. controls (-3.8,0.35) and (-3.8,-0.35) ..  (-2.95, -0.25);
    \node at (-3.84, 0) {$0.1$};
    \draw[->] (-2.05, 0.25) -- (-0.45,0.25);
    \draw[->] (-0.45, -0.25) -- (-2.05,-0.25);
    \node at (-1.25, -0.45) {$0.1$};
    
    \draw (0, 0) circle (.5) node{$2$};
    \draw[->] (0.45,0.25) -- (1.95,0.25);
    \draw[->] (1.95,-0.25) -- (0.45,-0.25);
    \node at (1.2, 0.45) {$0.9$};
    \node at (1.2, -0.45) {$0.9$};
    
    \draw (2.4, 0) circle (.5) node{$3$};
    \draw[->] (2.85,0.25) -- (4.35,0.25);
    \draw[->] (4.35,-0.25) -- (2.85,-0.25);
    \node at (3.7, 0.45) {$0.1$};
    \node at (3.7, -0.45) {$0.9$};
    
    \draw (4.8, 0) circle (.5) node{$4$};
    \draw [->] (5.25, 0.25) .. controls (6.1,0.35) and (6.1,-0.35) ..  (5.25, -0.25);
    \node at (6.19, 0) {$0.1$};
   
    \end{tikzpicture} \caption{The Problematic MDP under ``RRLL" Policy.}
\label{fig:chain-4} 
\end{figure}

We apply TD(0) with linear function approximation to the above MRP.
For each state $s\in\{1,2,3,4\}$, the corresponding $d=3$ dimensional
feature vector is given by 
\[
\phi(s)=(1,s,s^{2})^{\top},
\]
which is used in the work \citep{koller2000-pi}. We then normalize
each row of the feature matrix $\Phi\in\R^{n^{\cS}\times d}$ to have
unit norm; explicitly, we set 
\[
\phi(s)_{i}\leftarrow\frac{\phi(s)_{i}}{\sum_{s=1}^{4}\phi(s)_{i}},\quad i=1,2,3,4.
\]
Note that one may ensure the condition $\max_{s\in\cS}\|\phi(s)\|\le\frac{1}{\sqrt{1+\gamma}}$
required by our theory by further rescaling the entire matrix $\Phi$.
In our experiments, we ignore this rescaling step, as it is equivalent to simply rescaling the stepsize and iterates.

\subsection{Setup for SGD Experiments}
\label{sec:expt_detals_SGD} 

For the experiments on SGD applied to least squares regression, the
data $\obs_{t}\in\R^{2}$ is sequentially generated from an independent
Metroplis-Hastings (MH) sampler. The target stationary distribution
of the MH sampler is $\text{Uniform}[-1,1]\times\text{Uniform}[-1,1]$.
The states of the MH sampler are generated by employing a sampling
distribution $q$ for $h\in\real^{2}$, where each coordinate $h(i)$
has the density 
\[
q(h(i))=\begin{cases}
1/4 & \text{if }-1\leq h(i)<0\\
3/4 & \text{if }0\leq h(i)<1,
\end{cases}\quad\text{for }i=1,2.
\]
Given state $g_{t}$, the next state $\obs_{t+1}$ is generated as
follows. We first generate a new sample state $h$ coordinate by coordinate
from the sampling distribution $q$. Then, we accept $h$ and set
$\obs_{t+1}=h$ with probability 
\[
\min\left\{ \frac{q(\obs_{t}(1))}{q(h(1))}\frac{q(\obs_{t}(2))}{q(h(2))},1\right\} ;
\]
 otherwise, we set $\obs_{t+1}=\obs_{t}$.
\end{document}